\documentclass[10pt]{article}

\usepackage{amsthm,amsmath,amssymb,amsfonts,amssymb,mathtools}
\usepackage[colorlinks,citecolor=blue,linkcolor=blue,bookmarks=true]{hyperref}
\usepackage[nameinlink]{cleveref}
\usepackage{mysty}
\usepackage[letterpaper,top=2cm,bottom=2cm,left=3cm,right=3cm,marginparwidth=1.75cm]{geometry}
\usepackage{xspace}

\usepackage{graphicx}
\usepackage{natbib}
\usepackage{caption}
\usepackage{framed}

\usepackage{algorithm}
\usepackage[noend]{algorithmic}
\usepackage{tikz}
\usepackage{subcaption}
\newtheorem{theorem}{Theorem}[section]

\newtheorem{problem}[theorem]{Problem}

\newtheorem{lemma}[theorem]{Lemma}
\newtheorem{informal theorem}[theorem]{Theorem (informal statement)}

\newtheorem{proposition}[theorem]{Proposition}
\newtheorem{corollary}[theorem]{Corollary}
\newtheorem{claim}[theorem]{Claim}
\newtheorem{fact}[theorem]{Fact}

\newtheorem{remark}[theorem]{Remark}

\newtheorem{definition}[theorem]{Definition}

\newcommand{\eqdef}{\coloneqq}
\crefname{question}{question}{questions}
\usepackage{graphicx}

\title{Robustly Learning Monotone Generalized Linear Models\\ 
via Data Augmentation\thanks{Conference version appeared in proceedings of COLT'25.  Minor changes in the initialization section (\Cref{app:subsec:initialization-for-monotone}) compared to previous Arxiv version.}}

\author{
Nikos Zarifis ‖\thanks{Supported in part by NSF Medium Award CCF-2107079.} \\
UW Madison\\
{\tt zarifis@wisc.edu}\\
\and Puqian Wang ‖\thanks{Supported in part by NSF Award DMS-2023239 and by the Air Force Office of Scientific Research under award number FA9550-24-1-0076.} \\
UW Madison\\
{\tt pwang333@wisc.edu}
\and Ilias Diakonikolas\thanks{Supported in part by NSF Medium Award CCF-2107079 and an H.I.\ Romnes Faculty Fellowship.}\\
		UW Madison\\		{\tt ilias@cs.wisc.edu}\\
\and Jelena Diakonikolas\thanks{Supported in part by the Air Force Office of Scientific Research under award number FA9550-24-1-0076, by the U.S.\ Office of Naval Research under contract number  N00014-22-1-2348, and by the NSF CAREER Award CCF-2440563. Any opinions, findings and conclusions or recommendations expressed in this material are those of the author(s) and do not necessarily reflect the views of the U.S. Department of Defense.}\\
UW Madison\\
{\tt jelena@cs.wisc.edu}\\
}
\date{}

\newcommand{\Pm}[1]{\mathrm{P}_{#1}}

\newcommand{\Tr}{\mathrm{T}_{\rho}}

\newcommand{\Tre}{\mathrm{T}}

\newcommand{\OU}{Ornstein–Uhlenbeck\xspace}

\newcommand{\Exx}{\E_{\x\sim \D_\x}}

\newcommand{\Ey}{\E_{(\x,y)\sim \D}}

\newcommand{\Et}{\E_{t\sim \normal}}
\newcommand{\CS}{Cauchy-Schwarz\xspace}

\newcommand{\lp}{\left}
\newcommand{\rp}{\right}

\renewcommand\vec[1]{\mathbf{#1}}
\DeclareMathOperator*{\pr}{\mathbf{Pr}}
\DeclareMathOperator*{\E}{\mathbf{E}}
\newcommand{\proj}{\mathrm{proj}}

\def\d{\mathrm{d}}
\newcommand{\normal}{\mathcal{N}}

\DeclareMathOperator*{\argmin}{argmin}

\newcommand{\bx}{\mathbf{x}}
\newcommand{\by}{\mathbf{y}}
\newcommand{\bv}{\mathbf{v}}
\newcommand{\bu}{\mathbf{u}}
\newcommand{\bz}{\mathbf{z}}
\newcommand{\bw}{\mathbf{w}}

\newcommand{\e}{\mathbf{e}}

\newcommand{\B}{\mathbb{B}}

\newcommand{\R}{\mathbb{R}}

\newcommand{\Z}{\mathbb{Z}}

\newcommand{\eps}{\epsilon}
\newcommand{\poly}{\mathrm{poly}}

\newcommand{\sgn}{\mathrm{sign}}
\newcommand{\sign}{\mathrm{sign}}
\newcommand{\calN}{\mathcal{N}}
\newcommand{\calL}{\mathcal{L}}

\newcommand{\calF}{\mathcal{F}_M}

\newcommand{\opt}{\mathrm{OPT}}
\newcommand{\D}{\mathcal{D}}

\newcommand{\Ind}{\mathds{1}}
\newcommand{\1}{\Ind}

\newcommand{\littlesum}{\mathop{\textstyle \sum}}

\newcommand{\wt}{\widetilde}
\newcommand{\wh}{\widehat}

\newcommand{\wstar}{\bw^{\ast}}
\newcommand{\x}{\vec x}
\newcommand{\w}{\vec w}

\newcommand{\ith}{^{(i)}}
\newcommand{\tth}{^{(t)}}
\newcommand{\Ex}{\E_{\x\sim\D_\x}}
\newcommand{\Exy}{\E_{(\x,y)\sim \D}}
\newcommand{\Ez}{\E_{z\sim\calN}}

\newcommand*\diff{\mathop{}\!\mathrm{d}}

\newcommand{\he}{\mathrm{He}}

\newcommand{\g}{\mathbf{g}}

\newcommand{\hep}{\mathrm{he}}

\newcommand{\whg}{\wh{\g}}

\newcommand{\evt}{\mathcal{E}}
\newcommand{\ty}{\tilde{y}}

\newcommand{\tx}{\tilde{\mathbf{x}}}
\newcommand{\tstrth}{^{(t^*)}}

\begin{document}
\maketitle
\def\thefootnote{‖}\footnotetext{Equal contribution.}
\def\thefootnote{*}
\renewcommand{\thefootnote}{\arabic{footnote}}

\begin{abstract}
We study the task of learning 
Generalized Linear models (GLMs) in the agnostic 
model under the Gaussian distribution. 
We give the first polynomial-time 
algorithm that achieves a constant-factor approximation 
for {\em any} monotone Lipschitz activation. 
Prior constant-factor GLM learners succeed for a substantially 
smaller class of activations. 
Our work resolves a well-known open problem, 
by developing a robust counterpart to the classical GLMtron
algorithm~\citep{kakade2011efficient}.  
Our robust learner applies more generally, encompassing all 
monotone activations with bounded $(2+\zeta)$-moments, 
for any fixed $\zeta>0$---a condition that is 
essentially necessary. 
To obtain our results, we leverage a novel data augmentation technique 
with decreasing Gaussian noise injection and prove a number 
of structural results that may be useful in other settings. 
\end{abstract}

\setcounter{page}{0}
\thispagestyle{empty}

\newpage

\section{Introduction}

A Generalized Linear Model (GLM) is any function of the form 
$\sigma(\vec w^{\ast} \cdot \vec x)$, where $\sigma: \R \to \R$ 
is a known activation function and $\vec{w}^{\ast}$ is a hidden vector.
GLMs constitute one of the most basic supervised learning models 
capturing hidden low-dimensional structure in high-dimensional 
labeled data. As such, GLMs have been studied 
over the course of several decades~\citep{NW72, 
DobB08}. Specifically, the special case where $\sigma$ is the sign function corresponds to Linear Threshold 
Functions (LTFs) whose study goes back to~\cite{R58}. 

In the realizable setting, the learning problem is as follows: 
given labeled examples $(\x,y) \in \R^d \times \R$ 
from an unknown distribution $\D$, whose labels are consistent 
with a GLM, i.e., $y = \sigma(\vec w^{\ast} \cdot \vec x)$ where $\sigma$ 
is known and $\w^{\ast}$ unknown, the goal is to approximate 
the underlying function (and/or the hidden direction $\w^{\ast}$) 
with respect to the square loss. 

A classical work~\citep{kakade2011efficient} gave 
a simple gradient-based algorithm (GLMtron) for this problem 
when the data is supported on the unit 
ball, under the assumption that the activation function is monotone and Lipschitz. 
The GLMtron algorithm also succeeds in the presence 
of zero-mean random label noise. 

We point out that 
{for GLM learning to even be information-theoretically solvable, 
some regularity assumptions on the activation $\sigma$ are necessary. 
Moreover, even if $\sigma$ is sufficiently well-behaved so that no 
information-theoretic impediment exists, 
computational hardness results rule out efficient algorithms 
even for Gaussian data and a small amount of random label 
noise~\citep{Song2021}.} 

Over the past five years, there has been a resurgence 
of research interest on learning GLMs 
in the more challenging  {\em agnostic} 
(or adversarial label noise) model~\citep{Haussler:92, KSS:94}, 
where no assumptions are made on the labels 
and the goal is to compute a hypothesis 
that is competitive with the {\em best-fit} function in the class. 
The ideal result in this setting would be 
an efficient agnostic learner 
that succeeds for all marginal distributions 
and achieves optimal error. Such a goal appears 
unattainable, due to known computational hardness.
Specifically, even for Gaussian marginals and a ReLU activation, 
there is strong evidence that any such algorithm requires 
super-polynomial time~\citep{DKZ20, GGK20, DKPZ21, DKR23}. 
Moreover, even if we relax our goal to any 
constant factor approximation,  
distributional assumptions are necessary~\citep{MR18, DKMR22}. 
Thus,  research has focused on 
constant factor approximate learners 
in the distribution-specific setting. 

\noindent Denoting $\calL(\w)\eqdef \Exy[(\sigma(\w\cdot\x) - y)^2]$, 
our agnostic learning problem is defined as follows.

\begin{problem}[Robustly Learning GLMs] \label{def:agnostic-learning}
Let $\sigma:\R\to\R$ be a known activation and 
$\D$ be a distribution of 
$(\x,y) \in \R^d \times \R$ such that its $\x$-marginal 
$\D_\x$ is the standard normal. We say that an algorithm 
is a {\em $C$-approximate proper GLM learner}, for some $C \geq 1$, 
if given $\eps>0$, $W>0$, and i.i.d.\ samples from $\D$, 
the algorithm outputs a vector $\widehat{\w}\in \R^d$ 
such that with high probability it holds 
\( \Ey[(\sigma(\widehat {\w}\cdot\x)-y)^2] \leq C \, \opt  +\eps \), where  
$\opt \triangleq \min_{\|\vec w\|_2 \leq W}\Ey[(\sigma(\w\cdot\x)-y)^2] $.  
\end{problem}

{Motivated by the setting introduced in~\citep{kakade2011efficient}, 
a major algorithmic goal in this area has been to obtain an efficient 
constant-factor approximate learner that succeeds for {\em any} 
monotone Lipschitz activation function.} A line of recent work
~\citep{DGKKS20,DKTZ22,ATV22, WZDD23, GGKS23, ZWDD2024, guo2024agnostic}
has made algorithmic progress on various special cases of this question.  
This progress notwithstanding, the general case remained open, prompting the following question:\begin{center}
{\em Is there an efficient {\em constant-factor} approximate learner 
for monotone Lipschitz GLMs\\ under Gaussian marginals?}
\end{center}
As a special case of our main result, we answer this question in the affirmative.

\begin{theorem}[Robustly Learning Monotone \& Lipschitz GLMs]\label{thm:main-monotone-b-lip} There exists an algorithm 
with the following performance guarantee: For any known monotone 
and $b$-Lipschitz activation $\sigma$, given $\eps>0$, {$W>0$,} 
and $N = \tilde{\Theta}({d (b{W})^2}/{\eps} + d/\eps^2)$ samples, 
the algorithm runs in $\poly(d,N)$ time  
and returns a vector $\wh{\w}$ such that with high probability,   
$\Exy[(\sigma(\wh{\w}\cdot\x) - y)^2] \leq C\opt + \eps$, 
where $C$ is an absolute constant independent of $\eps, d, b, {W}$.
\end{theorem}

\noindent {We emphasize that the approximation ratio of 
our algorithm is a universal constant---independent of the dimension, 
the desired accuracy, the Lipschitz constant, and the radius 
of the space.} 

The key qualitative difference between prior work 
and \Cref{thm:main-monotone-b-lip} is  
in the assumptions on the activation. 
Specifically, 
prior constant-factor GLM learners succeed for a much  
smaller subclass of activations. 
In fact, our main algorithmic result (\Cref{thm:main-monotone})
applies more generally, encompassing all   
monotone activations with bounded $(2 + \zeta)$ moment, for any $\zeta > 0$ 
(\Cref{cor:main-monotone-4mom}). 
This in particular implies that the case of LTFs 
fits in our setting. 
{We stress here that some assumption on top of 
monotonicity is information-theoretically 
necessary, even for realizable learning (\Cref{thm:lb}).}

\medskip

\noindent {\bf Comparison to Prior Work\;} 
\citet{GGKS23} gave an efficient GLM 
learner for monotone Lipschitz activations 
and marginal distribution with bounded second moment.
{However, the error of their algorithm 
scales linearly with $W$ and the Lipschitz constant.}
\cite{WZDD23,ZWDD2024} studied \Cref{def:agnostic-learning} under `well-behaved' 
distributions, where $\sigma$ is monotone and $(a,b)$ unbounded, 
meaning that $|\sigma'(z)|\leq b$ and $\sigma'(z)\geq a$ when $z\geq 0$. 
They provided an efficient algorithm with error $O(\poly(b/a))\opt + \eps$.
Note that when $a = 0$, this error guarantee is vacuous. 
More recently, \cite{WZDD2024sample} studied the same problem 
under Gaussian marginals for activations 
with bounded information-exponent. The approximation ratio of their method 
inherently scales polynomially with the radius $W$ of the space. 
{See \Cref{app:sec:compare} for more details.}

\medskip

\noindent {\bf Remark}  
In the sequel, we assume that the scale of the 
target vector $\wstar,$ $\|\wstar\|_2$, is known and, by, 
rescaling the space, we optimize $\w$ on the unit sphere. 
The unknown scale of $\wstar$ can be resolved by a 
simple grid search. For our 
approach, this rescaling is w.l.o.g.\ because---unlike in 
prior work \citep{WZDD23,WZDD2024sample,ZWDD2024}---the 
approximation ratio of our algorithm is \textbf{independent} 
of any problem parameters. 
For a formal justification, see \Cref{app:rmk:rescale-w} and \Cref{app:prop:wlog-konw-norm-wstr}. 

\medskip

\noindent {\bf Organization\;} 
In \Cref{sec:tech}, we summarize our algorithmic ideas and techniques. 
In \Cref{sec:augmentation-and-landscape}, we analyze the landscape of the augmented loss. 
Our main algorithm and its analysis for learning Gaussian GLMs of general activations is presented in \Cref{sec:learn-sim-general}. In \Cref{sec:learn-monotone}, we focus on monotone activations 
and show that our algorithm achieves error 
$O(\opt) + \eps$ under very mild assumptions. 
Due to space limitations, several proofs have been deferred to the Appendix. 

 \subsection{Technical Overview}\label{sec:tech}

Our work relies on three main technical ingredients: (1) data augmentation, which we use as a method to mitigate the effect of the adversarial label noise, (2) an optimization-theoretic local error bound, which in our work is a structural result that identifies the ``signal'' vector field that guides the algorithm toward the set of target solutions, (3) a suite of structural results for $(B, L)$-regular monotone activations (see \Cref{ass:activation}), leveraging their piecewise-constant approximations, smoothing through data augmentation, and representation via Hermite polynomials.  

\medskip

\noindent {\bf Data Augmentation\;} 
Data augmentation encompasses a broad set of techniques for modifying or artificially generating data to enhance learning and estimation tasks. In the context of our work, data augmentation refers to the injection of Gaussian noise into the data vectors $\x$ while retaining the same labels. In particular, given any labeled example $(\x, y)\sim \D$ and a parameter $\rho \in (0, 1)$, the considered data augmentation process generates labeled examples $(\Tilde{\x}, y),$ where $\Tilde{\x} = \rho \x + \sqrt{1 - \rho^2}\bz$ and $\bz$ is an independently generated sample from the standard normal distribution. While this type of data augmentation is a common empirical technique in machine learning, it is considered to be a wild card: although  sometimes helpful, it can also be detrimental to learning guarantees (see, e.g., \citet{yin2019fourier,lin2024good}). Thus, on a conceptual level, one of our contributions is showing that for the considered GLM learning task, data augmentation is provably beneficial. 

The effect of data augmentation on the considered GLM learning task is that it simulates the \OU semigroup $\Tre_\rho f(t)\eqdef \E_{z\sim \normal(0, 1)}[f(\rho t + \sqrt{1-\rho^2}z)]$ applied to any function $f(\w\cdot \x)$. This process smoothens the function $f$ and induces other regularity properties. Unlike the common use of smoothing in the optimization literature, where the key utilized properties are continuity and smoothness of the smoothed objective function (see, e.g., \citet{Nesterov2017,duchi2012randomized,bubeck2019complexity,diakonikolas2024optimization}), in our work, the key feature is the effect of injected noise on enhancing the signal in the data, as explained below.  

Suppose we were given a GLM learning task. Since the goal of a learning algorithm is to minimize the mean squared loss $\calL(\w) = \Exy[(\sigma(\w\cdot\x) - y)^2]$, a natural approach is to follow a gradient field associated with the error $\sigma(\w\cdot\x) - y$. Indeed, all prior work for this task proceeds by applying (stochastic) gradient-based algorithms to either the original squared loss $\calL(\w)$ or its surrogate $\mathcal L_{\mathrm{sur}}(\vec w)=\Exy\left[\int_{0}^{\vec w\cdot\x}(\sigma(t)-y) \d t\right]$. In either case, the associated gradient field can be represented by $\Exy[(\sigma(\w \cdot \x) - y)h(\w \cdot \x)\x]$ for some function $h$ (in particular,  for the squared loss $h(\w\cdot \x) = 2\sigma'(\w \cdot \x),$ while for the surrogate loss, $h \equiv 2$). Since we are considering optimizing $\w$ over the unit sphere (see \Cref{app:rmk:rescale-w}), the relevant information for updating $\w$ is in its orthogonal complement (as we are not changing its length), so it suffices to consider 
\(\g(\w, h) \eqdef \Exy[(\sigma(\w \cdot \x) - y)h(\w \cdot \x)\x^{\perp \w}].\)

\noindent Intuitively, if we can show that $-\g(\w, h)$ 
strongly correlates with 
$\wstar$, then this information can be used to update $\w$ to better 
align with $\wstar$, until we reach the target approximation error. 
Observe first that, as the Gaussian distribution is independent across 
orthogonal directions, we have 
$- \g(\w, h) \cdot \wstar = \Exy[y h(\w \cdot \x)\x^{\perp \w}\cdot \wstar]$. 
Writing $y = \sigma(\wstar\cdot \x) + y - \sigma(\wstar\cdot \x),$ 
the quantity $- \g(\w, h) \cdot \wstar$ can be decomposed into two 
parts: (i) corresponding to ``clean'' labels $\sigma(\wstar\cdot \x)$, 
and (ii) corresponding to label noise $y - \sigma(\wstar\cdot \x)$. Letting $\theta$ denote the angle between $\w$ and $\wstar,$ 
it is possible to argue (using Stein's lemma, see \Cref{fct:stein}) 
that the ``clean label'' portion of $- \g(\w, h) \cdot \wstar$ 
equals $\Exy[\sigma'(\wstar\cdot\x) h(\vec w\cdot\x) ]\sin^2\theta.$ 
For the ``noisy'' label portion, by the Cauchy-Schwarz inequality and the definition of $\opt,$ we can write 
\begin{align*}
    &\; -\Exy[(y-\sigma(\wstar\cdot\x)) h(\vec w\cdot\x) (\wstar\cdot\x^{\perp \vec w})]\leq  \sqrt{\opt}\|h\|_{L_2}\sin(\theta(\vec w,\wstar)),
\end{align*}

\noindent where $\|h\|_{L_2}:= \big(\E_{z\sim \normal(0, 1)}[h^2(z)]\big)^{1/2}$.
Since labels are adversarial, the inequality can in fact be made to 
hold with equality. Thus, summarizing the above discussion, we have
\begin{align}\label{eq:final-update}
    -\vec g(\vec w,h)\cdot{\wstar}&\geq \Exy[\sigma'(\wstar\cdot\x) h(\vec w\cdot\x) ]\sin^2\theta-\sqrt{\opt}\|h\|_{L_2}{\sin\theta}.
\end{align}

\noindent We can assume w.l.o.g.\ that $\|h\|_{L_2} = 1,$ 
since dividing both sides by $\|h\|_{L_2}$ would give us 
the same conclusion. For $-\vec g(\vec w,h)$ to contain a useful 
signal guiding the algorithm towards target solutions, 
we need that $-\vec g(\vec w,h)\cdot{\wstar} > 0,$ 
for which we ought to argue that 
$G(h):= \Exy[\sigma'(\wstar\cdot\x) h(\vec w\cdot\x) ] > 0$. 
It is possible to argue that $G(h)$ is maximized for the ``ideal'' 
choice of 
$h(\w \cdot \x) \propto \sigma'(\cos\theta \w\cdot \x + \sin \theta \bz \cdot \x)$ 
with independently sampled $\bz \sim \normal(\bf{0, I}).$ 
This can equivalently be seen as applying the \OU semi-group 
with parameter $\rho = \cos \theta$ to $\sigma'$, 
which motivates its use in our work. 
Of course, since $\cos\theta$ is not known to the algorithm, 
the smoothing parameter $\rho$ needs to be carefully chosen 
and adjusted between the algorithm updates. 

\medskip

\noindent {\bf Alignment and Optimization\;}
Local error bounds have long history in optimization and represent some of 
the most important technical tools for establishing iterate convergence to 
target solutions, especially in the context of gradient-based algorithms 
(see, e.g., \cite{Pang1997}). Broadly speaking, local error bounds are 
inequalities that bound below some measure of the problem ``residual'' or 
error by a measure of distance to the target solution set. ``Local'' in the 
name refers to such inequalities being valid only in a local region around 
the target solution set. Within learning theory and in the context of GLM 
learning, they have played a crucial role in the analysis of 
(stochastic) gradient-based algorithms \citep{mei2018landscape,WZDD23,ZWDD2024,WZDD2024sample}.  

Our main structural result, stated in \Cref{prop:structural}, is a local 
error bound for which the residual is $-\g(\w)\cdot\w^*$ for the gradient 
field $\g(\w)$ corresponding to the data augmented squared loss function, as 
discussed above. This residual has the meaning of the ``alignment'' between 
$-\g(\w)$ and $\wstar.$ Specifically, we prove that in a local region around 
a certain set ${\cal S}$, the following inequality holds for any 
$\rho \in (0, 1)$ and $\theta$ being the angle between $\w, \wstar:$ 
\begin{equation}\label{eq:alignment-sharpness}
    -\g(\w)\cdot\w^* \geq (2/3) \|\Tre_{\sqrt{\rho\, \cos\theta}}\sigma'\|_{L_2}^2 \sin^2\theta.
\end{equation}
Observe that, since we are optimizing over the unit sphere, 
$\|\w - \wstar\|_2^2 \approx \sin^2(\theta).$ 
This structural result allows us to argue that, 
provided an initial parameter vector $\w_0$ 
for which \eqref{eq:alignment-sharpness} holds, 
we can update iterates $\w$ to contract the angle $\theta,$ 
until the set ${\cal S}$ is reached. While this general idea 
seems relatively simple, making it work requires 
a rather technical argument to 
(i) ensure we can initialize the algorithm in the region where \eqref{eq:alignment-sharpness} holds, 
 (ii) adjust the value of $\rho$ between the algorithm updates 
to ensure we remain in the region where \eqref{eq:alignment-sharpness} applies, and (iii) argue that all parameter vectors in ${\cal S}$ 
are $O(\opt)$ approximate solutions. 
Part (ii) is handled using an intricate inductive argument. 
Parts (i) and (iii) are addressed by proving 
a series of structural results for the class 
of $(B, L)$-regular monotone activations, discussed below.

\medskip

\noindent {\bf Approximation and Regularity of Monotone Functions\;}
While handling arbitrary monotone functions {is provably impossible}, 
we show that fairly minimal assumptions suffice for our approach. 
In particular, we handle all $(B, L)$-regular monotone activations, 
which we show can be well-approximated by monotone piecewise-constant (staircase) functions. In more detail, instead of directly proving 
the desired properties of monotone $(B, L)$-regular activations, 
we consider the class of staircase functions, 
which only increase within a compact interval 
(and are constant outside it). For this class of staircase functions, 
we prove that the high-degree terms in their Hermite expansion 
(see \Cref{app:sec:prelims} for relevant definitions)---namely, terms with 
degree $>1/\theta^2$ for $\theta$ sufficiently small---are bounded 
by $\|\Tre_{\cos\theta}\sigma'\|_{L_2}^2 \sin^2\theta$, and, 
further, this result extends to all $(B, L)$-regular functions 
(\Cref{prop:error-bound-smoothing-tails}). Proving this structural result 
relies on auxiliary results relating \OU semigroups of activations 
and their derivatives that may be of independent interest. 
\Cref{prop:error-bound-smoothing-tails} is then used to argue 
that the target set ${\cal S}$ to which the iterates 
of the algorithm converge only contains vectors 
with $L_2^2$ error $O(\opt)$, addressing the aforementioned issue (iii). 

Since the result from \Cref{prop:error-bound-smoothing-tails} only applies for sufficiently small $\theta,$ we need to argue that the algorithm can be appropriately initialized. In particular, random initialization is insufficient since we need roughly that $\theta_0 \leq O((\log(1/\eps))^{-1/2}).$ To address this requirement, we apply a label transformation $\tilde{y} = \1\{y \geq t\}$ for a carefully chosen threshold $t,$ where $\1$ denotes the indicator function. In particular, to select $t,$ we leverage the staircase approximation of monotone functions discussed above. We argue that the problem reduces to learning $\sign(\sigma(\wstar\cdot\x)-t)$, which is an instance of learning halfspaces with adversarial noise.\footnote{Note here that $\sign(\sigma(\wstar\cdot\x)-t)$ being a halfspace crucially relies on $\sigma$ being monotone.} In particular, we argue that constant approximate solutions to this halfspace learning problem suffice for our initialization.

\subsection{Preliminaries}

For $n \in \Z_+$, let $[n] \eqdef \{1, \ldots, n\}$.  We use bold lowercase letters to denote vectors
and  bold uppercase letters for matrices.  For $\bx \in \R^d$ and $i \in [d]$, $\bx_i$ denotes the
$i^{\mathrm{th}}$ coordinate of $\bx$, and $\|\bx\|_2 \eqdef (\littlesum_{i=1}^d \bx_i^2)^{1/2}$ denotes the
$\ell_2$-norm of $\bx$.  We use $\bx \cdot \by $ for the dot product of $\bx, \by \in \R^d$
and $ \theta(\bx, \by)$ for the angle between $\bx, \by$.  We slightly abuse notation and denote by
$\vec e_i$ the $i$-th standard basis vector in $\R^d$.  We use $\1\{A\}$ to denote the
characteristic function of the set $A$. 
For unit vectors $\bu,\bv$, we use $\bu^{\perp \bv}$ to denote the component of $\bu$ that is orthogonal to $\bv$ i.e., $\bu^{\perp \bv} = (\vec I - \bv\bv^\top)\bu$.  
Finally, we use $\mathbb{S}^{d-1}$ to denote the unit sphere in $\R^d$ and $\B$ to denote the unit ball.
For $(\x,y)$
distributed according to $\D$, we denote by $\D_\x$ the marginal distribution of $\x$.
We use the standard $O(\cdot), \Theta(\cdot), \Omega(\cdot)$ asymptotic notation and 
$\wt{O}(\cdot)$ to omit polylogarithmic factors in the argument.

\medskip

\noindent {\bf Gaussian Space\;}
 Let $\calN(\vec 0, \vec I)$ denote the standard normal distribution. 
The $L_2$ norm of a function $g$ with respect to the standard normal 
is $\|g\|_{L_2} = ( \E_{\x \sim \normal} [ |g(\x)|^2)^{1/2}]$, 
while $\|g\|_{L_\infty}$ is the essential supremum of the absolute value of $g$.  
We denote by $L_2(\normal)$ the vector space of all functions $f:\R^d
\to \R$ such that $\|f\|_{L_2} < \infty$.  
$\he_i(z)$ denotes the \emph{normalized} probabilist's Hermite polynomial of degree $i$.
For any function $f:\R\to\R$, $f \in L_2(\normal)$, we denote by $\Pm{k}f(z)$ the degree $k$ partial sum of the Hermite expansion of $f$, i.e., 
$\Pm{k} f (z) = \sum_{i \leq k} \hat{f}(i) \he_i(z)$, and let $\Pm{>k} f (z) = \sum_{i > k} \hat{f}(i) \he_i(z)$,
where \(
 \hat{f}(i) = \E_{z\ \normal(0,1)} [f(z) \he_i(z)] 
\).
An important tool for our work is the \OU semigroup, formally defined below.

\begin{definition}[\OU\ Semigroup]
\label{def:OU-operator}
Let $\rho \in (0,1)$. The \OU semigroup, denoted  by $\Tr$, is a linear operator that maps a function $g \in L_2(\normal)$ to the function 
$\Tr g$ defined as: $(\Tr g) (\vec x) \eqdef\E_{\vec z\sim \normal}[g(\rho\x+\sqrt{1-\rho^2}\vec z)]$. 
\end{definition}

\section{Data Augmentation and Its Effect on the $L_2^2$ Loss Landscape}\label{sec:augmentation-and-landscape}

This section describes the basic data augmentation approach and provides some of the key structural properties relating to the data-augmented $L_2^2$ loss. 
\subsection{Augmenting the Data: Connection to \OU Semigroup}

As already discussed in \Cref{sec:tech}, our algorithm relies on the data augmentation technique, i.e., in each iteration, the algorithm injects Gaussian noise (see \Cref{alg:aug}), which has the effect of improving the regularity properties of the loss landscape, as shown in this section. \begin{algorithm}[h]
   \caption{Augment Dataset with Injected White Noise}
   \label{alg:aug}
\begin{algorithmic}[1]
\STATE {\bfseries Input:} Parameters $\rho$, $m$; Sample data $\mathfrak{D}=\{(\vec x^{(1)},y^{(1)}),\ldots,(\vec x^{(N)},y^{(N)})\}$; $S\gets \emptyset$ 
\FOR{$(\x\ith,y\ith)\in \mathfrak{D}$}
\FOR{$j=1,\ldots, m$}
\STATE Sample $\vec z$ from $\mathcal{N}(\vec 0,\vec I)$ and let $\tilde{\x}^{(j)}\gets \rho \x\ith+(1-\rho^2)^{1/2} \vec z$.
\STATE $S\gets S\cup \{(\tilde{\x}^{(j)},y\ith)\}$.
\ENDFOR
\ENDFOR
\STATE {\bfseries Return:} $S$
\end{algorithmic}
\end{algorithm}

The augmentation can be viewed as a transformation of the distribution $\D$ to $\D_\rho$, where for any $(\tx,y)\sim\D_\rho$, we have $\tx \sim \rho\D_\x + (1 - \rho^2)^{1/2}\calN(\vec 0,\vec I)$. 
The data augmentation introduced in \Cref{alg:aug} in fact simulates the \OU semigroup, as stated below.
\begin{lemma}\label{lem:ou-aug}  Let $\D$ be a distribution of labeled examples
$(\x,y)$ such that 
$\D_\x = \mathcal N(\vec 0,\vec I)$ and let $\D_\rho$ be the distribution constructed by applying  \Cref{alg:aug} to  $\D$. Then, for any $f:\R\to\R$ and any unit vector $\vec w\in \R^d$ with $|\Exx[f(\vec w\cdot\x)]|<\infty$, we have 
$
\E_{\tx\sim (\D_{\rho})_{\tx}}[f(\vec w\cdot\tx)]=\Ex[\Tr f(\vec w\cdot \x)]\;.
$
\end{lemma}
\subsection{Alignment of the Gradients of the Augmented Loss}

Our main structural result is to show that the gradients of the square loss applied to the augmented data correlate with a target parameter vector $\wstar.$ We use $\mathcal L_\rho(\vec w)=\E_{(
\tx,y)\sim \D_{\rho}}[(\sigma( \w\cdot\tx)-y)^2] $
to denote the square loss on the augmented data and refer to it as the ``augmented loss.''

\begin{proposition}[Main Structural Result]\label{prop:structural}
Fix an activation $\sigma:\R\to\R$. Let $\D$ be a distribution of labeled examples
$(\x,y)$ such that 
$\D_\x=\mathcal N(\vec 0,\vec I)$ and let $\D_{\rho}$ with $\rho \in (0, 1)$ be the distribution resulting from applying \Cref{alg:aug} to $\D.$ Fix vectors $\wstar,\w\in \mathbb{S}^{d-1}$ such that $\calL(\wstar) = \opt$ and let $\theta=\theta(\wstar,\w)$. Let \(\g(\vec w)=(1/(2\rho))(\nabla_{\vec w} \mathcal L_\rho(\vec w))^{\perp \vec w}.\)
If $0<\rho\leq \cos\theta<1$ and  $\sin\theta\geq 3\sqrt{\opt}/\|\Tre_{\rho}\sigma'\|_{L_2}$, then, $    \g(\w)\cdot\w^*\leq -(2/3) \|\Tre_{\sqrt{\rho\, \cos\theta}}\sigma'\|_{L_2}^2 \sin^2\theta$.\end{proposition}

To prove the proposition, we rely on the following auxiliary lemma, which relates $\g(\w)$ to the \OU semigroup applied to the derivative of the activation. \begin{lemma}\label{lem:ou-aug-gradient}
    Let $\g(\vec w)=(1/(2\rho))(\nabla_{\vec w} \mathcal L_\rho (\vec w))^{\perp \vec w}$. Then, $\g(\w) = - \Exy[y\Tr\sigma'(\w\cdot\x)\x^{\perp\w}]$. 
\end{lemma}

\begin{proof}[Proof Sketch of \Cref{prop:structural}]
Assume that $(\w^*)^{\perp_\w}\neq \vec 0$; otherwise the statements hold trivially. Let $\bv \eqdef (\w^*)^{\perp_\w}/\|(\w^*)^{\perp_\w}\|_{2}$; then $\w^* = \cos\theta\w + \sin\theta\bv$ and $\w\cdot\x$, $\bv\cdot\x$ are independent standard Gaussians.
By \Cref{lem:ou-aug-gradient}, $-\g(\w)\cdot\w^* = \Exy[y\Tr\sigma'(\vec w\cdot\x)\bv\cdot\x]\sin\theta$. 
Hence, adding and subtracting $\sigma(\wstar\cdot\x)$ to $y$ in the expectation we get $-\g(\w)\cdot\w^* = ((Q_1) + (Q_2))\sin\theta$, where $(Q_1)\eqdef \Ex[\sigma(\w^*\cdot\x)\Tr\sigma'(\w\cdot\x)\bv\cdot\x]$ and $(Q_2)\eqdef \Exy[(y - \sigma(\w^*\cdot\x))\Tr\sigma'(\w\cdot\x)\bv\cdot\x]$.

\noindent By Cauchy-Schwarz inequality, $(Q_2)\geq -\sqrt{\opt\Ex[(\Tr\sigma'(\w\cdot\x))^2]}=-\sqrt{\opt}\|\Tr\sigma'\|_{L_2}$,
    where we used the definition of $\opt$ and that $\w\cdot\x$ and $\bv\cdot\x$ are independent Gaussians.  
To bound $(Q_1)$, applying Stein's lemma (\Cref{fct:stein}) as well as the properties of \OU semigroup (\Cref{fct:semi-group}) we can show that $(Q_1) =\Ex[\Tre_{\cos\theta}\sigma'(\w\cdot\x)\Tr\sigma'(\w\cdot\x)]\sin\theta=\|\Tre_{\sqrt{\rho\, \cos\theta}}\sigma'\|_{L_2}^2\sin\theta$.
Therefore, we have that $-\g(\w)\cdot\w^*\geq  \|\Tre_{\sqrt{\rho\, \cos\theta}}\sigma'\|_{L_2}^2\sin^2\theta -\sqrt{\opt}\|\Tre_{\rho}\sigma'\|_{L_2}\sin\theta$.

To finish the proof,  note that $\|\Tre_{\lambda}f\|_{L_2}$ is non-decreasing in $\lambda\in(0,1)$ for any function $f\in L_2(\normal)$ (\Cref{fct:semi-group}), therefore
$ \|\Tre_{\sqrt{\rho\, \cos\theta}}\sigma'\|_{L_2}\geq \|\Tre_{\rho}\sigma'\|_{L_2}$ if $\cos\theta\geq \rho$. Using the assumption that $\sin\theta\geq 3\sqrt{\opt}/\|\Tre_{\rho}\sigma'\|_{L_2}$, we obtain $-\g(\w)\cdot\w^*\geq  (2/3)\|\Tre_{\sqrt{\rho\, \cos\theta}}\sigma'\|_{L_2}^2\sin^2\theta.$
\end{proof}
\subsection{Critical Points and Their Connection to the $L_2^2$ Loss}\label{subsec:critical-point&connection-to-loss}
\Cref{prop:structural} provides sufficient conditions ensuring that the vector $-\vec g(\w)$ guides $\w$ towards the direction of $\wstar$ whenever we are in a region around approximate solutions. Specifically, if the parameter $\rho$ is chosen appropriately and the following alignment condition holds:
$\sin\theta \|\Tre_{\cos\theta} \sigma'\|_{L_2} \geq 3\sqrt{\opt}$,
 then  $-\vec g(\w)$ has a nontrivial correlation with $\wstar$. Otherwise, we can guarantee that the angle between $\w$ and $\wstar$ is already sufficiently small. 
This implies that the region of convergence of an algorithm that relies on $-\vec g(\w)$ depends on the quantity: $\psi_{\sigma}(\theta) \coloneqq \sin\theta \|\Tre_{\cos\theta} \sigma'\|_{L_2}.$
Motivated by this observation, we define the \emph{Convergence Region}, as follows. \begin{definition}[Critical Point and  Convergence Region of $\sigma$]\label{def:star-point-regions}
     Given $\sigma : \R \to \R$, $\sigma\in L_2(\mathcal N)$,  and $\theta_0\in[0,\pi/2]$, we define the error alignment function  $\psi_{\sigma}:[0,\pi/2]\to\R_+$  by  $\psi_{\sigma}(\theta)\eqdef \sin\theta \|\Tre_{\cos\theta}\sigma'\|_{L_2}$. For any $\eps>0$, we define the  Convergence Region $\mathcal R_{\sigma,\theta_0}(\eps)=\{\theta: \psi_{\sigma}(\theta)\leq {\sqrt\eps}\}\cap\{\theta: 0\leq \theta\leq\theta_0\}$. We say that $\theta^*$ is a $(\sigma,\theta_0,\eps)$-Critical Point if  $\theta^*=\{\max \theta: \theta \in \mathcal R_{\sigma,\theta_0}(\eps)\}$.
\end{definition}
\Cref{def:star-point-regions} utilizes an upper bound $\theta_0$ because $\psi_{\sigma}(\theta)$ is not necessarily monotonic. Specifically, it can be shown that $\psi_{\sigma}(\theta)$ is non-decreasing up to some \(\theta'\) and then non-increasing (see \Cref{app:fig:psi_sigma} for illustrative examples and \Cref{app:claim:initialization-threshold} in \Cref{app:sec:augmentation-and-landscape} for a more formal statement and proof). 
Consequently, the region $\mathcal R_{\sigma,\theta_0}(\eps)$ may consist of two disjoint intervals. 
The role of (appropriately selected) $\theta_0$ is to ensure that this does not happen. The significance of the above definition comes from the following proposition, which bounds the $L_2^2$ error within the  Convergence Region.

\begin{proposition}[Critical Points and $L_2^2$ Error]\label{prop:error}
    Given $\sigma: \R\to\R$, $\sigma\in L_2(\mathcal N)$, and a distribution $\D$ of labeled examples
$(\x,y)$ such that 
$\D_\x = \mathcal N(\vec 0,\vec I)$, let $\wstar$ be such that $\calL(\w^*) = \opt$. Then, for any unit vector $\vec w$ with $\theta=\theta(\vec w,\wstar)$ such that $\theta\leq \theta^*$, where $\theta^*$ is the $(\sigma,\theta_0,C\opt)$-Critical Point for some $\theta_0$ and $C>1$ an absolute constant, $
        \calL(\w)\leq O(\opt)
   +4\|\Pm{>(1/\theta^*)^2}\sigma\|_{L_2}^2.
$
\end{proposition}

To prove \Cref{prop:error}, we first prove the following technical lemma, which decomposes the error into $O(\opt)$ and error terms that depend on the properties of the activation $\sigma.$ 
A more formal version of \Cref{prop:starpoints-l2error} is stated as \Cref{app:prop:starpoints-l2error} in \Cref{app:sec:augmentation-and-landscape}, where its proof is also provided.

\begin{lemma}[Error Decomposition, Informal]\label{prop:starpoints-l2error}
Under the assumptions of \Cref{prop:error}, we have that $\calL(\w)\leq 2\opt + C\theta^2\|\Tr\sigma'\|_{L_2}^2  + 4\|\Pm{>k}\sigma\|_{L_2}^2$, where $C$ is an absolute constant, for the following choices of $\rho$: (i) if $k\leq 1$, $\rho$ can be any value in $(0,1)$; (ii) if $k\geq 2$, then $\rho = \sqrt{1-1/k}$.
\end{lemma}

\begin{proof}[Proof Sketch of \Cref{prop:error}]
Since $\theta^*$ is the $(\sigma,\theta_0,C\opt)$-Critical point, we have by its definition that $(\theta^*)^2\|\Tre_{\cos(\theta^*)}\sigma'\|_{L_2}^2\leq C\opt$. 
Let $k = \lfloor 1/(\theta^*)^2\rfloor$. Consider first $\theta^*\leq 1/\sqrt{2}$, which implies that $k\geq 2$. Observe that $(1 - 1/k)^{1/2}\leq \cos\theta^*$,
thus as $\|\Tr\sigma'\|_{L_2}$ is non-decreasing with respect to $\rho$ (\Cref{fct:semi-group}), we further have $\|\Tre_{(1 - 1/k)^{1/2}}\sigma'\|_{L_2}^2\leq \|\Tre_{\cos(\theta^*)}\sigma'\|_{L_2}^2$. 
Thus, applying \Cref{prop:starpoints-l2error}, for any $\theta\leq \theta^*$, we get $\calL(\w)\leq (2+ 8eC)\opt +4\|\Pm{>(1/\theta^*)^2}\sigma\|_{L_2}^2$.
When $\theta^* > 1/\sqrt{2}$, then $k = 0, 1$. Choose $\rho = \cos(\theta^*)\in (0,1)$ in \Cref{prop:starpoints-l2error}, note again that $(\theta^*)^2\|\Tre_{\cos(\theta^*)}\sigma'\|_{L_2}^2\leq C\opt$ by the definition of $\theta^*$, thus we have $\calL(\w)\leq (2 + C)\opt + 4\|\Pm{>(1/\theta^*)^2}\sigma\|_{L_2}^2$.
\end{proof}

\section{Learning GLMs via Variable Augmentation}\label{sec:learn-sim-general}
In this section, we present our main algorithm (\Cref{alg:GD-general-activation}) for robustly learning Gaussian GLMs, as stated in \Cref{def:agnostic-learning}.  
Our algorithm applies to the following large class of activations:
\begin{definition}[$(B,L)$-Regular Activations]\label{ass:activation}
   Given parameters $B,L>0$, we define the class of $(B,L)$-Regular activations, denoted by $\mathcal{H}(B,L)$, as the class containing all functions $\sigma:\R\to\R$ such that 1) $\|\sigma\|_{L_\infty}\leq B$ and 2) $\|\sigma'\|_{L_2}\leq L$. Given $\eps>0$, we define the class of $\eps$-Extended $(B,L)$-Regular activations, denoted by  $\mathcal{H}_\epsilon(B,L)$, as the class containing all activations $\sigma_1: \R\to \R$ for which there exists $\sigma_2\in \mathcal{H}(B,L)$ such that $\|\sigma_1-\sigma_2\|_{L_2}^2\leq\eps$. 
\end{definition}
\noindent Our results hold for any activation that is $\eps$-Extended $(B,L)$-Regular.
This class contains all Lipschitz activations and all activations with bounded 4\textsuperscript{th} moment. More examples are in \Cref{app:function-class}.

\Cref{alg:GD-general-activation} uses the main structural result of \Cref{sec:augmentation-and-landscape} (\Cref{prop:structural}) to update its iterates $\w^{(t)}$. In particular, 
for $\theta_t = \theta(\w\tth,\w^*)$, we show that 
after one gradient descent-style update, the angle $\theta_{t+1}$ shrinks by a factor $1 - c$, i.e., $\theta_{t+1}\leq (1 - c)\theta_t$, where $0<c<1$ is an absolute constant.  A crucial feature of \Cref{alg:GD-general-activation} is that 
in each iteration it carefully chooses a new value of $\rho_t$. This \textit{variable} update of $\rho_t$ ensures the `signal' of the gradient is present 
until $\w\tth$ reaches a small region centered at $\w^*$. Within this region, the agnostic noise corrupts the signal of the augmented gradient and convergence to $\w^*$ is no longer be guaranteed. However, the region that $\w\tth$ reaches is in fact the  Convergence Region $\mathcal{R}_{\sigma,\theta_0}(O(\opt))$, within which all points are solutions with the target approximation error. We show in \Cref{sec:learn-monotone} that for any \textit{monotone} $(B,L)$-Regular activations, any $\wh{\w}$ in $\mathcal{R}_{\sigma,\theta_0}(O(\opt))$ is a solution with error $C\opt + \eps$, under suitable initialization. We now present our algorithm and state our main result (\Cref{thm:main-general}) for general $(B, L)$-regular activations.

\begin{algorithm}[h]
   \caption{$\mathrm{SGD}-\mathrm{VA}$: SGD with Variable Augmentation
}
   \label{alg:GD-general-activation}
\begin{algorithmic}[1]
\STATE {\bfseries Input:} Parameters $\eps, T$; Sample access to $\D$
\STATE $[\w^{(0)}, \bar{\theta}] = \textbf{Initialization}[\sigma]$ {(\Cref{subsec:initialization-for-monotone})}; set $\rho_0 =  \cos\bar{\theta}$ 
\FOR{$t = 0,\dots, T$} 
\STATE Draw $n $ samples $\widehat{\D}_{\rho_t} = \{(\tx\ith,y\ith)\}_{i=1}^n$ from $\D_{\rho_t}$ using \Cref{alg:aug}\STATE $\whg(\w^{(t)}) = -(1/\rho_t)\E_{(\tx,y)\sim\widehat{\D}_{\rho_t}}[y\sigma'(\w^{(t)}\cdot\tx)(\tx)^{\perp{\w^{(t)}}}]$ \label{line:grad-empirical}
\STATE $\eta_t = \sqrt{(1 - \rho_t)/2}/(4\|\wh{\g}(\w\tth)\|_2)$
\STATE $\w^{(t+1)} = (\w^{(t)} - \eta_{t}\whg(\w^{(t)}))/\|\w^{(t)} - \eta_{t}\whg(\w^{(t)})\|_{2}$
\STATE $\rho_{t+1} = 1 - (1 - 1/256)^2(1 - \rho_t)$\label{line:rho}
\ENDFOR
\STATE $\wh{\w} = \textbf{Test}[\w^{(0)},\w^{(1)},\dots, \w^{(T)}]$ (\Cref{app:alg:test})
\STATE{\bfseries Return:} $\wh{\w}$
\end{algorithmic}
\end{algorithm}

\begin{theorem}\label{thm:main-general}Let $\eps>0$. Let $\sigma$ be a $(B,L)$-Regular activation. \Cref{alg:GD-general-activation}, given initialization $\vec w^{(0)}$ with $\theta(\vec w^{(0)},\wstar)\leq \bar{\theta}$, 
runs at most  $T =O(\log(L/\eps))$ iterations, 
    draws ${\Theta}({dB^2\log(L/\eps)}/{\eps} + B^4\log(L/\eps)/\eps^2)$ samples, 
    and returns a vector $\wh{\w}$ such that with probability at least $2/3$, 
    {$\wh{\w} \in \mathcal{R}_{\sigma,\theta_0}(O(\opt))$.}
Moreover,  $\calL(\wh{\w}) \leq O(\opt) + \eps + 4\|\Pm{>1/(\theta^*)^2}\sigma\|_{L_2}^2$.
\end{theorem}

Define $\zeta(\rho)\eqdef \sqrt{\opt}/\|\Tre_{\rho}\sigma'\|_{L_2}$. Recall that in \Cref{prop:structural} we showed when 
\begin{equation}\label{eq:condition-for-fast-convergence}
    \text{conditions for fast convergence:}\;\sin\theta_t\geq 3\zeta(\rho_t), \,\zeta(\rho_t)\eqdef \sqrt{\opt}/\|\Tre_{\rho_t}\sigma'\|_{L_2}, \rho_t\leq\cos\theta_t
\end{equation}
hold, $-\g(\w\tth)$ aligns well with $\w^*$, 
enabling $\theta_{t+1}\leq (1 - c)\theta_t$.
However, two critical issues arise:

\noindent \textbf{(1)} If $\sin\theta_t\lesssim \zeta(\rho_t)$, then  conditions in \Cref{eq:condition-for-fast-convergence} do not hold, and we cannot guarantee that $\theta_t$ contracts. Moreover, since $\|\Tre_{\rho_t}\sigma'\|_{L_2}\leq \|\Tre_{\cos\theta_t}\sigma'\|_{L_2}$, it is not necessarily the case that $\sin\theta_t\lesssim \zeta(\cos\theta_t)$, thus we also cannot assert that $\w\tth$ has reached the target region $\mathcal{R}_{\sigma,\theta_0}(C^2\opt)$. 

\noindent \textbf{(2)} Suppose that the conditions in \Cref{eq:condition-for-fast-convergence} apply, hence $\theta_{t+1}\leq (1 - c)\theta_t$. Assume that $\w^{(t+1)}$ is still far from $\w^*$ and $\theta_{t+1}\gtrsim \zeta(\cos\theta_{t+1})$. It is possible that $\zeta(\cos\theta_{t+1}) \lesssim \theta_{t+1}\lesssim \zeta(\rho_t)$, because $\|\Tre_{\cos\theta_{t+1}}\sigma'\|_{L_2}\geq \|\Tre_{\rho_t}\sigma'\|_{L_2}$, as $\rho_t\leq \cos\theta_t\leq\cos\theta_{t+1}$ and $\|\Tr\sigma'\|_{L_2}$ is an increasing function of $\rho$ (by \Cref{fct:semi-group}). This implies that the conditions in \Cref{eq:condition-for-fast-convergence} might become invalid for $\rho_t$.

\vspace{0.25em}

To overcome these issues, we consider the event $\evt_t\eqdef\{|\cos\theta_t-\rho_t|\leq \sin^2\theta_t,\, \sin\theta_t \leq C\zeta(\rho_t)\}$. We first observe that when $\evt_t$ is satisfied, then, 
$|\cos\theta_t-\rho_t|\leq \sin^2\theta_t$ indicates that $\rho_t$ and $\cos\theta_t$ are sufficiently close and we argue that  $\zeta(\rho_t)\approx\zeta(\cos(2\theta_t))$, therefore, we have that $\sin\theta_t \leq C\zeta(\cos(2\theta_t))$. From here we argue that $\w\tth\in \mathcal{R}_{\sigma,\theta_0}(4C^2\opt)$. This addresses (1).  
Now suppose $\evt_t$ does not hold. We use induction to show that updating $\rho_t$ by Line \ref{line:rho}, we have $\rho_{t+1}\leq \cos\theta_{t+1}$. Now if $\sin\theta_{t+1}\geq 3\zeta(\rho_{t+1})$, \Cref{eq:condition-for-fast-convergence} is satisfied and we decrease $\theta_{t+1}$, whereas if $\sin\theta_{t+1}\leq 3\zeta(\rho_{t+1})$, we know that $\w^{(t+1)}$ is the target vector as discussed above.
This addresses issue (2).
{\Cref{app:fig:alg} in the appendix provides a visual illustration of the mechanism of \Cref{alg:GD-general-activation}.}
\begin{proof}[Proof Sketch of \Cref{thm:main-general}]
 Let $\theta_t = \theta(\w^{(t)},\w^*)$ and define $\zeta(\rho)\eqdef \sqrt{\opt}/\|\Tre_{\rho}\sigma'\|_{L_2}$.  Assume that $\eps\le\opt$ (otherwise we can get additive error $O(\eps)$). Suppose further that we have access to the population gradients $\vec g^{(t)}$, so that the statistical error is negligible (we bound it in \Cref{app:sec:proof-of-concentration-lemma}).

Define the event $\evt_t\eqdef\{|\cos\theta_t-\rho_t|\leq \sin^2\theta_t,\, \sin\theta_t \lesssim\zeta(\rho_t)\}$. We claim that if $\evt_t$ holds at some iteration $t$, then the algorithm converges to a vector in the region $\mathcal{R}_{\sigma,\theta_0}(O(\opt))$. In particular, in this case we have that $\rho_t\geq \cos2\theta_t$, hence $\sin\theta_t\lesssim \zeta(\rho_t)\lesssim\zeta(\cos(2\theta_t))$, i.e., $\psi_{\sigma}(2\theta_t)\lesssim \sqrt{\opt}$ as $\zeta(\rho)$ is a decreasing function, which implies that $\w^{(t)}\in\mathcal{R}_{\sigma,\theta_0}(O(\opt))$.

It remains to show that there exists some $t^*\le T$ for which $\evt_{t^*}$ holds. In fact, it suffices to prove that $\rho_t\leq \cos\theta_t$ for all $t\leq t^*$. Since $\rho_t\to 1$, if no such $t^*$ existed then eventually $\cos\theta_t$ would be arbitrarily close to $1$, forcing $\sin\theta_t\lesssim\zeta(1)$ and yielding a contradiction. We prove $\rho_t\le \cos\theta_t$ for all $t\le t^*$ by induction. By the assumptions on $\theta_0$, we have $\rho_0\le \cos\theta_0$.

\noindent\textbf{Induction Step.} Suppose that for some $0\leq t<t^*$ we have $\rho_t\le \cos\theta_t$. We argue that $\rho_{t+1}\le \cos\theta_{t+1}$. If $\evt_t$ already holds for some $t'\le t$, there is nothing to prove. Otherwise, assume that the condition $\sin\theta_t \lesssim \zeta(\rho_t)$ is violated.
Since $\g\tth$ is orthogonal to $\w\tth$, the update is given by $\w^{(t+1)} = \proj_\B(\w\tth - \eta_t\g\tth)$. Thus, $\|\w^{(t+1)} - \w^*\|_{2}^2 
        \leq \|\w\tth - \eta_{t}\g\tth - \w^*\|_{2}^2 \nonumber
        = \|\w\tth - \w^*\|_{2}^2 + \eta_t^2\|\g\tth\|_{2}^2 + 2\eta_{t}\g\tth\cdot\w^*.$ By \Cref{prop:structural}, we have $\g\tth\cdot\w^*\lesssim-\|\Tre_{\rho_t}\sigma'\|_{L_2}^2\sin\theta_t^2 $ and $\|\g\tth\|_{2}\lesssim \|\Tre_{\rho_t}\sigma'\|_{L_2}^2\sin\theta_t$, hence $\|\w^{(t+1)} - \w^*\|_{2}^2 
        \leq  \|\w\tth - \w^*\|_{2}^2  -\eta_t |\g\tth\cdot\w^*|$.
Thus, by choosing $\eta_t$ appropriately, there exists $\xi>0$ such that $\theta_{t+1}\leq \theta_t-\xi$ and if we choose $\rho_{t+1}$ so that $\cos^{-1}\rho_{t}-\cos^{-1}\rho_{t+1}<\xi$, we ensure $\rho_{t+1}\leq \cos\theta_{t+1}$.   Alternatively, if $\sin\theta_t\lesssim \zeta(\rho_t)$ and $|\cos\theta_t - \rho_t|\geq \sin^2\theta_t$, then by the triangle inequality we obtain $ \|\w^{(t+1)} - \w^*\|_{2} 
        \leq  \|\w\tth - \w^*\|_{2}  +\eta_t \|\g\tth\|_2$.  In this case, we can choose $\eta_t$ so that even if $\theta_{t+1}\ge \theta_t$, the increase is bounded by a small $\xi>0$, i.e., $\theta_{t+1}\le \theta_t + \xi$. Since $ \cos\theta_t\geq \sin^2\theta_t +\rho_t$, we can adjust $\rho_t$ to ensure that $\cos(\theta_t+\xi)\ge \rho_{t+1}$. This completes the inductive step.
\end{proof}

\section{$\mathrm{SGD}-\mathrm{VA}$ Efficiently Learns Monotone GLMs}\label{sec:learn-monotone}

We have shown in \Cref{sec:learn-sim-general} that \Cref{alg:GD-general-activation} converges to a parameter vector $\vec w$ with an $L_2^2$ error at most $O(\opt)+4\|\Pm{>1/(\theta^*)^2}\sigma\|_{L_2}^2$, where $\theta^*$ is a Critical Point.
 One of the technical difficulties is that in general we cannot bound $\|\Pm{>1/(\theta^*)^2}\sigma\|_{L_2}^2$ by $O(\opt)$. One such example is when $\sigma(t)=\he_{(1/(\theta^*)^2+1)}(t)$; in this case $\|\Pm{>1/(\theta^*)^2}\sigma\|_{L_2}^2=\|\sigma\|_{L_2}^2$, which can be  much larger than $\opt$. 
In this section, we show that if the activation is also monotone, then for sufficiently small $\theta^*$, we can bound $\|\Pm{>1/(\theta^*)^2}\sigma\|_{L_2}^2$ by the \OU semigroup of $\sigma'$. Specifically, we provide an initialization method that along with \Cref{alg:GD-general-activation} gives an algorithm that guarantees error $O(\opt)$. Formally, our main result is stated in the following theorem.

\begin{theorem}[Learning Monotone $(B,L)$-Regular Activations]\label{thm:main-monotone}
    Let $\eps>0$, and let $\sigma\in \mathcal{H}(B,L)$ be a monotone  activation. 
    Then, \Cref{alg:GD-general-activation} 
    draws $N = \tilde{\Theta}({dB^2\log(L/\eps)}/{\eps} + d/\eps^2)$ samples, 
    runs in $\poly(d, N)$ time, 
    and outputs $\wh{\w}$ such that with probability at least $2/3$,
    $\wh{\w}\in\mathcal{R}_{\sigma,\theta_0}(O(\opt) + \eps)$ and 
     $\calL(\wh{\w}) \leq C\opt + \eps$, where $C$ is an absolute constant independent of $\eps, d, B, L$.
\end{theorem}

The main result of this section is an initialization routine that allows us to bound the higher coefficients of the spectrum, $\|\Pm{>1/(\theta^*)^2}\sigma\|_{L_2}^2$. In particular, we prove the following.

\begin{proposition}[Initialization]\label{prop:initialization}
     Let $\sigma\in \mathcal{H}(B,L)$ be a monotone activation. There exists an algorithm that draws $N=\widetilde O(d/\eps^2)$ samples, runs in $\poly(N,d)$ time, and with probability at least $2/3$, returns a unit vector $\w^{(0)}\in \R^d$ such that for any unit $\w'\in \R^d$ with $\theta=\theta(\w',\wstar)\leq \theta(\w^{(0)},\wstar)$, it holds that
$
\|\Pm{>1/\theta^2}\sigma\|_{L_2}^2\lesssim \sin^2\theta\|\Tre_{\cos\theta}\sigma'\|_{L_2}^2\;.$
\end{proposition}
The proof of \Cref{thm:main-monotone} combines \Cref{thm:main-general} and \Cref{prop:initialization}, and is provided in the appendix.

If $\sigma$ satisfies $\Ez[\sigma^{2+\zeta}(z)] \leq B_{\sigma}$ for $\zeta > 0$, then $\sigma$ is an $\eps$-Extended {$((B_\sigma/\eps)^{1/\zeta},(B_\sigma/\eps)^{4/\zeta}/\eps^2)$}-Regular activation (see \Cref{app:lem:activation-truncation-2+zeta}). We thus have the following immediate corollary.
\begin{corollary}[Learning Monotone Activations With Bounded $(2+\zeta)$  Moments]\label{cor:main-monotone-4mom}
        Let $\eps>0,$ $ \zeta > 0$, and let $\sigma$ be a monotone activation such that $\Ez[\sigma^{2+\zeta}(z)] \leq B_{\sigma}$. 
    Then, \Cref{alg:GD-general-activation} 
    draws $N = \tilde{\Theta}(d(B_\sigma/\eps)^{2/\zeta}\log(B_{\sigma}/\eps)/{\eps} + d/\eps^2)$ samples, 
    runs in $\poly(d, N)$ time, 
    and outputs $\wh{\w}$ such that with probability at least $2/3$, 
     $\calL(\wh{\w}) \leq C\opt + \eps$, where $C$ is an absolute constant.
    \end{corollary}

To prove \Cref{prop:initialization}, we combine two main technical pieces: (1) proving that there exists a threshold $\bar{\theta}$ such that for any $\theta\leq \bar{\theta}$,  $\|\Pm{>1/\theta^2}\sigma\|_{L_2}^2\lesssim \sin^2\theta\|\Tre_{\cos\theta}\sigma'\|_{L_2}^2$; and (2) proving that there exists an efficient algorithm that finds a vector $\w^{(0)}$ such that $\theta(\w^{(0)},\w^*)\leq \bar{\theta}$.
\Cref{subsec:bound-high-order-coefficients} addresses (1), with main technical result stated in \Cref{prop:error-bound-smoothing-tails}. 
To prove this result, we approximate the considered monotone activations $\sigma$ by sequences of ``monotone staircase'' functions. 
\begin{definition}[Monotone Staircase Functions]\label{def:staircase}
Let 
$\phi(z;t) \eqdef \1\{z\geq t\}$ and let $m \in \Z_+$, $M > 0$. 
The class of monotone staircase functions (of $M$-bounded support) 
are defined as
$   \calF\eqdef \{\Phi_m:\R\to\R: \Phi_m(z)=\sum_{i=1}^m A_i\phi(z;t_i) + A_0: A_0\in\R; A_i > 0, |t_i|\leq M, \forall i\in[m]; m<\infty \}.$
\end{definition}

If $\Phi_k$ converges to $\sigma$ pointwise, we argue that $\|\Pm{>1/\theta^2}\sigma\|_{L_2}^2\lesssim 2\|\Phi_k-\Tre_{\cos\theta} \Phi_k\|_{L_2}^2+\theta^2\|\Tre_{\cos\theta} \Phi_k'\|_{L_2}^2$. We further show that $\Tre_{\cos\theta} \Phi_k'\to\Tre_{\cos\theta}\sigma'$, therefore, 
it remains to show that $\|\Phi_k-\Tre_{\cos\theta} \Phi_k\|_{L_2}^2\lesssim \theta^2\|\Tre_{\cos\theta} \Phi_k'\|_{L_2}^2$.
\Cref{prop:smoothing-error-bound} in \Cref{subsection:bounding-aug-error} proves the claim that when $\rho$ is not too small, $\|\Phi-\Tr \Phi\|_{L_2}^2\lesssim \theta^2\|\Tr \Phi'\|_{L_2}^2$, for any $\Phi(z)$ that is a {\itshape monotonic staircase function}. 
These staircase functions constitute a dense subset of the monotone function class and have a simple and easy-to-analyze form, therefore they serve well for our purpose. In \Cref{def:staircase}, $M$ is chosen to be a bound on the support of $\sigma'$, which is always finite by \Cref{app:claim:bounded-support-sigma'}.
In \Cref{subsec:initialization-for-monotone}, we prove (2) by providing an initialization algorithm. 
Finally, combining (1) and (2), we prove \Cref{prop:initialization}.

\subsection{Bounding Higher Order Hermite Coefficients of Monotone Activations}\label{subsec:bound-high-order-coefficients}
The main result of this subsection is the following:
\begin{proposition}[From Hermite Tails to \OU Semigroup]\label{prop:error-bound-smoothing-tails}
    Let $\sigma\in L_2(\mathcal N)$ be a monotone activation, $M$ be the upper bound for the support of $\sigma'(z)$\footnote{In \Cref{app:claim:bounded-support-sigma'}, we show that $\forall\sigma\in\mathcal{H}(B,L)$, the support of $\sigma'(z)$ can be truncated at some $M<+\infty$ w.l.o.g.}. 
For any $\theta\in[0,\pi]$ such that $1 - C/M^2<\cos^2\theta$ with $C>0$ an absolute constant, it holds $\|\Pm{>1/\theta^2}\sigma\|_{L_2}^2\lesssim \sin^2\theta\|\Tre_{\cos\theta}\sigma'\|_{L_2}^2$. 
\end{proposition}
\begin{proof}[Proof Sketch of \Cref{prop:error-bound-smoothing-tails}]
Let $\Phi_k$ be a sequence of monotone staircase  functions (\Cref{def:staircase}) that converges to $\sigma$ with respect to $L_2$; this is true because piecewise constant functions are dense over compact sets with respect to the $L_2$ norm (in this case the compact set is $[-M,M]$). For $\rho^2 \geq 1 - C/M^2$, where $M$ is the upper bound on the support of $\sigma'$ and $\Phi_k'$, by Young's inequality we have $\|\Pm{>1/\theta^2}\sigma\|_{L_2}^2\leq 2\|\Pm{>1/\theta^2}(\sigma-\Phi_k)\|_{L_2}^2+4\|\Pm{>1/\theta^2}(\Phi_k-\Tr \Phi_k)\|_{L_2}^2+4\|\Pm{>1/\theta^2}\Tr \Phi_k\|_{L_2}^2$.
Observe that 
$\|\Pm{>m}f\|_{L_2}^2= \littlesum_{i > m} \hat{f}(i)^2\leq\|f\|_{L_2}^2.$
    Therefore, $\|\Pm{>1/\theta^2}(\sigma-\Phi_k)\|_{L_2}^2\leq \|\sigma - \Phi_k\|_{L_2}^2\to 0$. In addition, note that for any $f,f'\in L_2(\calN)$, it holds $\|\Pm{>m} f\|_2^2 \leq \sum_{i > m} (i/m) \hat{f}(i)^2\leq (1/m)\| f'\|_{L_2}^2, $
thus $\|\Pm{>1/\theta^2}\Tr\Phi_k\|_{L_2}^2\leq \theta^2\|(\Tr\Phi_k)'\|_{L_2}^2$. Further, by \Cref{fct:semi-group}, we have $\|(\Tr\Phi_k)'\|_{L_2}^2\leq \|\Tr\Phi_k'\|_{L_2}^2$ since $\rho < 1$, thus,   $\|\Pm{>1/\theta^2}\sigma\|_{L_2}^2\leq 4\|\Phi-\Tr \Phi\|_{L_2}^2+4\theta^2\|\Tr \Phi'\|_{L_2}^2$ when $k\to \infty$.  
Next, by \Cref{prop:smoothing-error-bound}, we conclude that $\|\Phi_k-\Tr\Phi_k\|_{L_2}^2\lesssim (1-\rho^2)\|\Tr \Phi_k'\|_{L_2} $, and, therefore, we have that $\|\Pm{>1/\theta^2}\sigma\|_{L_2}^2\leq 4((1-\rho^2)+\theta^2)\|\Tr \Phi_k'\|_{L_2}^2$.
  In \Cref{app:lem:norm-derivative-sequence}, we
show that the sequence of smoothed derivatives $\Tr \Phi_k'$ also converges to $\sigma'$, therefore it holds $\|\Tr \Phi_k'\|_{L_2}^2\to\|\Tr \sigma'\|_{L_2}^2$. Letting $\rho = \cos\theta$ completes the proof.  
\end{proof}

\subsection{Bounding the Augmentation Error}\label{subsection:bounding-aug-error}

Our main technical result provides an upper bound on the smoothing error of piecewise staircase functions using the $L_2(\calN)$ norm of the smoothed derivative, as stated below. 
\begin{proposition}\label{prop:smoothing-error-bound}
    Let $\Phi\in\calF$. For any $\rho\in(0,1)$ such that $\rho^2\geq 1 - C/M^2$ where $C<M^2$ is an absolute constant, we have $\|\Tr\Phi-\Phi\|_{L_2}^2\lesssim (1 - \rho^2)\|\Tr\Phi'\|_{L_2}^2$.
\end{proposition}

Proceeding to the proof of \Cref{prop:smoothing-error-bound}, technical difficulties arise when we try to relate $\|\Tr\Phi(z)-\Phi(z)\|_{L_2}^2$ with $\|\Tr\Phi'(z)\|_{L_2}^2$. The main obstacle is that it is hard to analyze $\Tr\phi(z;t)-\phi(z;t)$, since $\Tr\phi(z;t)= \pr_{u\sim\calN}[u\geq (t - \rho z)/(1- \rho^2)^{1/2}],$ 
and the probability term does not have a closed form. Our workaround is to introduce a new type of `centered augmentation (smoothing)' operator $\Tr\Phi(z/\rho)$ that takes a more simple and easy-to-analyze form, and then translate the upper bound on the centered augmentation error back to the upper bound on the standard augmentation error. We show that $\Delta\eqdef \|\Tr\Phi(z)-\Phi(z)\|_{L_2}^2$ is bounded by the following three terms $\Delta\lesssim\Delta_1 + \Delta_2 + \Delta_3$, where $\Delta_1\eqdef \|\Tr\Phi(z) - \Tre_{\rho_1}\Phi(z)\|_{L_2}^2$, $\Delta_2\eqdef \|\Tre_{\rho_1}\Phi(z) - \Tre_{\rho_1}\Phi(z/\rho_1)\|_{L_2}^2$ and $\Delta_3\eqdef \|\Tre_{\rho_1}\Phi(z/\rho_1) - \Phi(z)\|_{L_2}^2$,
with $\rho_1\in(0,1)$ being a carefully chosen parameter that is slightly larger than $\rho$. Taking advantage of the nice analytic form of $\Tr\Phi(z/\rho)$, we show that all these three terms can be bounded by $\|\Tr\Phi'(z)\|_{L_2}^2$, using the properties of $\Tr\Phi(z/\rho)$ provided in \Cref{lem:center-aug-all-results}.

We define the centered augmentation as $\Tr\sigma(z/\rho) = \E_{u\sim\calN}[\sigma(z + (\sqrt{1 - \rho^2}/\rho) u)].$
We show that the $L_2^2$ error between the centered augmentation $\Tr\Phi(z/\rho)$ and $\Phi(z)$, $\Tr\Phi(z)$ are well controlled, as summarized in the following lemma (see \Cref{app:subsection:bounding-aug-error} for complete statements):
\begin{lemma}\label{lem:center-aug-all-results}
    Let $\Phi\in\calF$, $C\in(0,M^2/2]$. For any $\rho^2\geq 1-C/M^2$, it holds:
    \begin{gather}
        \|\Tr\Phi(z/\rho) - \Phi(z)\|_{L_2}^2\leq 4((1 - \rho^2)/\rho^2) \|\Tr\Phi'(z/\rho)\|_{L_2}^2\;; \label{eq:lem:upper-bound-centered-smoothed-error}\\
        \|\Tr\Phi(z) - \Tr\Phi(z/\rho)\|_{L_2}^2 \leq C'(1 - \rho^2)(\|\Tr\Phi'(z/\rho)\|_{L_2}^2 + \|\Tr\Phi'\|_{L_2}^2)\;; \label{eq:lem:bound-the-differece-TrPhi-TrPhi(z/rho)}\\
    \|\Tre_{\rho_1}\Phi'(z/\rho_1)\|_{L_2}^2\leq 2e^C  \|\Tr\Phi'(z)\|_{L_2}^2,\;\text{where }\rho_1^2 = \rho^2 + {C(1 - \rho^2)}/{M^2}. \label{eq:lem:upper-bound-centered-smoothed-derivative-squared}
    \end{gather}
\end{lemma}

\begin{proof}[Proof Sketch of \Cref{prop:smoothing-error-bound}] Let $\Delta\eqdef \|\Tr\Phi(z) - \Phi(z)\|_{L_2}^2$. Observe that by adding and subtracting $\Tre_{\rho_1}\Phi$, $\Tre_{\rho_1}\Phi(z/\rho_1)$ in the norm and repeatedly using $(a+b)^2\leq 2a^2 + 2b^2$, we have $\Delta\leq 4\Delta_1 + 4\Delta_2 + 2\Delta_3$ where $\Delta_1\eqdef \|\Tr\Phi(z) - \Tre_{\rho_1}\Phi(z)\|_{L_2}^2$, $\Delta_2\eqdef \|\Tre_{\rho_1}\Phi(z) - \Tre_{\rho_1}\Phi(z/\rho_1)\|_{L_2}^2$ and $\Delta_3\eqdef \|\Tre_{\rho_1}\Phi(z/\rho_1) - \Phi(z)\|_{L_2}^2$.

For $\Delta_1$, observe that since $\rho<\rho_1<1$, we can use the property that $\Tr\Phi(z) = \Tre_{\rho/\rho_1}(\Tre_{\rho_1}\Phi(z))$ and $(\Tre_{\rho_1}\Phi(z))' = \rho_1\Tre_{\rho_1}\Phi'(z)$ (\Cref{fct:semi-group}). Using \Cref{clm:difference} with $f(z) = \Tre_{\rho_1}\Phi(z)$ and noting that $\|\Tre_{\rho_1}\Phi'(z)\|_{L_2}^2\lesssim\|\Tre_{\rho}\Phi'(z)\|_{L_2}^2$ for our $\rho$ and $\rho_1$ (\Cref{app:clm:||Tr rho1 Phi'||2<= eC ||Tr rho Phi'||2}), we have $\Delta_1\lesssim (1 - \rho^2)\|\Tre_{\rho_1}\Phi'(z)\|_{L_2}^2\lesssim (1 - \rho^2)\|\Tre_{\rho}\Phi'(z)\|_{L_2}^2$.
For $\Delta_2$, applying \Cref{eq:lem:bound-the-differece-TrPhi-TrPhi(z/rho)} with $\rho_1$, and noting that $\|\Tre_{\rho_1}\Phi'(z)\|_{L_2}^2\lesssim\|\Tre_{\rho}\Phi'(z)\|_{L_2}^2$ (\Cref{app:clm:||Tr rho1 Phi'||2<= eC ||Tr rho Phi'||2}), then combining with \Cref{eq:lem:upper-bound-centered-smoothed-derivative-squared}, we obtain: $\Delta_2\lesssim (1 - \rho)^2 \|\Tre_{\rho}\Phi'(z)\|_{L_2}^2.$
Finally, for $\Delta_3$, using \Cref{eq:lem:upper-bound-centered-smoothed-error} and \Cref{eq:lem:upper-bound-centered-smoothed-derivative-squared} from \Cref{lem:center-aug-all-results}, and plugging in the value of $\rho_1$, we get $\Delta_3\leq {4(1 - \rho_1^2)}/{\rho_1^2}\|\Tre_{\rho_1}\Phi'(z/\rho_1)\|_{L_2}^2\lesssim (1 - \rho^2)\|\Tr\Phi'(z)\|_{L_2}^2$.
\end{proof}

\vspace{-0.3cm}

\subsection{{Initialization Algorithm and }Proof of \Cref{prop:initialization}}\label{subsec:initialization-for-monotone}

In this section, we provide an initialization algorithm for $\sigma$  that is a monotone $(B,L)$-Regular activation. The algorithm generates a vector $\w^{(0)}$ satisfying $\theta(\w^{(0)},\w^*)\leq C/M$, where $C$ is an absolute constant and $M \leq  \sqrt{\log(B/\eps) - \log\log(B/\eps)}$. Our key idea is to convert the regression problem to a problem of robustly learning halfspaces via {\itshape data transformation}. In particular, we transform $y$ to $\tilde{y}\in\{0,1\}$ by truncating the labels $y$ to $\tilde{y} = \1\{y\geq t'\}$, where this $t'$ is a  carefully chosen threshold. Then, we show that there exists a halfspace $\phi(\w^*\cdot\x; t) = \1\{\w^*\cdot\x\geq t\}$ such that the transformed labels $\tilde{y}$ can be viewed as the corrupted labels of $\phi(\w^*\cdot\x;t)$. Finally we utilize a previous algorithm from~\citet{DKTZ22b} to robustly learn $\w^*$. 
In particular, we show:
\begin{proposition}\label{prop:initial2}
Let $\sigma$ be a non-decreasing $(B,L)$-Regular function. Let $M$ be defined as in \Cref{app:claim:bounded-support-sigma'}. Then, there exists an algorithm that draws $O(d/\eps^2\log(1/\delta))$ samples, it runs in $\poly(d,N)$ time, and, with probability at least $1-\delta$, it outputs a vector $\vec w$ such  that $\theta(\vec w,\wstar)\leq C/M$, where $C>0$ is a universal constant, independent of any problem parameters.
\end{proposition}
We defer the proof of \Cref{prop:initial2} to \Cref{app:subsec:initialization-for-monotone}. 
\begin{proof}[Proof of \Cref{prop:initialization}]
    \Cref{prop:initial2} implies that there exists an algorithm that uses $O(d/\eps^2)$ samples and outputs a vector $\w^{(0)}$ such that $\theta(\w^{(0)},\w^*)\leq C/M$. Now for any $\theta\leq \theta_0$, it holds $\cos\theta^2\geq 1 - \theta^2\geq 1 - C^2/M^2$. Thus, using \Cref{prop:error-bound-smoothing-tails}, we have $\|\Pm{>1/\theta^2}\sigma\|_{L_2}^2\lesssim \sin^2\theta\|\Tre_{\cos\theta}\sigma'\|_{L_2}^2$. \end{proof}

\section{Conclusions and Open Problems} \label{sec:concl}
In this work, we give a constant-factor approximate robust learner 
for monotone GLMs under the Gaussian distribution, 
answering a recognized open problem in the field. 
A number of open questions remain. An immediate goal 
is to generalize our algorithmic result to Single-Index Models (SIMs), 
corresponding to the case where the monotone activation is unknown. 
We believe that progress in this direction is attainable. 
Another question is whether one can obtain a similarly robust GLM learner 
(even for the known activation case)
for more general marginal distributions, 
e.g., encompassing all isotropic log-concave distributions. 
This remains open even for the special case 
of a single general (i.e., potentially-biased) halfspace, 
where known constant-factor approximate learners~\citep{DKS18a,DKTZ22b} 
make essential use of the Gaussian assumption.

\bibliographystyle{plainnat}
\bibpunct{(}{)}{;}{a}{,}{,}

\bibliography{mydb}

\begin{thebibliography}{42}
\providecommand{\natexlab}[1]{#1}
\providecommand{\url}[1]{\texttt{#1}}
\expandafter\ifx\csname urlstyle\endcsname\relax
  \providecommand{\doi}[1]{doi: #1}\else
  \providecommand{\doi}{doi: \begingroup \urlstyle{rm}\Url}\fi

\bibitem[Awasthi et~al.(2023)Awasthi, Tang, and Vijayaraghavan]{ATV22}
P.~Awasthi, A.~Tang, and A.~Vijayaraghavan.
\newblock Agnostic learning of general {ReLU} activation using gradient
  descent.
\newblock In \emph{The Eleventh International Conference on Learning
  Representations, {ICLR}}, 2023.

\bibitem[Bogachev(1998)]{Bog:98}
V.~Bogachev.
\newblock \emph{Gaussian measures}.
\newblock Mathematical surveys and monographs, vol. 62, 1998.

\bibitem[Bubeck et~al.(2019)Bubeck, Jiang, Lee, Li, and
  Sidford]{bubeck2019complexity}
S.~Bubeck, Q.~Jiang, Y.-T. Lee, Y.~Li, and A.~Sidford.
\newblock Complexity of highly parallel non-smooth convex optimization.
\newblock \emph{Advances in neural information processing systems}, 32, 2019.

\bibitem[Clarke(1990)]{clarke1990optimization}
F.~H Clarke.
\newblock \emph{Optimization and nonsmooth analysis}.
\newblock SIAM, 1990.

\bibitem[Damian et~al.(2023)Damian, Nichani, Ge, and Lee]{damian2023smoothing}
A.~Damian, E.~Nichani, R.~Ge, and J.~D. Lee.
\newblock Smoothing the landscape boosts the signal for sgd: Optimal sample
  complexity for learning single index models.
\newblock \emph{Advances in Neural Information Processing Systems}, 36, 2023.

\bibitem[Diakonikolas et~al.(2018)Diakonikolas, Kane, and Stewart]{DKS18a}
I.~Diakonikolas, D.~M. Kane, and A.~Stewart.
\newblock Learning geometric concepts with nasty noise.
\newblock In \emph{Proceedings of the 50\textsuperscript{th} Annual {ACM}
  {SIGACT} Symposium on Theory of Computing, {STOC} 2018}, pages 1061--1073,
  2018.

\bibitem[Diakonikolas et~al.(2020{\natexlab{a}})Diakonikolas, Goel, Karmalkar,
  Klivans, and Soltanolkotabi]{DGKKS20}
I.~Diakonikolas, S.~Goel, S.~Karmalkar, A.~R. Klivans, and M.~Soltanolkotabi.
\newblock Approximation schemes for {ReLU} regression.
\newblock In \emph{Conference on Learning Theory, {COLT}}, volume 125 of
  \emph{Proceedings of Machine Learning Research}, pages 1452--1485. {PMLR},
  2020{\natexlab{a}}.

\bibitem[Diakonikolas et~al.(2020{\natexlab{b}})Diakonikolas, Kane, and
  Zarifis]{DKZ20}
I.~Diakonikolas, D.~M. Kane, and N.~Zarifis.
\newblock Near-optimal {SQ} lower bounds for agnostically learning halfspaces
  and {ReLUs} under {G}aussian marginals.
\newblock In \emph{Advances in Neural Information Processing Systems,
  {NeurIPS}}, 2020{\natexlab{b}}.

\bibitem[Diakonikolas et~al.(2021)Diakonikolas, Kane, Pittas, and
  Zarifis]{DKPZ21}
I.~Diakonikolas, D.~M. Kane, T.~Pittas, and N.~Zarifis.
\newblock The optimality of polynomial regression for agnostic learning under
  {Gaussian} marginals in the {SQ} model.
\newblock In \emph{Proceedings of The 34\textsuperscript{th} Conference on
  Learning Theory, {COLT}}, 2021.

\bibitem[Diakonikolas et~al.(2022{\natexlab{a}})Diakonikolas, Kane, Manurangsi,
  and Ren]{DKMR22}
I.~Diakonikolas, D.~Kane, P.~Manurangsi, and L.~Ren.
\newblock Hardness of learning a single neuron with adversarial label noise.
\newblock In \emph{Proceedings of the 25th International Conference on
  Artificial Intelligence and Statistics (AISTATS)}, 2022{\natexlab{a}}.

\bibitem[Diakonikolas et~al.(2022{\natexlab{b}})Diakonikolas, Kontonis, Tzamos,
  and Zarifis]{DKTZ22}
I.~Diakonikolas, V.~Kontonis, C.~Tzamos, and N.~Zarifis.
\newblock Learning a single neuron with adversarial label noise via gradient
  descent.
\newblock In \emph{Conference on Learning Theory (COLT)}, pages 4313--4361,
  2022{\natexlab{b}}.

\bibitem[Diakonikolas et~al.(2022{\natexlab{c}})Diakonikolas, Kontonis, Tzamos,
  and Zarifis]{DKTZ22b}
I.~Diakonikolas, V.~Kontonis, C.~Tzamos, and N.~Zarifis.
\newblock Learning general halfspaces with adversarial label noise via online
  gradient descent.
\newblock In Kamalika Chaudhuri, Stefanie Jegelka, Le~Song, Csaba Szepesvari,
  Gang Niu, and Sivan Sabato, editors, \emph{Proceedings of the 39th
  International Conference on Machine Learning}, volume 162 of
  \emph{Proceedings of Machine Learning Research}, pages 5118--5141. PMLR,
  17--23 Jul 2022{\natexlab{c}}.

\bibitem[Diakonikolas et~al.(2023)Diakonikolas, Kane, and Ren]{DKR23}
I.~Diakonikolas, D.~M. Kane, and L.~Ren.
\newblock Near-optimal cryptographic hardness of agnostically learning
  halfspaces and {ReLU} regression under {Gaussian} marginals.
\newblock In \emph{ICML}, 2023.

\bibitem[Diakonikolas and Guzm{\'a}n(2024)]{diakonikolas2024optimization}
J.~Diakonikolas and C.~Guzm{\'a}n.
\newblock Optimization on a finer scale: Bounded local subgradient variation
  perspective.
\newblock \emph{arXiv preprint arXiv:2403.16317}, 2024.

\bibitem[Dobson and Barnett(2008)]{DobB08}
A.~J. Dobson and A.~G. Barnett.
\newblock \emph{{An Introduction to Generalized Linear Models}}.
\newblock Chapman and Hall/CRC, 3 edition, May 2008.
\newblock ISBN 1584889500.

\bibitem[Duchi et~al.(2012)Duchi, Bartlett, and
  Wainwright]{duchi2012randomized}
J.~C. Duchi, P.~L. Bartlett, and M.~J. Wainwright.
\newblock Randomized smoothing for stochastic optimization.
\newblock \emph{SIAM Journal on Optimization}, 22\penalty0 (2):\penalty0
  674--701, 2012.

\bibitem[Federer(1969)]{federer1969geometric}
H.~Federer.
\newblock \emph{Geometric Measure Theory}.
\newblock Die Grundlehren der mathematischen Wissenschaften in
  Einzeldarstellungen. Springer, 1969.
\newblock ISBN 9780387045054.

\bibitem[Goel et~al.(2020)Goel, Gollakota, and Klivans]{GGK20}
S.~Goel, A.~Gollakota, and A.~R. Klivans.
\newblock Statistical-query lower bounds via functional gradients.
\newblock In \emph{Advances in Neural Information Processing Systems,
  {NeurIPS}}, 2020.

\bibitem[Gollakota et~al.(2023{\natexlab{a}})Gollakota, Gopalan, Klivans, and
  Stavropoulos]{GGKS23}
A.~Gollakota, P.~Gopalan, A.~R. Klivans, and K.~Stavropoulos.
\newblock Agnostically learning single-index models using omnipredictors.
\newblock In \emph{Thirty-seventh Conference on Neural Information Processing
  Systems}, 2023{\natexlab{a}}.

\bibitem[Gollakota et~al.(2023{\natexlab{b}})Gollakota, Gopalan, Klivans, and
  Stavropoulos]{Gollakota2023d}
A.~Gollakota, P.~Gopalan, A.~R. Klivans, and K.~Stavropoulos.
\newblock Agnostically learning single-index models using omnipredictors.
\newblock In \emph{Thirty-seventh Conference on Neural Information Processing
  Systems}, 2023{\natexlab{b}}.

\bibitem[Guo and Vijayaraghavan(2024)]{guo2024agnostic}
A.~Guo and A.~Vijayaraghavan.
\newblock Agnostic learning of arbitrary {ReLU} activation under {G}aussian
  marginals.
\newblock \emph{arXiv preprint arXiv:2411.14349}, 2024.

\bibitem[Haussler(1992)]{Haussler:92}
D.~Haussler.
\newblock {Decision theoretic generalizations of the PAC model for neural net
  and other learning applications}.
\newblock \emph{Information and Computation}, 100:\penalty0 78--150, 1992.

\bibitem[Hu et~al.(2024)Hu, Tian, and Yang]{hu2024SimsOmni}
L.~Hu, K.~Tian, and C.~Yang.
\newblock Omnipredicting single-index models with multi-index models.
\newblock \emph{arXiv preprint arXiv:2411.13083}, 2024.

\bibitem[Kakade et~al.(2011)Kakade, Kanade, Shamir, and
  Kalai]{kakade2011efficient}
S.~M. Kakade, V.~Kanade, O.~Shamir, and A.~Kalai.
\newblock Efficient learning of generalized linear and single index models with
  isotonic regression.
\newblock \emph{Advances in Neural Information Processing Systems}, 24, 2011.

\bibitem[Kalai and Sastry(2009)]{kalai2009isotron}
A.~T. Kalai and R.~Sastry.
\newblock The isotron algorithm: High-dimensional isotonic regression.
\newblock In \emph{COLT}, 2009.

\bibitem[Kearns et~al.(1994)Kearns, Schapire, and Sellie]{KSS:94}
M.~Kearns, R.~Schapire, and L.~Sellie.
\newblock Toward efficient agnostic learning.
\newblock \emph{Machine Learning}, 17\penalty0 (2/3):\penalty0 115--141, 1994.

\bibitem[Klivans et~al.(2008)Klivans, O'Donnell, and Servedio]{KOS:08}
A.~Klivans, R.~O'Donnell, and R.~Servedio.
\newblock Learning geometric concepts via {G}aussian surface area.
\newblock In \emph{Proc.\ 49th IEEE Symposium on Foundations of Computer
  Science (FOCS)}, pages 541--550, Philadelphia, Pennsylvania, 2008.

\bibitem[Lin et~al.(2024)Lin, Kaushik, Dyer, and Muthukumar]{lin2024good}
C.-H. Lin, C.~Kaushik, E.~L. Dyer, and V.~Muthukumar.
\newblock The good, the bad and the ugly sides of data augmentation: An
  implicit spectral regularization perspective.
\newblock \emph{Journal of Machine Learning Research}, 25\penalty0
  (91):\penalty0 1--85, 2024.

\bibitem[Manurangsi and Reichman(2018)]{MR18}
P.~Manurangsi and D.~Reichman.
\newblock The computational complexity of training {ReLU}(s).
\newblock \emph{arXiv preprint arXiv:1810.04207}, 2018.

\bibitem[Mei et~al.(2018)Mei, Bai, and Montanari]{mei2018landscape}
S.~Mei, Y.~Bai, and A.~Montanari.
\newblock The landscape of empirical risk for nonconvex losses.
\newblock \emph{The Annals of Statistics}, 46\penalty0 (6A):\penalty0
  2747--2774, 2018.

\bibitem[Nelder and Wedderburn(1972)]{NW72}
J.~A. Nelder and R.~W.~M. Wedderburn.
\newblock Generalized linear models.
\newblock \emph{Royal Statistical Society. Journal. Series A: General},
  135\penalty0 (3):\penalty0 370--384, 1972.
\newblock ISSN 0035-9238.
\newblock \doi{10.2307/2344614}.
\newblock URL \url{https://doi.org/10.2307/2344614}.

\bibitem[Nesterov and Spokoiny(2017)]{Nesterov2017}
Y.~Nesterov and V.~Spokoiny.
\newblock Random gradient-free minimization of convex functions.
\newblock \emph{Foundations of Computational Mathematics}, 17:\penalty0
  527--566, 2017.

\bibitem[O'Donnell(2014)]{ODonnell:BFA}
R.~O'Donnell.
\newblock \emph{Analysis of Boolean Functions}.
\newblock Cambridge University Press, 2014.

\bibitem[Pang(1997)]{Pang1997}
J.-S. Pang.
\newblock Error bounds in mathematical programming.
\newblock \emph{Mathematical Programming}, 79\penalty0 (1):\penalty0 299--332,
  1997.

\bibitem[Rosenblatt(1958)]{R58}
F.~Rosenblatt.
\newblock The perceptron: A probabilistic model for information storage and
  organization in the brain.
\newblock \emph{Psychological Review}, 65(6):\penalty0 386--408, 1958.

\bibitem[Song et~al.(2021)Song, Zadik, and Bruna]{Song2021}
M.~J. Song, I.~Zadik, and J.~Bruna.
\newblock On the cryptographic hardness of learning single periodic neurons.
\newblock In \emph{Advances in Neural Information Processing Systems,
  {NeurIPS}}, 2021.

\bibitem[Stein(1981)]{Stein1981}
C.~M. Stein.
\newblock Estimation of the mean of a multivariate normal distribution.
\newblock \emph{The Annals of Statistics}, 9\penalty0 (6):\penalty0 1135--1151,
  1981.

\bibitem[Vershynin(2018)]{vershynin2018Bookhigh}
R.~Vershynin.
\newblock \emph{High-dimensional probability: An introduction with applications
  in data science}, volume~47.
\newblock Cambridge university press, 2018.

\bibitem[Wang et~al.(2023)Wang, Zarifis, Diakonikolas, and
  Diakonikolas]{WZDD23}
P.~Wang, N.~Zarifis, I.~Diakonikolas, and J.~Diakonikolas.
\newblock Robustly learning a single neuron via sharpness.
\newblock \emph{40th International Conference on Machine Learning}, 2023.

\bibitem[Wang et~al.(2024)Wang, Zarifis, Diakonikolas, and
  Diakonikolas]{WZDD2024sample}
P.~Wang, N.~Zarifis, I.~Diakonikolas, and J.~Diakonikolas.
\newblock Sample and computationally efficient robust learning of gaussian
  single-index models.
\newblock \emph{The Thirty-Eighth Annual Conference on Neural Information
  Processing Systems}, 2024.

\bibitem[Yin et~al.(2019)Yin, Gontijo~Lopes, Shlens, Cubuk, and
  Gilmer]{yin2019fourier}
D.~Yin, R.~Gontijo~Lopes, J.~Shlens, E.~D. Cubuk, and J.~Gilmer.
\newblock A fourier perspective on model robustness in computer vision.
\newblock \emph{Advances in Neural Information Processing Systems}, 32, 2019.

\bibitem[Zarifis et~al.(2024)Zarifis, Wang, Diakonikolas, and
  Diakonikolas]{ZWDD2024}
N.~Zarifis, P.~Wang, I.~Diakonikolas, and J.~Diakonikolas.
\newblock Robustly learning single-index models via alignment sharpness.
\newblock In \emph{Proceedings of the 41st International Conference on Machine
  Learning}, volume 235 of \emph{Proceedings of Machine Learning Research},
  pages 58197--58243. PMLR, 21--27 Jul 2024.

\end{thebibliography}

\newpage
\appendix
\section*{Appendix}

\paragraph{Organization} The appendix is organized as follows.  In \Cref{app:sec:compare}, we provide a detailed comparison with related prior works. In \Cref{app:sec:prelims}, we give additional background 
on the \OU semigroup and introduce useful facts that will repeatedly appear in the technical sections. In \Cref{app:function-class}, we provide detailed discussions on the (Extended)-$(B,L)$-activation class and our assumptions. In \Cref{app:sec:augmentation-and-landscape}, \Cref{app:sec:learn-sim-general}, and \Cref{app:sec:learn-monotone} we provide the full versions of \Cref{sec:augmentation-and-landscape}, \Cref{sec:learn-sim-general}, \Cref{sec:learn-monotone}, with complete proofs and supplementary lemmas. 

\section{Detailed Comparison with Prior Work}\label{app:sec:compare}

In this section, we provide a detailed comparison with related prior works.

\begin{table}[h]
\centering
\resizebox{\textwidth}{!}{
\begin{tabular}{|c|c|c|c|c|}
\hline
\textbf{}  & \textbf{Distribution} & \textbf{Activation}& \textbf{Error Bound} \\
\hline
\textbf{[WZDD23]}  & Well-Behaved & Monotonic $(a,b)$-unbounded & $O(\poly(b/a))\opt$ \\
\hline
\textbf{[WZDD24]}  & Gaussian & \( k^* \)-information exponent & \( O(\|\sigma'\|_{L_2})\opt \) \\
\hline
\textbf{[GV2024]}  & Gaussian & Biased ReLUs & $C\opt$\\
\hline
\textbf{Ours}  & Gaussian & Monotone + Lipschitz or Bounded $(2+\zeta)$ Moment  & $C\opt$ \\
\hline
\end{tabular}}
\caption{Comparison of our approach with prior work on robustly learning GLMs. }
\label{tab:comparison}
\end{table}

\citet{DKTZ22,WZDD23,ZWDD2024} studied agnostic learning of GLMs under `well-behaved' distributions, where $\sigma$, possibly not known a priori, is monotone and $(a,b)$-unbounded, meaning that $|\sigma'(z)| \leq b$ and $\sigma'(z) \geq a$ when $z \geq 0$. They provided an algorithm that finds $\widehat{\w} \in \mathbb{B}(W)$ with error $O(\poly(b/a))\opt + \eps$.
Note that in these works, rescaling $\w$ to $\mathbb{S}^{d-1}$ is not required; therefore, $a, b$ do not have dependencies on the parameter $W$. However, the main drawback of these works is that their algorithm cannot be applied to all monotone and Lipschitz functions. In particular, when $a = 0$, the previous works do not provide any useful results at all. Furthermore, if $a = O(\eps)$, the algorithms in \citet{DKTZ22,WZDD23,ZWDD2024} only provide an approximate solution with $O({\rm poly}(1/\eps))\opt$ error.
In stark comparison, in our work, we can deal with any $b$-Lipschitz activations and obtain $C \opt + \eps$ error, where the absolute constant $C$ does not depend on $b$, $\eps$, or $W$, as shown in \Cref{thm:main-monotone-b-lip}.

\citet{WZDD2024sample} studied robust learning of GLMs under Gaussian marginals, 
similar to our setting. They considered a broader class of activations 
where $\sigma$ has constant information exponent $k^*$, 
defined as the degree of the first non-zero Hermite coefficient: 
$\sigma(z) \doteq \sum_{k \geq 1} c_k \he_k(z)$, 
with $k^* = \min\{k \geq 1: c_k \neq 0\}$.
\citet{WZDD2024sample} makes the following assumptions: 
$\|\w\|_{2} = 1$, $\|\sigma\|_{L_2} = 1$, $\|\sigma\|_{L_4} \leq +\infty$, 
and that $c_{k^*}$ is an absolute constant.
Their algorithm requires $O(d^{\lceil k^*/2 \rceil}/c_{k^*} + d/\eps)$ samples and outputs $\wh{\w} \in \mathbb{S}^{d-1}$ with error $O(\|\sigma'\|_{L_2})\opt + \eps$. 

However, their approach has the following key limitations: 
(1) It does not generalize to $\w^*\in\B(W)$, 
as rescaling to $\mathbb{S}^{d-1}$ affects 
the gradient norm---leading to an error bound of $O(W \|\sigma'\|_{L_2})\opt$, 
which depends on $W$. 
(2) Rescaling $\sigma$ to satisfy $\|\sigma\|_{L_2} = 1$ 
can inadvertently amplify $\|\sigma'\|_{L_2}$, increasing the error.  
(3) Finally, note that their sample complexity depends on $c_{1}$, 
therefore their sample complexity can be even larger 
if $c_{1}$ is extremely small.

Our results address these issues: 
(1) as discussed in the introduction, this work's 
error bound in \Cref{thm:main-monotone} is independent 
of all the parameters 
$\|\sigma'\|_{L_2}$, $\|\sigma\|_{L_\infty}$, $d$ and $\eps$, 
and therefore rescaling 
the activation will not impact the approximation error; 
(2) similarly, the quantity  
$\|\sigma\|_{L_2}$ also does not impact our approximation error; 
(3) finally, our sample 
complexity is independent of $c_{1}$, and will therefore  
not be impacted if $c_1$ is very small.

Recent independent work \citep{guo2024agnostic} studied 
agnostic learning of biased ReLUs under Gaussian $\x$-marginals, 
also achieving $C\opt + \eps$ error. We note that 
their algorithm is tailored to the special case of ReLUs.  
On the other hand, our framework handles all monotone Lipschitz activations 
(including all biased ReLUs as a special case), 
and even all monotone activations with bounded ($2+\zeta$)-order 
moments for $\zeta > 0$; see \Cref{app:lem:activation-truncation-2+zeta}.

\cite{Gollakota2023d,hu2024SimsOmni} studied agnostic learning of GLMs 
with unknown activation $\sigma$. These works focused on general distributions: \cite{Gollakota2023d} only requires the marginal distribution of $\x$ 
to have its second moment bounded by $\lambda$; 
and \cite{hu2024SimsOmni} only requires $\x$ to be supported on a Euclidean ball.  
However, the error bounds that \cite{Gollakota2023d,hu2024SimsOmni} 
achieve cannot be considered constant factor approximations. 
\cite{Gollakota2023d} provides $O(W\sqrt{\lambda\opt})$ error guarantee 
for $1$-Lipschitz activations;   
their algorithm achieves $O(b/a)\opt + \eps$ 
error when restricted to $(a,b)$-bi-Lipschitz activations, 
i.e., for $0< a\leq \sigma'(z)\leq b$. 
\cite{hu2024SimsOmni} does not provide an $L_2^2$-error guarantee 
but instead focuses on finding an omnipredictor 
that minimizes a convex surrogate loss.

In \citet{damian2023smoothing}, the authors considered GLMs 
with bounded information exponent  
and employed a smoothing technique different than ours, 
with a constant smoothing parameter. Importantly, 
their algorithm is limited to the realizable setting.  
As explained in \citet{WZDD2024sample}, their algorithm 
fails in the more challenging robust learning setting, 
even for monotone functions (with information exponent $k^* = 1$).

Moreover, their smoothing approach differs from ours both conceptually and practically. 
Conceptually, as discussed in \Cref{sec:tech}, our method 
is based on the observation that the gradient of 
the augmented/\OU-semigroup-smoothed $L_2^2$ loss \textit{maximizes} 
the signal from $\w^*$, which is otherwise obscured by agnostic noise. 
In contrast, \citet{damian2023smoothing} applied a spherical smoothing 
technique aimed at capturing higher-order information 
and improving the ratio $\|\g(\w)\|_{2}/(\g(\w)\cdot\w^*)$, 
where $\g(\w) = \nabla \calL(\w)$. This is sufficient 
for the realizable setting, 
but not for the more challenging adversarial setting. 
Practically, our algorithm and techniques diverge significantly 
from those in \citet{damian2023smoothing}. First, whereas 
they implemented spherical smoothing, we utilize Gaussian noise injection 
while also reweighting the marginals $\x$. Second, instead of fixing 
the smoothing parameter, we employ variable augmentation/smoothing. 
This variable smoothing is crucial to our algorithm, 
as it ensures that the signal of the augmented gradient 
is not obscured by noise in each iteration 
(see the discussion and analysis in \Cref{thm:main-general}).

\citet{kalai2009isotron,kakade2011efficient} studied the problem of learning GLMs in the realizable setting. They considered monotone $1$-Lipschitz activations under any distribution $\D$ that is supported on $\B\times[0,1]$. 
Their analysis is not applicable to our robust learning setting.

\section{Additional Notation and Preliminaries}\label{app:sec:prelims}

\paragraph{Additional Notation}
 Let $\normal( \boldsymbol\mu, \vec \Sigma)$ denote the $d$-dimensional Gaussian distribution with mean $\boldsymbol\mu\in  \R^d$ and covariance $\vec \Sigma\in \R^{d\times d}$. In this work we usually consider the standard normal distribution, i.e., $\mu = \vec 0$ and $\vec \Sigma = \vec I$, and thus denote it simply by $\normal$. 
The usual
inner product for this Gaussian space is
$\E_{\vec x \sim \normal}[f(\vec x) g(\vec x)]$. {We write  $f(z)\doteq g(z)$ to mean that $\E_{z\sim\calN(0,1)}[(f(z) - g(z))^2] = 0$.}
We use \emph{normalized} probabilists' Hermite polynomial of degree $i$, defined via 
\(
\he_i(x) = {\hep_i(x)}/{\sqrt{i!}}, \)
where by $\hep_i(x)$ we denote the probabilist's Hermite polynomial of degree $i$: 
\begin{equation*}
    \hep_k(z) = {(-1)^k} \exp({z^2}/{2})\frac{\mathrm{d}^k}{\diff{z}^k}\exp(-{z^2}/{2}). 
\end{equation*}
These normalized Hermite polynomials form a complete orthonormal basis for the
single-dimensional version of the inner product space defined above.
Given a function $f:\R\to\R$, $f \in L_2(\normal)$, we compute its Hermite coefficients as
\(
 \hat{f}(i) = \E_{z\sim \normal(0,1)} [f(z) \he_i(z)],
\)
and express the function uniquely as
\(
f(z)\doteq\sum_{i\geq 0} \hat{f}(i) \he_i(z).
\)

\subsection{\OU Semigroup}\label{app:subsec:prelims-OU}

An important tool for our work is the \OU semigroup. The \OU semigroup and operators are broadly used in stochastic analysis and control theory (see, e.g., \citet{Bog:98}). Within learning theory, they have found applications in bounding the sensitivity of a Boolean function \citep{KOS:08}. A formal definition of the \OU semigroup is provided below.

\begin{definition}[\OU\ Semigroup]
Let $\rho \in (0,1)$. The \OU semigroup, denoted  by $\Tr$, is a linear operator that maps a function $g \in L_2(\normal)$ to the function 
$\Tr g$ defined as:
    \[
    (\Tr g) (\vec x) \eqdef\E_{\vec z\sim \normal}\left[g(\rho\x+\sqrt{1-\rho^2}\vec z)\right]\;.
    \]
    To simplify the notation, we often write 
    $\Tr g (\vec x) $ instead of  $(\Tr g) (\vec x) $.
\end{definition}

The following fact summarizes useful properties of the \OU semigroup. 
\begin{fact}[see, e.g.,~\cite{Bog:98},~\cite{ODonnell:BFA}(Chapter 11)]\label{fct:semi-group}\label{lem:interchange-Tr-and-differentiation}
Let $f,g\in L_2(\calN)$.
\begin{enumerate}
    \item For any $f,g\in L_2$ and any $t>0$, 
    $
   \E_{\x\sim\calN}[ (\Tre_{t}f(\x)) g(\x)]=\E_{\x\sim\calN}[ (\Tre_{t}g(\x)) f(\x)]\;.
    $
    \item For any $g: \R^d \to \R,$ $g\in L_2$, all of the following statements hold. 
     \begin{enumerate}
     \item For any $t,s>0$,  $   \Tre_{t}\Tre_{s}g =\Tre_{ts}g$.
        \item For any $\rho\in(0,1)$, $\Tr g(\x)$ is differentiable at every point $\x \in \R^d$.
        \item For any $\rho\in(0,1)$, $\Tre_{\rho} g(\x)$ is $\|g\|_{L_\infty}/(1-\rho^2)^{1/2}$-Lipschitz, i.e., $\|\nabla \Tre_{\rho} g(\x)\|_{L_\infty}\leq \|g\|_{L_\infty}/(1-\rho^2)^{1/2}$, $\forall \x \in \R^d$.
        \item For any $\rho\in (0,1)$, $\Tr g(\x)\in \mathcal{C}^{\infty}$.
        \item For any $p\geq 1$, $\Tr$ is nonexpansive with respect to the norm $\|\cdot\|_{L_p}$, i.e., $\|\Tr g\|_{L_p}\leq \|g\|_{L_p}$.
        \item $\|\Tr g(
        \x)\|_{L_2}$ is  non-decreasing w.r.t.\ $\rho$.
        \item If $g$ is, in addition, a differentiable function, then  for all $\rho \in (0,1)$, it holds that:
    \(
    \nabla_{\x}\Tr g(\x)=\rho\Tr \nabla_{\x}g(\x)
    \), for any $\x\in \R^d$.
    \end{enumerate}
    \item For all $\rho\in(0,1)$ and $i\in \Z_+$,  $\Tr\he_i(z) = \rho^i\he_i(z)$.
\end{enumerate}
\end{fact}

The \OU semigroup induces an operator $\mathrm{L}$ applying to functions $f \in L_2(\normal)$, defined below.  \begin{definition}[Definition 11.24 in \cite{ODonnell:BFA}]\label{def:ou-operator-ind}
    The \OU operator is a linear operator applied that applies to functions $f\in L_2(\normal)$ and is defined by $\mathrm{L}f = \frac{\d \Tr f}{\d \rho}\mid_{\rho=1}$, provided that $\mathrm{L}f $ exists. 
\end{definition}

\begin{fact}\label{fact:OU-op} Let $f,g \in L_2(\normal)$, $\rho\in(0,1)$. Then:
\begin{enumerate}
    \item {(\cite[Proposition 11.27]{ODonnell:BFA})} $\frac{\d \Tr f}{\d \rho}=\frac{1}{\rho}\mathrm{L}\Tr f=\frac{1}{\rho}\Tr \mathrm{L}f $.
    \item (\cite[Proposition 11.28]{ODonnell:BFA}) $ \E_{\x \sim \normal}\left[f(\x) \mathrm{L}\Tre_{\rho} g(\x) \right]=\E_{\x \sim \normal}\left[\nabla f(\x) \nabla\Tre_{\rho} g(\x) \right]$ .
\end{enumerate}
\end{fact}

We use \Cref{fact:OU-op} to prove the following \Cref{clm:difference}:
\begin{restatable}{lemma}{DifferenceTrfAndfBoundByGrad}\label{clm:difference} Let $f \in L_2(\normal)$ be a continuous and  (almost everywhere) differentiable function.  Then $\E_{\x \sim \normal}[( \Tr f(\x) - f(\x) )^2]\leq 3 (1-\rho) \E_{\x \sim \normal}[\|\nabla f(\x)\|_2^2] $.
 \end{restatable}
\begin{proof}
    Observe that  $\left(\E_{\x \sim \normal}[( \Tr f(\x) - f(\x) )^2]\right)^{1/2}=\sup_{g\in \mathcal C^{\infty},\|g\|_{L_2}\leq 1}\E_{\x \sim \normal}[g(\x)( \Tr f(\x) - f(\x) )]$. Consider any $g\in \mathcal C^{\infty}$ with $\|g\|_{L_2}\leq 1$. We have that
    \begin{align*}
        \E_{\x \sim \normal}[g(\x)( \Tr f(\x) - f(\x) )]=\E_{\x \sim \normal}[f(\x)( \Tr g(\x) - g(\x) )]=\E_{\x \sim \normal}\left[f(\x)\int_{\rho}^1 \frac{\d \Tre_{t} g(\x)}{\d t} \d t\right].
    \end{align*}
    As $g\in \mathcal C^{\infty}$, we can use \Cref{fact:OU-op} to conclude that
    \[
    \E_{\x \sim \normal}\left[f(\x)\int_{\rho}^1 \frac{\d \Tre_{t} g(\x)}{\d t} \d t\right]=\E_{\x \sim \normal}\left[f(\x)\int_{\rho}^1 (1/t)\mathrm{L}\Tre_{t} g(\x) \d t\right],
    \]
    where $ \mathrm{L}$ is the \OU operator. Using the identity that for $f$ such that $\nabla f\in L_2(\normal)$ and $g\in \mathcal C^{\infty}$ it holds that  $ \E_{\x \sim \normal}\left[f(\x) \mathrm{L}\Tre_{t} g(\x) \right]=\E_{\x \sim \normal}\left[\nabla f(\x) \nabla\Tre_{t} g(\x) \right]$ (\Cref{fact:OU-op}), we have that
        \begin{align*}
        \E_{\x \sim \normal}[g(\x)( \Tr f(\x) - f(\x) )]=\int_{\rho}^1 (1/t)\E_{\x \sim \normal}\left[\nabla f(\x) \nabla\Tre_{t} g(\x) \right]\d t .
    \end{align*}
    Note that using \Cref{fct:semi-group} (f) and Stein's lemma \Cref{fct:stein}, we have:
    \begin{align*}
        \nabla\Tr g(\x) &= \rho\Tr\nabla g(\x) = \rho\Ez[\nabla g(\rho\x + \sqrt{1 -\rho^2}\bz)]\\
        &=
        (\rho/({\sqrt{1-\rho^2}}))\E_{\vec z\sim \normal}[g(\rho \x + \sqrt{1-\rho^2}\vec z)\vec z].
    \end{align*}
Therefore, since $\bz,\x$ are independent standard Gaussian random vectors, we have that 
    \begin{align*}
        &\quad \E_{\x \sim \normal}[g(\x)( \Tr f(\x) - f(\x) )]\\
        &=\int_{\rho}^1 \frac{1}{\sqrt{1-t^2}}\E_{\x,\bz \sim \normal}\left[g(t \x + \sqrt{1-t^2}\vec z)\vec z\cdot\nabla f(\x) \right]\d t
        \\
        &\leq \int_{\rho}^1 \frac{1}{\sqrt{1-t^2}}\left(\E_{\x, \bz \sim \normal}\left[g(t \x + \sqrt{1-t^2}\vec z)^2\right]\E_{\x,\bz \sim \normal}\left[(\bz\cdot\nabla f(\x))^2\right]\right)^{1/2}\d t\\
        &= \int_{\rho}^1 \frac{1}{\sqrt{1-t^2}}\left(\E_{\bu\sim \normal}\left[g(\bu)^2\right]\E_{\x \sim \normal}\left[\|\nabla f(\x) \|_2^2\right]\right)^{1/2}\d t
        \\&\leq \left(\E_{\x \sim \normal}\left[\|\nabla f(\x) \|_2^2\right]\right)^{1/2} \int_{\rho}^1 \frac{1}{\sqrt{1-t^2}}\d t=\left(\E_{\x \sim \normal}\left[\|\nabla f(\x) \|_2^2\right]\right)^{1/2}\arccos{\rho}\;,
    \end{align*}
    where we used the \CS inequality and the fact that $\|g\|_{L_2}\leq 1$.
    Using the inequality $\arccos{\rho}\leq \sqrt{3}\sqrt{1-\rho}$, we complete the proof of \Cref{clm:difference}.
\end{proof}

\subsection{Gaussian Distribution and Hermite Polynomials}

In the following fact, we gather some useful facts about Hermite polynomials that are used throughout the paper.
\begin{fact}[See, e.g., \cite{Bog:98}]\label{fct:hermite}
The following statements hold for Hermite polynomials as defined above. 
\begin{enumerate}
    \item (Parseval's identity) For any $f\in L_2(\normal)$, we have 
    \(
  \E_{z \sim \normal(0,1)}[ \lp(f(z) - \Pm{k}f(z) \rp)^2 ]
  = \sum_{i = k+1}^{\infty} \hat{f}(i)^2
\).
\item (Mehler's Identity)
    For any real number $|\rho|<1$ and $x,y\in\R$, it holds
    \begin{equation}\label{fct:Mehler}
        \sum_{k\geq 0}\rho^k\he_k(x)\he_k(y) = \frac{1}{\sqrt{1 - \rho^2}}\exp\bigg(-\frac{(\rho x - y)^2}{2(1 - \rho^2)} + \frac{y^2}{2}\bigg).
    \end{equation}
\item (Differentiation) $(\he_i(z))' = \sqrt{i}\he_{i-1}(z)$.
\end{enumerate}

\end{fact}

Finally, the following facts about Gaussian distribution are useful to our paper:
\begin{fact}[Stein's Lemma~\citep{Stein1981}] \label{fct:stein}Suppose that $\x$ is distributed as $\mathcal{N}(\bm{\mu},\sigma^2 \vec I)$ for some $\bm{\mu}\in\R^d,\sigma\in \R_+$ and let $g:\R^d\to\R$ be an almost everywhere differentiable function such that both $\E[g(\x)\x]$ and $\E[\nabla g(\x)]$ exist. Then, it holds
\[
\E[g(\x)(\x-{\bm{\mu}})]=\sigma^2\E[\nabla g(\x)]\;.
\]    
\end{fact}

\begin{fact}[Komatsu's Inequality]\label{fct:komatsu}
For any $t\geq 0$ it holds:
\[
C\frac{\exp(-t^2/2)}{t+\sqrt{t^2+4}}<\pr_{x\sim \mathcal N(0,1)}[x\geq t]< C\frac{\exp(-t^2/2)}{t+\sqrt{t^2+2}}\;,
\]
    where $C>0$ is a universal constant. 
\end{fact}

\section{Discussion on Regular Activations}\label{app:function-class}

Let us first recall the definitions of the (Extended-)$(B,L)$-Regular activations.
\begin{definition}[$(B,L)$-Regular Activations]Given parameters $B,L>0$, we define the class of $(B,L)$-Regular activations, denoted by $\mathcal{H}(B,L)$, as the class containing all functions $\sigma:\R\to\R$ such that 1) $\|\sigma\|_{L_\infty}\leq B$ and 2) $\|\sigma'\|_{L_2}\leq L$.
   
 Given $\eps>0$, we define the class of $\eps$-Extended $(B,L)$-Regular activations, denoted by  $\mathcal{H}_\epsilon(B,L)$, as the class containing all activations $\sigma_1: \R\to \R$ for which there exists $\sigma_2\in \mathcal{H}(B,L)$ such that $\|\sigma_1-\sigma_2\|_{L_2}^2\leq\eps$. 
\end{definition}

We remark that our algorithm can be applied to many non-differentiable activations.
\begin{remark}[On Differentiability]\label{rmk:differentiability}
{\em In the definition of (Extended-)$(B,L)$-Regular activations, the differentiability of $\sigma$ is required. However, this restriction can be relaxed for any activation that is a locally-Lipschitz\footnote{We say a function $\sigma$ is locally-Lipschitz if for any $z_0\in\R$, there exists positive reals $b$ and $\delta$ such that for any $z\in[z_0-\delta,z_0+\delta]$, it holds $|\sigma(z) - \sigma(z_0)|\leq b|z-z_0|$.} function, since they are differentiable almost-everywhere \citep{federer1969geometric}. Therefore, since the set of non-differentiable points is measure-zero, we can define the derivative of $\sigma$ at those non-differentiable points freely (for example, using Clarke Differentials \citep{clarke1990optimization}).

Furthermore, our results can also be applied to functions that are not even locally-Lipschitz. In particular, functions that have finite `monotone jumps' like $\sigma(z) = \sgn(z-t)$ are subject to our results (\Cref{app:lem:activation-truncation-lip-and-mon-jump}). 
 
In fact, the set of smoothed functions are dense to our functions (i.e., there always exists a $\rho\in(0,1)$ such that $\|\sigma-\Tr\sigma\|_{L_2}^2\leq \eps$). Therefore, statistically, there is no difference in using either of the functions.}
\end{remark}

\subsection{Rescaling to the Unit Sphere}

Next, we comment on the impact of rescaling the activation $\sigma$.
\begin{remark}[Rescaling the Parameter]\label{app:rmk:rescale-w}
 {\em   Let $\sigma$ be a monotone (Extended-)$(B,L)$-regular activation. In our approach, it is without loss of generality to assume
 that $\|\w\|_2 = 1$.   This is because, for any nonzero vector $\w\in\B(W)$, 
 we can always rescale the activation $\sigma(\w\cdot\x)$ to $\sigma(\|\w\|_2(\w/\|\w\|_2)\cdot\x)\eqdef \bar{\sigma}((\w/\|\w\|_2)\cdot\x) =\bar{\sigma}(\w'\cdot\x)$ where $\|\w'\|_2 = 1$. 
 In other words, we define $\bar{\sigma}(z) = \sigma(\|\w^*\|_2 z)$, 
 where $\w^*$ is one of the target vectors.
 
 After rescaling, 
the second moment of $\bar{\sigma}'$ increases to 
 $\|\bar{\sigma}'\|_{L_2}\leq \|\w^*\|_2 \|\sigma'(\|\w^*\|_2 z)\|_{L_2}$ (which can be further bounded by using that $\sigma$ is close to a function $\hat{\sigma}$ with $\|\hat{\sigma}'\|_{\infty}\leq \|\w^*\|_2 B/\eps$).
 Therefore, 
 the parameter $L$ can potentially scale with $W$. 
 
 {For instance, if $\sigma$ is a $b$-Lipschitz activation, then the derivative of the rescaled function satisfies $|\bar{\sigma}'(z)| = \|\w^*\|_{2}|\sigma'(\|\w^*\|_{2}z)|\leq \|\w^*\|_2 b$ meaning that $\bar{\sigma}$ effectively becomes a $Wb$-Lipschitz activation. }
 However, we emphasize that our approximation error 
 obtained in \Cref{thm:main-general} does not scale 
 with any of these parameters $B,L$. 
These parameters only influence the sample complexity and runtime of our algorithm in a polynomial manner.}
\end{remark}

In the following lemma, we show that without loss of generality we can assume that we know the norm of the unknown vector $\wstar$ as we can reduce the problem into testing $1/\poly(\eps,1/L,1/B)$  different values for the norm.
\begin{lemma}\label{app:prop:wlog-konw-norm-wstr}
      Fix $\eps>0$ and $\delta\in(0,1)$. Let $W>0$ and let $\sigma$ be an activation such that for all $\lambda\in(0,W)$, $\sigma(\lambda z)$ is a $(B,L)$-Regular activation. Fix a unit vector $\vec w$  and assume that for some $0<\lambda\leq W$, it holds that $\Exy[(\sigma(\lambda\vec w\cdot\x)-y)^2]\leq \eps$. Then, with $N=\poly(1/\eps,W,B,L)$ samples and $\poly(d,N)$ runtime, we can find $\lambda'>0$ so that $\Exy[(\sigma(\lambda'\vec w\cdot\x)-y)^2]\leq 4\eps$.
\end{lemma}
\begin{proof}
Let $r=\poly(\eps,1/L,1/B)$ and we fix the following grid $(1+r)^k$, for $k=1,\ldots, O(\log(W)/r)$. From \Cref{app:lem:without-loss-of-generality}, for some $k\leq O(\log(W)/r)$ it holds that $\Exy[(\sigma(\lambda\vec w\cdot\x)-\sigma((1+r)^k\vec w\cdot\x))^2]\leq 2\eps$. Hence, by testing all the possible choices and outputting the one with the minimum error suffices. This testing can be done with $\poly(1/\eps,W,B,L)$ samples.
\end{proof}
\begin{lemma}\label{app:lem:without-loss-of-generality}
      Let $\sigma$ be a $(B,L)$-Regular activation. Let $\eps>0$ sufficiently small, then for $r\leq \poly(\eps,1/L,1/B)$, it holds that
      $\E_{z\sim \normal(0,1)}[(\sigma(z)-\sigma((1+r)z))^2]\leq \eps$.
\end{lemma}
\begin{proof}
    From \Cref{clm:difference}, it holds that $ \E_{z\sim \normal(0,1)}[(\sigma(z)-\Tre_{\delta}\sigma(z))^2]\leq 3(1-\delta)L^2$.
    Using Markov's inequality we have that 
    \[
    \pr[|\sigma(z)-\Tre_{\delta}\sigma(z)|\geq \eps]\leq \frac{\E_{z\sim \normal(0,1)}[(\sigma(z)-\Tre_{\delta}\sigma(z))^2]}{\eps^2}\leq\frac{3(1-\delta)L^2}{\eps^2}\;.
    \]
    Furthermore, note that 
    \[
     \pr_{z\sim\normal(0,1)}[|\sigma((1+r)z)-\Tre_{\delta}\sigma((1+r)z)|\geq \eps]=  \pr_{z\sim\normal(0,(1+r)^2)}[|\sigma(z)-\Tre_{\delta}\sigma(z)|\geq \eps]\;.
    \]
    Furthermore, note that the total variation distance between two zero mean Guassians with variance $\sigma_1$ and $\sigma_2$ is bound from above by $\sqrt{\log(\sigma_1/\sigma_2)-\sigma_2^2/(2\sigma_1^2)-1/2}$, for our case this is smaller than $2 r$. Therefore, we have that
    \[
    \pr_{z\sim\normal(0,1)}[|\sigma((1+r)z)-\Tre_{\delta}\sigma((1+r)z)|\geq \eps]\leq \pr_{z\sim\normal(0,1)}[|\sigma(z)-\Tre_{\delta}\sigma(z)|\geq \eps]+2r\;.
    \]
    Therefore,
    \[
    \pr_{z\sim\normal(0,1)}[|\sigma((1+r)z)-\Tre_{\delta}\sigma((1+r)z)|\geq \eps]\leq \frac{3(1-\delta)L^2}{\eps^2}+2r\;.
    \]
    Combining, we have that
    \begin{align*}
        \E_{z\sim \normal(0,1)}[(\sigma((1+r)z)-\Tre_{\delta}\sigma((1+r)z))^2]&\leq\eps^2 \pr_{z\sim\normal(0,1)}[|\sigma((1+r)z)-\Tre_{\delta}\sigma((1+r)z)|\leq \eps] \\&+ B^2 \pr_{z\sim\normal(0,1)}[|\sigma((1+r)z)-\Tre_{\delta}\sigma((1+r)z)|\geq \eps]
        \\&\leq \eps^2+B^2(2r+\frac{3(1-\delta)L^2}{\eps^2})\;.
    \end{align*}
    Furthermore, from \Cref{fct:semi-group} it holds that $\Tre_{\delta}\sigma(z) $ is $B/(1-\delta^2)^{1/2}$-Lipschitz. Therefore, we have that
    \begin{align}\label{app:eq:lem:without-loss-of-generality}
     \E_{z\sim \normal(0,1)}[(\Tre_{\delta}\sigma(z)-\Tre_{\delta}\sigma((1+r)z))^2]\leq \frac{B^2 r^2 }{(1-\delta^2)}\;.    
    \end{align}
    Choosing $\delta$ so that $1-\delta,1-\delta^2\leq O(\eps^4/(B^2+L^2))$ and $r=\eps (1-\delta^2)/(B+1)$, we get that
    \[
    \E_{z\sim \normal(0,1)}[(\sigma(z)-\sigma((1+r)z))^2]\leq \eps\;.
    \]
\end{proof}

\subsection{Truncating the Regular Activations}

Next, we observe that for any $\sigma\in\mathcal{H}_\eps(B,L)$, one can assume without loss of generality that the labels $y$ are bounded and the support of $\sigma'$ is also bounded.
First, we show that we can truncate the labels $y$ without loss of generality.

\begin{claim}\label{app:claim:bound-y}
    Let $\sigma$ be a $(B,L)$-Regular activation. Let $\bar{y} = \sgn(y) \min\{|y|, B\}$. Then, $\Exy[(\bar{y} - \sigma(\w^*\cdot\x))^2]\leq \opt$. Furthermore, for any $\wh{\w}$ such that $\Exy[(\bar{y} - \sigma(\wh{\w}\cdot\x))^2]\leq O(\opt)$, we have $\Exy[(y - \sigma(\wh{\w}\cdot\x))^2]\leq O(\opt)$. Hence it is w.l.o.g.\ to assume that $|y|\leq B$. \end{claim}
\begin{proof}
    Let $\Pi(u) = \sgn(u)\min\{|u|,B\}$ be the projection operator projecting $u\in\R$ to the interval $[-B,B]$ and let
$\bar{y}\eqdef \Pi(y)$. Since $|\sigma(z)|\leq B$ almost surely, we have $\Pi(\sigma(z)) = \sigma(z)$. Thus by the property of projection operators, we have $|y - \sigma(\w^*\cdot\x)|\geq |\Pi(y) - \Pi(\sigma(\w^*\cdot\x))|$. Therefore, we have $\Exy[(\bar{y} - {\sigma}(\w^*\cdot\x))^2] = \Exy[(\Pi(y) - \Pi(\sigma(\w^*\cdot\x)))^2]\leq \Exy[(y - \sigma(\w^*\cdot\x))^2]\leq \opt$.
{The arguments above shows that $\w^*$ is also an $\opt$ solution when $y$ is truncated.}

{
Now let $\wh{\w}$ be a constant approximate solution with respect to the truncated labels: $\Exy[(\bar{y} - \sigma(\wh{\w}\cdot\x))^2]\leq C\opt$. Then, we have
\begin{align*}
     \Exy[(y - \sigma(\wh{\w}\cdot\x))^2]&\leq 2\Exy[(y - \bar{y})^2] + 2\Exy[(\bar{y} - \sigma(\wh{\w}\cdot\x))^2]\\
    &\leq 4\Exy[(y - \sigma(\w^*\cdot\x))^2] + 4\Exy[(\bar{y} - \sigma(\w^*\cdot\x))^2] + 2C\opt\\
    &\leq  (8+2C)\opt.
\end{align*}
Therefore, $\wh{\w}$ is also an absolute constant approximate solution with respect to the true labels $y$, thus, it is without loss of generality to consider the $L_2^2$ loss with the truncated labels $\bar{y}$. 
}
\end{proof}

Finally, we show that for $\sigma\in\mathcal{H}(B,L)$, $\sigma$ can be truncated so that the support of $\sigma'$ can be bounded by $M<\infty$.

\begin{claim}\label{app:claim:bounded-support-sigma'}
    Let $\sigma$ be a $(B,L)$-Regular 
 activation. {Then, there exists a function $\tilde{\sigma}\in\mathcal{H}(B,L)$ that satisfies $\|\tilde{\sigma} - \sigma\|_{L_2}^2\leq \eps$ and such that the support of $\tilde{\sigma}'$ is $M$ and is bounded from above by 
 $$M \leq \sqrt{2\log(4B^2/\eps) - \log\log(4B^2/\eps)}.$$ Moreover, if $\wh{\w}$ satisfies $\Exy[(y- \tilde{\sigma}(\wh{\w}\cdot\x))^2]\leq O(\opt)+\eps$, then also $\Exy[(y- {\sigma}(\wh{\w}\cdot\x))^2]\leq O(\opt) + \eps$. Thus,}
    one can {replace $\sigma$ with $\tilde{\sigma}$ and } assume without loss of generality that the support of $\sigma'$ is bounded by $M$. \end{claim}
\begin{proof}
    Note that by choosing $M = \sqrt{2\log(4B^2/\eps) - \log\log(4B^2/\eps)}$, using \Cref{fct:komatsu} we have \begin{align*}
    \pr[|z|\geq M]\leq\frac{2\exp(-M^2/2)}{M} = \frac{(\eps/(4B^2))\sqrt{\log(4B^2/\eps)}}{\sqrt{2\log(4B^2/\eps) - \log\log(4B^2/\eps)}}\leq \frac{\eps}{4B^2}.
\end{align*}
Let us define
\begin{equation*}
    \tilde{\sigma}(z) = \begin{cases}
        {\sigma}(z),\;\text{when }|z|\leq M\\
        {\sigma}(M),\;\text{when }z\geq M\\
        {\sigma}(-M),\;\text{when }z\leq -M.
    \end{cases}
\end{equation*}
Then, since $\|{\sigma}\|_\infty\leq B$, we have $\|\tilde{\sigma}\|_\infty\leq B$, and it holds
\begin{align*}
    \Ez[({\sigma}(z) - \tilde{\sigma}(z))^2] &= \Ez[({\sigma}(z) - \tilde{\sigma}(z))^2\1\{|z|\geq M\}]\leq 4B^2\pr[|z|\geq M]\leq \eps.
\end{align*}
In addition, $\|\tilde{\sigma}'(z)\|_{L_2} =\|\sigma'(z)\1\{|z|\leq M\}\|_{L_2}\leq L$. In other words, there exists an activation $\tilde{\sigma}\in\mathcal{H}(B,L)$ such that $\|\tilde{\sigma} - \sigma\|_{L_2}^2\leq \eps$.
Furthermore, we have
\begin{align*}
    \Exy[(y - \tilde{\sigma}(\w^*\cdot\x))^2]&\leq 2\Exy[(y - \sigma(\w^*\cdot\x))^2] + 2\Ex[(\sigma(\w^*\cdot\x) - \tilde{\sigma}(\w^*\cdot\x))^2]\\
    &\leq C\opt + \eps,
\end{align*}
Now let $\wh{\w}$ satisfy $\Exy[(y- \tilde{\sigma}(\wh{\w}\cdot\x))^2]\leq C\opt + \eps$. We show that $\calL(\wh{\w}) \leq O(\opt) + \eps$. We only need to observe that
\begin{align*}
    \Exy[(y - \sigma(\wh{\w}\cdot\x))^2]&\leq 2\Exy[(y-\tilde{\sigma}(\wh{\w}\cdot\x))^2] + 2\Ex[(\tilde{\sigma}(\wh{\w}\cdot\x) - \sigma(\wh{\w}\cdot\x))^2]\\
    &\leq 2C\opt + 4\eps.
\end{align*}

Hence we can replace $\sigma$ with $\tilde{\sigma}\in\mathcal{H}(B,L)$ and focus on the $L_2^2$ loss with respect to $\tilde{\sigma}$. Therefore, we can assume without loss of generality that $\sigma(z)$ is a constant when $|z|\geq M$, in other words, for any $|z|\geq M$, we have $\sigma'(z) = 0$, and the support of $\sigma'$ is indeed bounded by $M$.
\end{proof}

\begin{remark}
    {\em Since $M$ is an upper bound on the support of $\sigma'$, we will assume without loss of generality throughout the rest of the paper that $M^2$ is larger than any absolute constant $C$.}
\end{remark}

\subsection{Examples of Regular Activations}
We now show that $\mathcal{H}_\eps(B,L)$ contains a wide range of activations.
First, we show that all monotone functions with bounded $2+\zeta$-moment are Extended Regular activations:
\begin{lemma}\label{app:lem:activation-truncation-2+zeta}
    If $\sigma$ satisfies $\Ez[\sigma(z)^{2+\zeta}]\leq B_{\sigma}$ for some $\zeta>0$ and $\sigma$ is monotone, then $\sigma\in\mathcal{H}_\eps(c_1 D, c_2D^4/\eps^2)$ where $D=(B_{\sigma}/4\eps)^{1/\zeta}$ and $c_1,c_2$ are absolute constants.
\end{lemma}
\begin{proof}
    For some $\zeta>0$, we have $\Ez[\sigma(z)^{2+\zeta}]\leq B_{\sigma}$. From Markov's inequality, we have that 
\begin{align*}
    \pr[|\sigma(z)|\geq T] \leq \frac{\Ez[\sigma(z)^{2+\zeta}]}{T^{2+\zeta}}\leq \frac{B_{\sigma}}{T^{2+\zeta}}.
\end{align*}
Note that $\E[\sigma^2(z)]=\int_{0}^\infty \pr[\sigma^2(z)\geq t]\d t=\int_{0}^\infty 2u\pr[\sigma^2(z)\geq u^2]\d u$ (the last part is after change of variables to $u^2=t$). Therefore, we have that

\begin{align*}
    \E[\sigma^2(z)\1\{|\sigma(z)|\geq D\}]&=\int_{0}^\infty 2u\pr[\sigma^2(z)\1\{|\sigma(z)|\geq D\}\geq u^2]\d u
    \\&=\int_{0}^\infty 2u\pr[|\sigma(z)|\1\{|\sigma(z)|\geq D\}\geq u]\d u
    \\&=\int_{D}^\infty 2u\pr[|\sigma(z)|\geq u]\d u\leq \int_{D}^\infty 2u\frac{B_{\sigma}}{u^{2+\zeta}}\leq 4 \frac{B_{\sigma}}{D^{\zeta}}\;.
\end{align*}
Set $D=(B_{\sigma}/4\eps)^{1/\zeta}$ and let $\bar{\sigma}(z) = \sgn(\sigma(z))\min\{|\sigma(z)|,D\}$. We show that $\Ez[(\sigma(z) - \bar{\sigma}(z))^2]\leq\eps$:
\begin{align*}
\Ez[(\sigma(z) - \bar{\sigma}(z))^2]&=\Ez[(\sigma(z) - \bar{\sigma}(z))^2\1\{|\sigma(z)|\geq D\}]  \\
&\leq \Ez[(\sigma(z))^2\1\{|\sigma(z)|\geq D\}]\leq \eps\;.
\end{align*}
Therefore, $\sigma$ is $\eps$-close to a $\bar{\sigma}$ with $\|\bar{\sigma}\|_{\infty}\leq (B_{\sigma}/4\eps)^{1/\zeta}$.

It remains to show that the activation $\bar{\sigma}$ is also $\eps$-close to an activation with bounded $\|\sigma'\|_{L_2}$. Without loss of generality we assume that $ \bar{\sigma}(z) \geq 0$ and $\bar{\sigma}(z) \in[0,2D]$, because we can just add assume that the function if $\bar{\sigma}(z)'=\bar{\sigma}(z) +D $.
Note that it holds that $\bar{\sigma}(z) =\int_{0}^{2D}\1\{\bar{\sigma}(z) \geq t\}\d t$. It suffices to show that there exists a parameter $\rho$ so that $\|\bar{\sigma}(z) -\Tr \bar{\sigma}(z) \|_{L_2}^2\leq (1-\rho^2)\poly(D)$.
We have that
\begin{align*}
    \|\bar{\sigma}(z) -\Tr \bar{\sigma}(z) \|_{L_2}^2& = \Ez\bigg[\bigg(\int_0^{2D} \1\{\bar{\sigma}(z)\geq t\} - \1\{\Tr\bar{\sigma}(z)\geq t\}\diff{t}\bigg)^2\bigg]\\
    &\leq \Ez\bigg[\bigg(\int_0^{2D} |\1\{\bar{\sigma}(z)\geq t\} - \1\{\Tr\bar{\sigma}(z)\geq t\}|\diff{t}\bigg)^2\bigg]\\
    &\leq 2D\Ez\bigg[\int_0^{2D} \bigg(\1\{\bar{\sigma}(z)\geq t\} - \1\{\Tr\bar{\sigma}(z)\geq t\}\bigg)^2\diff{t}\bigg]\\
    &= 2D\int_{0}^{2D}\|\1\{\bar{\sigma}(z)\geq t\} -\Tr \1\{\bar{\sigma}(z)\geq t\}  \|_{L_2}^2\d t\;,
\end{align*}
where we used the Jensen's inequality 
$((1/(b-a)\int_a^b f(z)\diff{z}))^2\leq (1/(b-a))\int_a^b f^2(z)\diff{z}$ 
for positive functions $f$
and we exchange the integrals using the Fubini's theorem. 
Note that because the function $\bar{\sigma}(z)$ is monotone, 
then there exists a function $q(z)$ so that 
$\1\{\bar{\sigma}(z)\geq t\}=\1\{z\geq q(t)\}$. 
Therefore, using this transformation, it suffices to bound the difference 
\begin{align*}
    \|\1\{z\geq q(t)\}-\Tr\1\{z\geq q(t)\}\|_{L_2}^2&\leq \E_{x,z\sim \mathcal N(0,1)}[(\1\{x\geq q(t)\}-\1\{x \rho +z(1-\rho^2)^{1/2}\geq q(t)\})^2]
    \\&\leq 4(1-\rho^2)^{1/2}\;,
\end{align*}
where in the first inequality we used Jensen, 
and in the second one we used that 
$\E[|\sign(\vec w\cdot \x+t)-\sign(\vec v\cdot \x+t)|\leq \theta(\vec v,\vec u)$ 
for any two unit vectors $\vec v,\vec w$ (see Fact C.11 of \cite{DKTZ22b}). 
Hence, we have that
\[
\int_{0}^{2D}\|\1\{\bar{\sigma}(z)\geq t\} -\Tr \1\{\bar{\sigma}(z)\geq t\}  \|_{L_2}^2\d t\leq 8D (1-\rho^2)^{1/2}\;.
\]
Therefore, the function $ \|\bar{\sigma}(z) -\Tr \bar{\sigma}(z) \|_2^2\leq \eps$ for $\rho=\sqrt{1-(\eps/(16D^2))^2}$. 
That means that $\|(\Tr \bar{\sigma}(z))'\|_{L_2}^2\leq 16^2D^4/\eps^2$ (cf. \Cref{fct:semi-group}(c)). 
Thus, we conclude that $\sigma\in\mathcal{H}_\eps(2 D, 16^2D^4/\eps^2)$.
\end{proof}

Now let us define a special type of activation that has an `exponential-tail' property.  We will show in \Cref{app:lem:activation-truncation-lip-and-mon-jump} that all $b$-Lipschitz functions are such kind of activations.
\begin{definition}[$(R,r)$-Sub-exponential Activations]\label{assum:activation-tail}
        We say that an activation $\sigma(z)$ is $(R,r)$-sub-exponential for some positive constants $R, r$, if for any $p>0$, we have $(\Ez[\sigma(z)^p])^{1/p}\leq Rp^r$. \end{definition}

We will make use of the following fact in the proof of \Cref{app:lem:activation-truncation-lip-and-mon-jump}.
\begin{fact}[\citep{vershynin2018Bookhigh} Theorem 5.2.2]\label{fct:lip-function-subG}
    Let $z\sim\calN(0,1)$ and let $\sigma$ be a $b$-Lipschitz function. Then, $\sigma(z)$ is a sub-Gaussian random variable with $\|\sigma(z)\|_{\psi_2}\leq cb$, where $\|\cdot\|_{\psi_2}$ is the Orlicz 2-norm and $c$ is an absolute constant. \end{fact}

We show that all the following function classes belong to the Extended Regular activation class:
\begin{lemma}\label{app:lem:activation-truncation-lip-and-mon-jump}
All of the following activations are $\eps$-Extended $(B,L)$-Regular. 
\begin{enumerate}
    \item If $\sigma$ satisfies $\Ez[\sigma(z)^4]\leq B_{\sigma,4}$ and $\|\sigma'\|_{L_2}\leq L$, then $\sigma\in\mathcal{H}(\sqrt{B_{\sigma,4}/\eps},L)$.
    \item If $\sigma$ is $(R,r)$-Sub-exponential and $\|\sigma'\|_{L_2}\leq L$, then $\sigma\in\mathcal{H}_\eps(cR (r + \log(R/\eps))^r, L)$, where $c$ is an absolute constant.
    \item If $\sigma$ is $b$-Lipschitz, then $\sigma\in\mathcal{H}_\eps(cb\log^{1/2}(b/\eps), b)$, where $c$ is an absolute constant. 
    \item If $\sigma = \sigma_1 + \Phi$, where $\sigma_1\in\mathcal{H}_\eps(B,L)$, $|\Phi(z)|\leq A$, $\Phi\in\calF$ (recall \Cref{def:staircase}), i.e., 
    \begin{equation*}
        \Phi(z)=\sum_{i=1}^m A_i\phi(z;t_i) + A_0: A_0\in\R; A_i > 0, |t_i|\leq M, \forall i\in[m]; m<\infty
    \end{equation*}
then $\sigma\in\mathcal{H}_\eps(B+A, L+\max\{A^2L/\sqrt{\eps}, A^4/\eps\})$.
\end{enumerate}

\end{lemma}
\begin{proof}
We prove each claim in order.

1. Suppose first that $\Ez[\sigma(z)^4]\leq B_{\sigma,4}$. Let $\bar{\sigma}(z) = \sgn(\sigma(z))\min\{|\sigma(z)|,\sqrt{B_{\sigma,4}/\eps}\}$, which is an activation in the $(B,L)$-Regular class. We show that $\Ez[(\sigma(z) - \bar{\sigma}(z))^2]\leq\eps$:
\begin{align*}
\Ez[(\sigma(z) - \bar{\sigma}(z))^2]&=\Ez[(\sigma(z) - \bar{\sigma}(z))^2\1\{|\sigma(z)|\geq \sqrt{B_{\sigma,4}/\eps}\}]  \\
&\leq \Ez[(\sigma(z))^2\1\{|\sigma(z)|\geq \sqrt{B_{\sigma,4}/\eps}\}]\\
&\leq \sqrt{\Ez[\sigma^4(z)]\pr[|\sigma(z)|\geq \sqrt{B_{\sigma,4}/\eps}]}.
\end{align*}
By Markov's inequality we have
\begin{align}\label{eq:P[sigma>Bsigma]}
    \pr[|\sigma(z)|\geq \sqrt{B_{\sigma,4}/\eps}] = \pr[\sigma^2(z)\geq B_{\sigma,4}/\eps]\leq \frac{\Ez[\sigma(z)^4]\eps^2}{B_{\sigma,4}^2}\leq \frac{\eps^2}{B_{\sigma,4}}.
\end{align}
Therefore, the $L_2^2$ difference between $\sigma$ and $\bar{\sigma}$ is bounded above by
\begin{align*}
    \Ez[(\sigma(z) - \bar{\sigma}(z))^2]\leq \sqrt{B_{\sigma,4}\frac{\eps^2}{B_{\sigma,4}}}\leq \eps.
\end{align*}
Therefore, $\sigma$ is an Extended $(\sqrt{B_{\sigma,4}/\eps},L)$-Regular activation.

2. Next, assume that $\sigma$ is $(R,r)$-Sub-exponential. Similarly, let 
\begin{equation*}
    \bar{\sigma}(z) = \sgn(\sigma(z))\min\{|\sigma(z)|,e R  (4r\log(4) + \log(R^4/\eps^2))^r\}
\end{equation*}
Denote for simplicity $B_\sigma\eqdef e R  (4r\log(4) + \log(R^4/\eps^2))^r$ Then, $\bar{\sigma}$ is a $(B_\sigma, L)$-Regular activation.
Using the same arguments as above, we have
\begin{align*}
\Ez[(\sigma(z) - \bar{\sigma}(z))^2]&=\Ez[(\sigma(z) - \bar{\sigma}(z))^2\1\{|\sigma(z)|\geq B_\sigma\}] \\
&\leq \sqrt{\Ez[\sigma^4(z)]\pr[|\sigma(z)|\geq B_\sigma]}.
\end{align*}
Now since $\sigma$ is a $(R,r)$-Sub-exponential activation, we have  \begin{align*}
    \pr[|\sigma(z)|\geq t]\leq \frac{\Ez[\sigma^p(z)]}{t^p}\leq \bigg(\frac{Rp^r}{t}\bigg)^p.
\end{align*}
Choosing $p = (t/(Re))^{1/r}$, we get
\begin{align*}
    \pr[|\sigma(z)|\geq t]\leq \exp\bigg(-(t/(Re))^{1/r}\bigg).
\end{align*}
Let $t = B_\sigma = R e (4r\log(4) + \log(R^4/\eps^2))^r$, then it holds
\begin{equation*}
    \pr[|\sigma(z)|\geq B_\sigma]\leq \exp\bigg(-\log(4^{4r}\eps^2/R^4)\bigg)= \frac{\eps^2}{R^4 4^{4r}}.
\end{equation*}
Furthermore, since $\Ez[\sigma^4]\leq R^4 4^{4r}$, we thus obtain
\begin{equation*}
    \Ez[(\sigma(z) - \bar{\sigma}(z))^2]\leq \eps.
\end{equation*}

3. Next, suppose $\sigma$ is $b$-Lipschitz. 
Then, since $|\sigma'|\leq b$, we have $\|\sigma'\|_{L_2}\leq b$. 
Next, we show that $\sigma$ is $(b,1/2)$-Sub-exponential. 
Note that it is without loss of generality to assume 
that $\Ez[\sigma(z)] = 0$, since we can always consider 
shifting the activation $\sigma$ and 
the labels $y$ to $\sigma(z) - \Ez[\sigma(z)]$, $y - \Ez[\sigma(z)]$ 
and obtaining the same results. Since $z\sim\calN(0,1)$, 
we can use \Cref{fct:lip-function-subG}, which yields that $\sigma(z)$ is sub-Gaussian with sub-Gaussian constant $\|\sigma(z)\|_{\psi_2} = cb$.
Because $\sigma(z)$ is sub-Gaussian, we know $\|\sigma(z)\|_{L_p}\leq c\|\sigma(z)\|_{\psi_2}p^{1/2}\leq cb p^{1/2}$, this implies that $\sigma$ is a $(b,1/2)$-Sub-exponential activation. Using the previous result (Part 2), we immediately obtain that $\sigma$ is an Extended $(cb\log^{1/2}(b/\eps), b)$-Regular activation.

4. Finally, consider $\sigma = \sigma_1 + \Phi$ where $\sigma_1\in\mathcal{H}_\eps(B,L)$, $\Phi\in\calF$, $|\Phi(z)|\leq A$.
Let $\tilde{\sigma} = \Tre_{1 - \eps_0}\sigma_1 +  \Tre_{1 - \eps_0}\Phi$, where $\eps_0\leq \min(\eps/L^2, 1/M^2, (\eps/A^2)^2)$, and $M$ is defined as in \Cref{app:claim:bounded-support-sigma'}. Then we have
\begin{align*}
    \|\tilde{\sigma} - \sigma\|_{L_2}^2\leq 2\|\Tre_{1 - \eps_0}\sigma_1 - \sigma_1\|_{L_2}^2 + 2\|\Tre_{1 - \eps_0}\Phi - \Phi\|_{L_2}^2.
\end{align*}
By \Cref{clm:difference}, we have $\|\Tre_{1 - \eps_0}\sigma_1 - \sigma_1\|_{L_2}^2\lesssim \eps_0\|\sigma_1'\|_{L_2}^2\leq \eps$. In addition, applying \Cref{app:prop:smoothing-error-bound}, since $\eps_0\leq 1/M^2$, we have $\|\Tre_{1-\eps_0}\Phi - \Phi\|_{L_2}^2\lesssim \eps_0\|\Tre_{1-\eps_0}\Phi'(z)\|_{L_2}^2$. Note that $\max_{z\in \R}\Phi(z) = \sum_{i=1}^m A_i\leq A$. Then, 
by \Cref{app:lem:expression-E[(Trho-Phi'(z))**2]}, we have 
\begin{equation*}
    \|\Tre_{1-\eps_0}\Phi'(z)\|_{L_2}^2\lesssim (1/\sqrt{\eps_0})\sum_{i,j=1}^m A_iA_j = (1/\sqrt{\eps_0})A^2 \leq \max\{A^2L/\sqrt{\eps}, A^4/\eps\}. 
\end{equation*}
Therefore, we further obtain $\|\Tre_{\eps}\Phi - \Phi\|_{L_2}^2\leq \sqrt{\eps_0}A^2\leq \eps$. This implies that $\|\tilde{\sigma} - \sigma\|_{L_2}^2\lesssim\eps$. Furthermore, observe that $\|\tilde{\sigma}\|_{L_\infty}\leq \|\sigma_1\|_{L_\infty} + \|\Phi\|_{L_\infty}\leq B + A$, and $\|\tilde{\sigma}'\|_{L_2}\leq \|\Tre_{1 - \eps_0}\sigma'\|_{L_2} + \|\Tre_{1 - \eps_0}\Phi'\|_{L_2}\leq L + \sqrt{\eps}$. Thus, we conclude that $\tilde{\sigma}\in\mathcal{H}_\eps(B+A, L+\max\{A^2L/\sqrt{\eps}, A^4/\eps\})$.
\end{proof}

\subsection{Required Assumptions on Activation}\label{app:sec:assum-required}

Here we point out that it is information-theoretically impossible 
to learn all monotone activations, in the realizable setting, 
if we only assume that $\sigma\in L_2(\normal)$---even if 
we further assume that $\|\sigma\|_{L_2}\leq 1$.
The intuition behind this fact is the following: 
there exists a {monotone} function $\sigma$, 
which is equal to $0$ everywhere except in the tails of a direction $\vec v$. 
That means that in order to see a point where the labels are non-zero, 
we need to see a label from the tails. 

We show that for any choice of the threshold in the tails, 
i.e., $\pr_{z\sim \mathcal N}[z\geq t]$, 
there exists a {monotone} function that has $\|\sigma\|_{L_2}=1$ 
and for any unit vectors $\vec v,\vec u$ with 
$\|\vec v-\vec u\|_2= \Omega(1)$, 
we have $\|\sigma(\vec v\cdot\x)- \sigma(\vec u\cdot\x)\|_{L_2}^2= \Omega(1)$. Formally, we show that: 
\begin{theorem}[Impossibility of Learning All Monotone Functions] \label{thm:lb}
Consider the class $\mathcal F$ consisting of all monotone activations 
$\sigma \in L_2(\normal)$ satisfying $\|\sigma\|_{L_2}\leq 1$. 
There is no finite-sample algorithm that realizably 
learns $\mathcal F$ 
up to error $1/8$.
\end{theorem}
\begin{proof}
Let 
$\gamma^{-1}(\delta)=\sup_t \{t: \pr_{z\sim \mathcal{N}(0,1)}[z\geq t]\leq \delta\}$. 
Let $\delta<1/16$ and consider the following function 
$\sigma(t)=(1/\sqrt{\delta})\1\{t\geq \gamma^{-1}(\delta)\}$. 
Note that this function belongs to the class $\sigma \in L_2(\normal)$ with $\|\sigma\|_{L_2}\leq 1$, since 
    \[
    \E[\sigma^2(z)]=1/\delta \pr[t\geq \gamma^{-1}(\delta)]=1\;.
    \]
Consider a set of unit vectors $V$ such that for any $\vec u,\vec v\in V$ 
we have $\|\vec u-\vec v\|_2\geq 1/2$. By standard packing arguments, there exists 
such a set of size $2^{\Theta(d)}$.  
Let $\theta\eqdef\theta(\bu,\bv)$. 
Then we have $\|\vec u-\vec v\|_2 = 2\sin(\theta/2)\geq 1/2$, 
hence $\cos\theta= 1 - 2\sin^2(\theta/2)\leq 7/8$.  
Note that for any $\vec u,\vec v\in V$, with $\vec u\neq \vec v$, it holds that
    \begin{align*}
        \|\sigma(\vec v\cdot\x)- \sigma(\vec u\cdot\x)\|_{L_2}^2&=2(1-\E[\sigma(\vec v\cdot\x)\sigma(\vec u\cdot\x)])
        \\&=2\sum_{k\geq 0}(1 - \cos^k\theta)\hat{\sigma}(i)^2
        \\&\geq 2(1-\hat{\sigma}(0)^2)- \cos\theta\sum_{k\geq 1}\hat{\sigma}(i)^2
        \geq 2(1-\hat{\sigma}(0)^2) -2\cos\theta \|\sigma\|_{L_2}^2\;,
    \end{align*}
    where $\hat{\sigma}(i)$ are the Hermite coefficients of $\sigma$.
    Furthermore, note that $\hat{\sigma}(0)=\E[\sigma(z)]=\sqrt{\delta}$. 
    Hence, we have that
    \[
    \|\sigma(\vec v\cdot\x)- \sigma(\vec u\cdot\x)\|_{L_2}^2\geq 2(1-\delta - 7/8)\geq 1/4-2\delta\geq 1/8\;.
    \] 
    Intuitively, in order to learn up to error $\eps<1/2$, 
    we need to see at least one sample $\x$ such that 
    $\sigma(\vec v\cdot\x)>0$, which happens with probability $\delta$. 
    Since $\delta$ can be selected to be an arbitrarily small positive number, 
    by taking $\delta\to 0$, we see that in order to observe 
    one sample, we need $\Omega(1/\delta)$ samples;  
    if we choose $\vec v$ at random, we succeed with probability at most $\exp(-c d)$. 

One way to formalize the argument is to reduce the problem 
of learning Gaussian halfspaces to the above task. 
Consider the following transformation: 
$\forall \bv\in V$, let $y_{\vec v}'=1$ when $\sigma(\bv\cdot\x)=1/\sqrt{\delta}$ 
and $y_{\vec v}'=-1$ otherwise. This gives an instance of learning 
halfspaces under the Gaussian distribution.
We have that
\[
\|\sigma(\vec v\cdot\x)- \sigma(\vec u\cdot\x)\|_{L_2}^2=\pr[\sign(\vec v\cdot\x -\gamma^{-1}(\delta))\neq \sign(\vec u\cdot\x -\gamma^{-1}(\delta))]/\delta\;.
\]
Therefore, in order to get error of $1/8$ for our monotone GLM learning task, 
we need to learn halfspaces with accuracy better than $\delta/8$. 
This task is known to have a sample complexity lower bound of $\Omega(d/\delta)$. Therefore, since an algorithm that learns $\sigma(\vec v\cdot\x)$ 
with error better than $1/8$ will also learn halfspaces, 
it follows that achieving error better than $1/8$ requires $\gtrsim d/\delta$ samples. As $\delta\to 0$, the number of samples becomes unbounded. 
If we output a function at random, the probability of success is lower bounded 
by the number of elements in $|V|$, which gives the result.
\end{proof}

\section{Full Version of \Cref{sec:augmentation-and-landscape}}\label{app:sec:augmentation-and-landscape}

\subsection{Augmenting the Data: Connection to \OU Semigroup}

As already discussed in \Cref{sec:tech}, our algorithm relies on the data augmentation technique, i.e., in each iteration, the algorithm injects Gaussian noise (see \Cref{app:alg:aug}), which has the effect of improving the regularity properties of the loss landscape, as shown in this section. 

\begin{algorithm}[h]
   \caption{Augment Dataset with Injected White Noise}
   \label{app:alg:aug}
\begin{algorithmic}[1]
\STATE {\bfseries Input:} Parameters $\rho$, $m$; Sample data $\mathfrak{D}=\{(\vec x^{(1)},y^{(1)}),\ldots,(\vec x^{(N)},y^{(N)})\}$; $S\gets \emptyset$. 
\FOR{$(\x\ith,y\ith)\in \mathfrak{D}$}
\FOR{$j=1,\ldots, m$}
\STATE Sample $\vec z$ from $\mathcal{N}(\vec 0,\vec I)$. Let $\tilde{\x}^{(j)}\gets \rho \x\ith+(1-\rho^2)^{1/2} \vec z$ and add to $S\gets S\cup \{(\tilde{\x}^{(j)},y\ith)\}$.
\ENDFOR
\ENDFOR
\STATE {\bfseries Return:} $S$.
\end{algorithmic}
\end{algorithm}

\begin{remark}
{\em    Note that \Cref{app:alg:aug} does not require new samples from the distribution $\D$. As we will see in \Cref{app:lem:ou-aug-gradient}, the purpose of \Cref{app:alg:aug} is to estimate $\Tr\sigma'(\w\cdot\x\ith)$ for each sample $\x\ith$ we received. By standard concentration bounds using at most $\tilde{O}(L/\eps^2)$ independent (unlabeled) Gaussian samples suffices for all $\{\x\ith,y\ith\}\in\mathfrak{D}$. Since this affects only the runtime in polynomially, for simplicity of analysis we assume \Cref{app:alg:aug} can be executed efficiently and have access to the population one.}
\end{remark}

The augmentation can be viewed as a transformation of the distribution $\D$ to $\D_\rho$, where for any $(\tx,y)\sim\D_\rho$, it holds $\tx \sim \rho\D_\x + (1 - \rho^2)^{1/2}\calN(\vec 0,\vec I)$. 
The data augmentation introduced in \Cref{app:alg:aug} in fact simulates the \OU semigroup, which we formalize in the following lemma.
\begin{lemma}\label{app:lem:ou-aug}  Let $\D$ be a distribution of labeled examples
$(\x,y) \in \R^d \times \R$ such with $\x$-marginal 
$\D_\x = \mathcal N(\vec 0,\vec I)$. Furthermore, let $\D_\rho$ be the distribution constructed by applying  \Cref{alg:aug} to samples from $\D$. Then for any $f:\R\to\R$ and any unit vector $\vec w\in \R^d$ with $|\Exx[f(\vec w\cdot\x)]|<\infty$, we have 
$
\E_{\tx\sim (\D_{\rho})_{\tx}}[f(\vec w\cdot\tx)]=\Ex[\Tr f(\vec w\cdot \x)]\;.
$
\end{lemma}
\begin{proof}
Using the definition of $\D_\rho,$ we have that
\begin{align*}
    \E_{\tx\sim (\D_{\rho})_{\tx}}[f(\vec w\cdot\tx)]&=\E_{\x\sim \D_\x}[\E_{\vec z\sim \mathcal{N}(\vec 0,\vec I)}[f(\rho\vec w\cdot\x+\sqrt{1 - \rho^2}\vec w\cdot\vec z)]]\\&=\E_{\x\sim \D_\x}[\E_{\zeta\sim \mathcal{N}(0,1)}[f(\rho\vec w\cdot\x+\sqrt{1 -\rho^2}\zeta)]]=\Ex[\Tr f(\vec w\cdot \x)]\;,
\end{align*}
where we have used that $\vec w\cdot\vec z$ is distributed according to the standard normal distribution.
\end{proof}
\Cref{app:lem:ou-aug} shows that our data augmentation technique is equivalent to applying the \OU semigroup to our dataset. This application of the \OU semigroup to the dataset has the effect of smoothing the landscape of the square loss, which in turn allows us to prove that the gradient of the smoothed/augmented loss carries information about the direction of the target vector. This is the main structural result obtained in the next subsection. 

\subsection{Alignment of the Gradients of the Augmented Loss}

In this section, we provide the main structural result of this work, showing that the gradients of the square loss applied to the augmented data correlate with a target parameter vector $\wstar.$ For notational convenience, we use \begin{equation}
    \mathcal L_\rho(\vec w)=\E_{(
\tx,y)\sim \D_{\rho}}[(\sigma( \w\cdot\tx)-y)^2] \label{app:eq:loss-function}
\end{equation}
to denote the square loss on the augmented data and refer to it as the ``augmented loss.''

\begin{proposition}[Main Structural Result]\label{app:prop:structural}
Fix an activation $\sigma:\R\to\R$. Let $\D$ be a distribution of labeled examples
$(\x,y) \in \R^d \times \R$ such that its $\x$-marginal 
$\D_\x$ is $\mathcal N(\vec 0,\vec I)$. Moreover, let $\D_{\rho}$ be the distribution constructed by applying \Cref{alg:aug} with parameter $\rho\in(0,1)$ to the distribution $\D$. Fix unit vectors $\wstar,\w\in \R^d$ such that $\Exy[(\sigma(\wstar\cdot\x)-y)^2] = \opt$ and let $\theta=\theta(\wstar,\w)$. Let $\g(\vec w)$ be the gradient of the loss  $\mathcal L_\rho(\vec w)=\E_{(
\tx,y)\sim \D_{\rho}}[(\sigma( \w\cdot\tx)-y)^2]$ projected on the subspace $\vec w^{\perp}$ and {scaled by $1/(2\rho)$}, i.e., $\g(\vec w)=(1/(2\rho))(\nabla_{\vec w} \mathcal L_\rho(\vec w))^{\perp \vec w}$.
 Then, \begin{equation*}
    \g(\w)\cdot\w^*\leq  -\|\Tre_{\sqrt{\rho\, \cos\theta}}\sigma'\|_{L_2}^2\sin^2\theta + \sqrt{\opt}\|\Tre_{\rho}\sigma'\|_{L_2}\sin\theta.
\end{equation*}
In particular, if $0<\rho\leq \cos\theta<1$ and  $\sin\theta\geq 3\sqrt{\opt}/\|\Tre_{\rho}\sigma'\|_{L_2}$, then \begin{equation*}
    \g(\w)\cdot\w^*\leq -(2/3) \|\Tre_{\sqrt{\rho\, \cos\theta}}\sigma'\|_{L_2}^2 \sin^2\theta\;.\end{equation*}
\end{proposition}

Before proving the proposition, we first prove the following auxiliary lemma, which establishes a connection between the Riemannian gradient of \Cref{app:eq:loss-function} and the \OU semigroup applied to the derivative of the activation. 
\begin{lemma}\label{app:lem:ou-aug-gradient}
    Let $\g(\vec w)=(1/(2\rho))(\nabla_{\vec w} \mathcal L_\rho (\vec w))^{\perp \vec w}$. Then, $\g(\w) = - \Exy[y\Tr\sigma'(\w\cdot\x)\x^{\perp\w}]$. 
\end{lemma}
\begin{proof}
    By definition, the projected gradient vector $\g(\w)$ equals
    \begin{align*}
        \g(\w) = \frac{1}{\rho}\bigg(\E_{(\tx,y)\sim \D_{\rho}}[\sigma(\w\cdot\tx)\sigma'(\w\cdot\tx)\tx^{\perp\w}] - \E_{(\tx,y)\sim \D_{\rho}}[y\sigma'(\w\cdot\tx)\tx^{\perp\w}]\bigg).
    \end{align*}
    Note that since $(\D_\rho)_{\tx}$ is also a standard Gaussian distribution, we have $\E_{\tx\sim(\D_\rho)_{\tx}}[\tx^{\perp\w}] = \vec 0$, hence $\E_{(\tx,y)\sim \D_{\rho}}[\sigma(\w\cdot\tx)\sigma'(\w\cdot\tx)\tx^{\perp\w}] = \vec 0$, as $\w\cdot\tx$ and $\tx^{\perp\w}$ are independent. Thus, $\rho\g(\w) = - \E_{(\tx,y)\sim \D_{\rho}}[y\sigma'(\w\cdot\tx)\tx^{\perp\w}]$. Now, since $(\D_\rho)_{\tx} = \rho\D_\x + (1 - \rho^2)^{1/2}\calN(\vec 0, \vec I)$, we have
    \begin{align*}
        \g(\w) &= - \frac{1}{\rho}\E_{(\x,y)\sim \D, \bz\sim\calN(\vec 0, \vec I)}[y\sigma'(\w\cdot(\rho\x + \sqrt{1 - \rho^2}\bz))(\rho\x + \sqrt{1 - \rho^2}\bz)^{\perp\w}].
    \end{align*}
    As $\bz$ is independent of $\x,y$ and  follows  the standard Gaussian distribution, it must be $\E_{(\x,y)\sim \D, \bz\sim\calN(\vec 0, \vec I)}[y\sigma'(\w\cdot\tx)\bz^{\perp\w}] = 0$, and thus we further have
    \begin{align*}
        \g(\w) &= - \frac{1}{\rho}\E_{(\x,y)\sim \D, \bz\sim\calN(\vec 0, \vec I)}[y\sigma'(\w\cdot(\rho\x + \sqrt{1 - \rho^2}\bz))\rho\x^{\perp\w}]\\
        &=-\Exy[y\E_{\bz\sim\calN(\vec 0, \vec I)}[\sigma'(\w\cdot(\rho\x + \sqrt{1 - \rho^2}\bz))]\x^{\perp\w}]\\
        &=-\Exy[y\Tr\sigma'(\w\cdot\x)\x^{\perp\w}],
    \end{align*}
    completing the proof.
\end{proof}

We are now ready to prove our main structural result. 

\smallskip

\begin{proof}[Proof of \Cref{app:prop:structural}]
When $\w^*$ is parallel to $\w$, then $\g(\w)\cdot\w^* = 0$ since $\g(\w)$ is orthogonal to $\w$, and $\sin\theta = 0$, hence the statements hold trivially. Thus in the rest of the proof we assume that $(\w^*)^{\perp_\w}\neq \vec 0$.
    Denote $\bv \eqdef (\w^*)^{\perp_\w}/\|(\w^*)^{\perp_\w}\|_{2}$. Then, $\w^* = \w\cos\theta + \bv\sin\theta$, where $\theta \eqdef \theta(\w,\w^*)$. 
Using \Cref{lem:ou-aug-gradient}, the inner product between $\w^*$ and $-\g(\w)$ equals
    \begin{align*}
       -\g(\w)\cdot\w^* =\Exy[y\Tr\sigma'(\vec w\cdot\x)\bv\cdot\x]\sin\theta\nonumber\;.
    \end{align*}
By adding and subtracting $\sigma(\wstar\cdot\x)$ on the right-hand side, we get that 
    \begin{align}\label{app:eq:g(w)-cdot_wstr}
-\g(\w)\cdot\w^*
        &=\Ex[\sigma(\w^*\cdot\x)\Tr\sigma'(\w\cdot\x)\bv\cdot\x]\sin\theta \nonumber\\
        &\quad +\Exy[(y - \sigma(\w^*\cdot\x))\Tr\sigma'(\w\cdot\x)\bv\cdot\x]\sin\theta  .
    \end{align}
    Observe that since $\w^* = \cos\theta\w + \sin\theta\bv$ and $\x$ is a standard Gaussian random vector, we have $\w\cdot\x$ and $\bv\cdot\x$ are independent standard Gaussian random variables. 
    By applying Cauchy-Schwarz inequality to the expectation in the last term, we obtain
    \begin{align*}
        &\quad \Exy[(y - \sigma(\w^*\cdot\x))\Tr\sigma'(\w\cdot\x)\bv\cdot\x]  \\
        &\geq -\bigg(\Exy[(y - \sigma(\w^*\cdot\x))^2]\Ex[(\Tr\sigma'(\w\cdot\x)\bv\cdot\x)^2]\bigg)^{1/2}\\
        &= -\sqrt{\opt\Ex[(\Tr\sigma'(\w\cdot\x))^2]}=-\sqrt{\opt}\|\Tr\sigma'\|_{L_2},
    \end{align*}
    where in the first equality we used the definition of $\opt$ and that $\w\cdot\x$ and $\bv\cdot\x$ are independent standard Gaussian random variables noted above.

    To bound the first term on the right-hand side of \eqref{app:eq:g(w)-cdot_wstr}, we again use that $\w\cdot\x$ and $\bv\cdot\x$ are independent standard Gaussian random variables and apply Stein's lemma (\Cref{fct:stein}) to obtain
    \begin{align*}
        \Ex[\sigma(\w^*\cdot\x)\Tr\sigma'(\w\cdot\x)\bv\cdot\x] &= \Ex[\sigma(\cos\theta\w\cdot\x + \sin\theta\bv\cdot\x)\Tr\sigma'(\w\cdot\x)\bv\cdot\x]\\
        & =\Ex[\sigma'(\cos\theta\w\cdot\x + \sin\theta\bv\cdot\x)\sin\theta\Tr\sigma'(\w\cdot\x)]\\
        &=\Ex[\Tre_{\cos\theta}\sigma'(\w\cdot\x)\Tr\sigma'(\w\cdot\x)]\sin\theta=\|\Tre_{\sqrt{\rho\, \cos\theta}}\sigma'\|_{L_2}^2\sin\theta\;,
    \end{align*}
    where in the last equality we used the identity $\Tre_{a}\Tre_{b}=\Tre_{ab} = \Tre_{\sqrt{ab}}\Tre_{\sqrt{ab}}$.
    Therefore, we have that 
    \begin{equation*}
    -\g(\w)\cdot\w^*\geq  \|\Tre_{\sqrt{\rho\, \cos\theta}}\sigma'\|_{L_2}^2\sin^2\theta -\sqrt{\opt}\|\Tre_{\rho}\sigma'\|_{L_2}\sin\theta.
\end{equation*}

To argue the last part of \Cref{app:prop:structural}, recall (by \Cref{fct:semi-group}, Part 2(e)) that the function $g(\lambda)\eqdef\|\Tre_{\lambda}f\|_{L_2}$ is non-decreasing in $\lambda\in(0,1)$ for any function $f\in L_2(\normal)$, therefore
$ \|\Tre_{\sqrt{\rho\, \cos\theta}}\sigma'\|_{L_2}\geq \|\Tre_{\rho}\sigma'\|_{L_2}$ if $\cos\theta\geq \rho$. By using the assumption that $\sin\theta\geq 3\sqrt{\opt}/\|\Tre_{\rho}\sigma'\|_{L_2}$, we obtain that 
    \begin{equation*}
    -\g(\w)\cdot\w^*\geq  (2/3)\|\Tre_{\sqrt{\rho\, \cos\theta}}\sigma'\|_{L_2}^2\sin^2\theta\;.
\end{equation*}
This completes the proof of \Cref{app:prop:structural}.
\end{proof}

\subsection{Critical Points and Their Connection to the $L_2^2$ Loss}\label{app:subsec:critical-point&connection-to-loss}

\Cref{app:prop:structural} provides sufficient conditions ensuring that the vector $-\vec g(\w)$ directs $\w$ towards the direction of $\wstar$ whenever we are in a region around approximate solutions. Specifically, if the parameter $\rho$ is chosen appropriately and the following alignment condition holds:
$\sin\theta \|\Tre_{\cos\theta} \sigma'\|_{L_2} \geq 3\sqrt{\opt}$,
then  $-\vec g(\w)$ has a nontrivial correlation with $\wstar$. Otherwise, we can guarantee that the angle between $\w$ and $\wstar$ is already sufficiently small. 
This implies that the region of convergence of an algorithm that relies on $-\vec g(\w)$ depends on the quantity:
\begin{equation*}
    \psi_{\sigma}(\theta) \coloneqq \sin\theta \|\Tre_{\cos\theta} \sigma'\|_{L_2}.
\end{equation*}
Motivated by this observation, we define the \emph{ Convergence Region}, which characterizes the region of $\theta$  (and, equivalently, the region of $\w$) for which the algorithm makes progress towards $\wstar$.
\begin{definition}[Critical Point and  Convergence Region of $\sigma$]\label{app:def:star-point-regions}
     Given $\sigma : \R \to \R$, $\sigma\in L_2(\mathcal N)$,  and $\theta_0\in[0,\pi/2]$, we define the error alignment function  $\psi_{\sigma}:[0,\pi/2]\to\R_+$ with respect to $\sigma$ as follows: $\psi_{\sigma}(\theta)\eqdef \sin\theta \|\Tre_{\cos\theta}\sigma'\|_{L_2}$. For any $\eps>0$, we define the  Convergence Region $\mathcal R_{\sigma,\theta_0}(\eps)=\{\theta: \psi_{\sigma}(\theta)\leq {\sqrt\eps}\}\cap\{\theta: 0\leq \theta\leq\theta_0\}$. We say that $\theta^*$ is  $(\sigma,\theta_0,\eps)$-Critical Point if  $\theta^*$ is the maximum $\theta$ in $\mathcal R_{\sigma,\theta_0}(\eps)$.
\end{definition}
\Cref{app:def:star-point-regions} defines the  Convergence Region using an upper bound $\theta_0$. This upper bound is necessary because $\psi_{\sigma}(\theta)$ is not necessarily monotonic. Specifically, it can be shown that $\psi_{\sigma}(\theta)$ is non-decreasing up to a critical point \(\theta'\) and then non-increasing (see \Cref{app:fig:psi_sigma} for illustrative examples). 
Consequently, the region $\mathcal R_{\sigma,\theta_0}(\eps)$ may consist of two disjoint intervals. The role of (an appropriately selected) $\theta_0$ is to ensure that this does not happen. \begin{claim}\label{app:claim:initialization-threshold}
    Let $\sigma\in L_2(\calN)$. Then there exists a real number $\bar{\theta}\in (0,\pi/2)$, such that for any $\theta\leq \bar{\theta}$, $\psi_\sigma(\theta)$ is non-decreasing. If $\|\sigma''\|_{L_2}\leq L'$, then $\bar{\theta}\geq \min(\pi/3, \|\Tre_{1/2}\sigma'\|_{L_2}^2/(L')^2)$.
\end{claim}
\begin{proof}
    Since $\psi_\sigma(\theta)\geq 0$, to show that $\psi_\sigma(\theta)$ is non-decreasing is equivalent to show that $\psi_\sigma^2(\theta)$ is non-decreasing. Let us calculate the derivative of $\psi_\sigma^2(\theta)$:
    \begin{align*}
        (\psi_\sigma^2(\theta))' &= \frac{\diff{}}{\diff{\theta}}(\sin^2\theta\Ez[(\Tre_{\cos\theta}\sigma')^2]) \\
        &= 2\sin\theta\cos\theta\|\Tre_{\cos\theta}\sigma'\|_{L_2}^2 + 2\sin^2\theta\Ez\bigg[\Tre_{\cos\theta}\sigma' \frac{\diff{}}{\diff{\theta}}\Tre_{\cos\theta}\sigma'\bigg]
    \end{align*}
    Using \Cref{fact:OU-op}, since $\frac{\diff{}}{\diff{\rho}}\Tr f = (\mathrm{L}\Tr f)/\rho$, we further have
    \begin{align*}
        (\psi_\sigma^2(\theta))' &=2\sin\theta\cos\theta\|\Tre_{\cos\theta}\sigma'\|_{L_2}^2 + 2\sin^2\theta\Ez\bigg[\Tre_{\cos\theta}\sigma' \frac{1}{\cos\theta}\mathrm{L}\Tre_{\cos\theta}\sigma'(-\sin\theta)\bigg]\\
        &=2\sin\theta\cos\theta\|\Tre_{\cos\theta}\sigma'\|_{L_2}^2 -2\frac{\sin^3\theta}{\cos\theta}\Ez\bigg[\frac{\diff}{\diff{z}}\Tre_{\cos\theta}\sigma' \frac{\diff}{\diff{z}}\Tre_{\cos\theta}\sigma'\bigg]\\
        &=2\sin\theta\cos\theta\bigg(\|\Tre_{\cos\theta}\sigma'\|_{L_2}^2 - \tan^2\theta\|(\Tre_{\cos\theta}\sigma')'\|_{L_2}^2\bigg),
    \end{align*}
    where in the second equality we used \Cref{fact:OU-op} that $ \E_{z \sim \normal}\left[f(z) \mathrm{L}\Tre_{\rho} g(z) \right]=\E_{z \sim \normal}\left[ f'(z) (\Tre_{\rho} g(z))' \right]$. Therefore, we only need to prove that $h(\theta)\eqdef\|\Tre_{\cos\theta}\sigma'\|_{L_2}^2 - \tan^2\theta\|(\Tre_{\cos\theta}\sigma')'\|_{L_2}^2\geq 0$ in a region $(0,\bar{\theta})$. Note that $h(0) = \|\sigma'\|_{L_2}^2 > 0$. Furthermore, since $\|\Tre_{\cos\theta}\sigma'\|_{L_2}^2$ and $\|(\Tre_{\cos\theta}\sigma')'\|_{L_2}^2$ are continuous functions of $\theta$ (as we can see by Hermite expansion), we know that there exists a threshold $\bar{\theta}$ such that for any $\theta\leq \bar{\theta}$, it holds $\psi_\sigma'(z)\geq 0$. Furthermore, if $\sigma''$ is in $L_2(\calN)$ and $\|\sigma''\|_{L_2}\leq L'$, then since $(\Tr f(z))' = \rho\Tr f'(z)$ and $\Tr$ is a non-expansive operator (\Cref{fct:semi-group}), we have
    \begin{equation*}
        h(\theta)= \|\Tre_{\cos\theta}\sigma'\|_{L_2}^2 - \sin^2\theta\|\Tre_{\cos\theta}\sigma''\|_{L_2}^2\geq \|\Tre_{\cos\theta}\sigma'\|_{L_2}^2 - \sin^2\theta(L')^2.   
    \end{equation*}
    Assuming $\theta\leq \pi/3$, we have $h(\theta)\geq 0$ as long as $\theta\leq \bar{\theta} = \min(\pi/3, \|\Tre_{1/2}\sigma'\|_{L_2}^2/(L')^2)$.
\end{proof}

The significance of the Critical Point and the  Convergence Region comes from the following proposition, which bounds the $L_2^2$ error for points within the  Convergence Region. 

\begin{figure}[htbp]
  \centering
  \begin{minipage}[b]{0.32\textwidth}
    \centering
    \includegraphics[width=\linewidth]{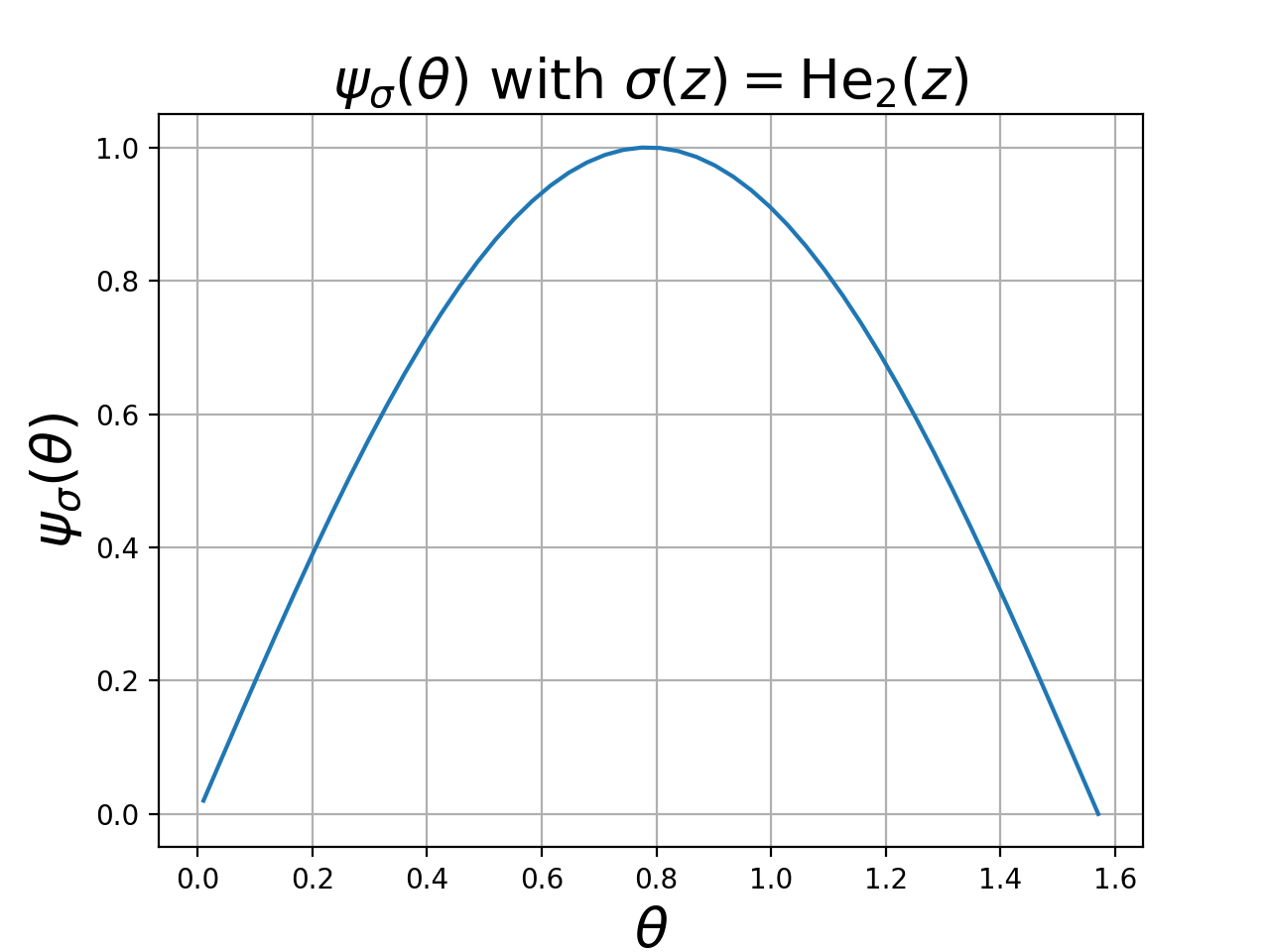}
\end{minipage}\hfill
  \begin{minipage}[b]{0.32\textwidth}
    \centering
    \includegraphics[width=\linewidth]{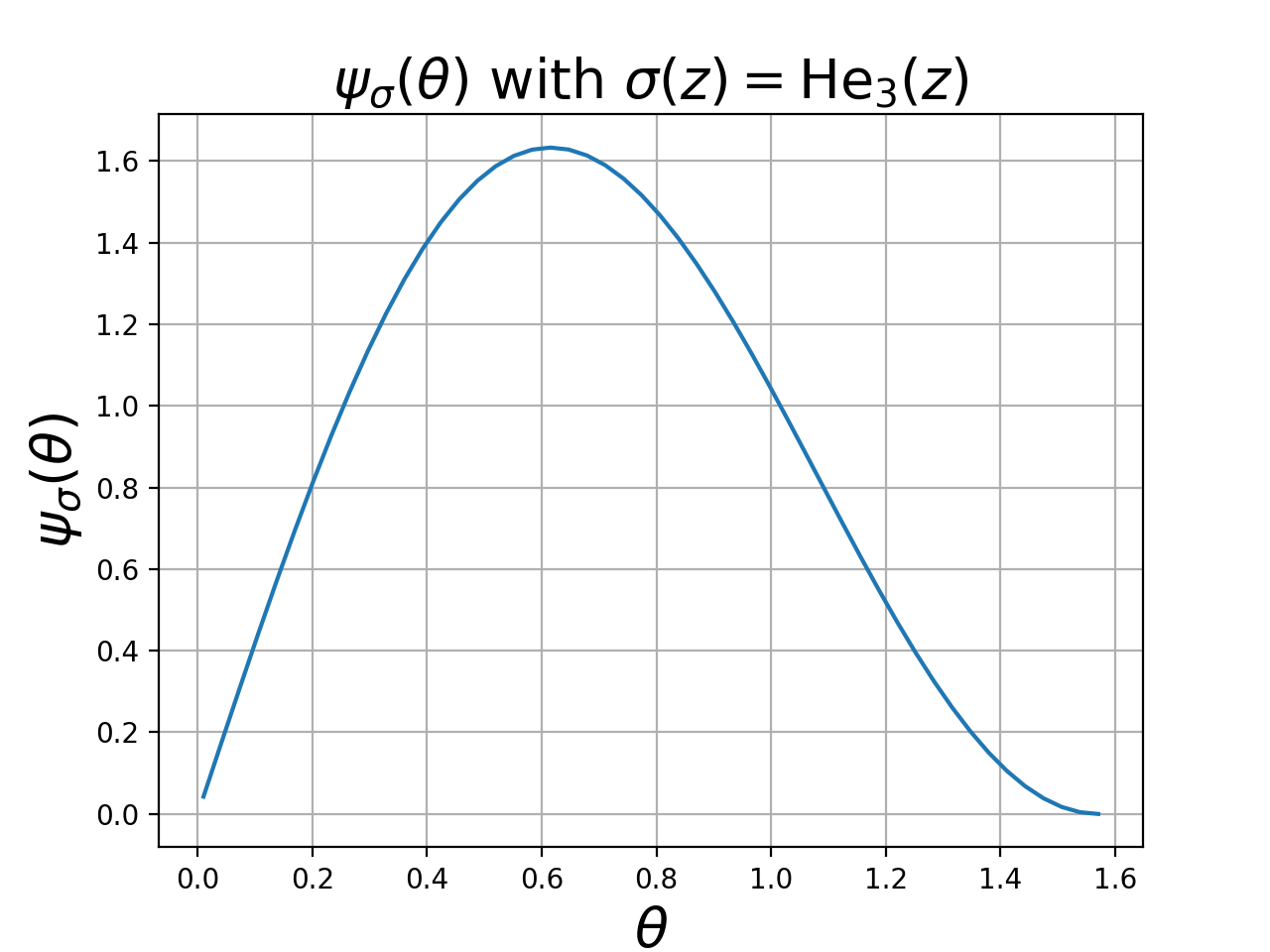}
\end{minipage}\hfill
  \begin{minipage}[b]{0.32\textwidth}
    \centering
    \includegraphics[width=\linewidth]{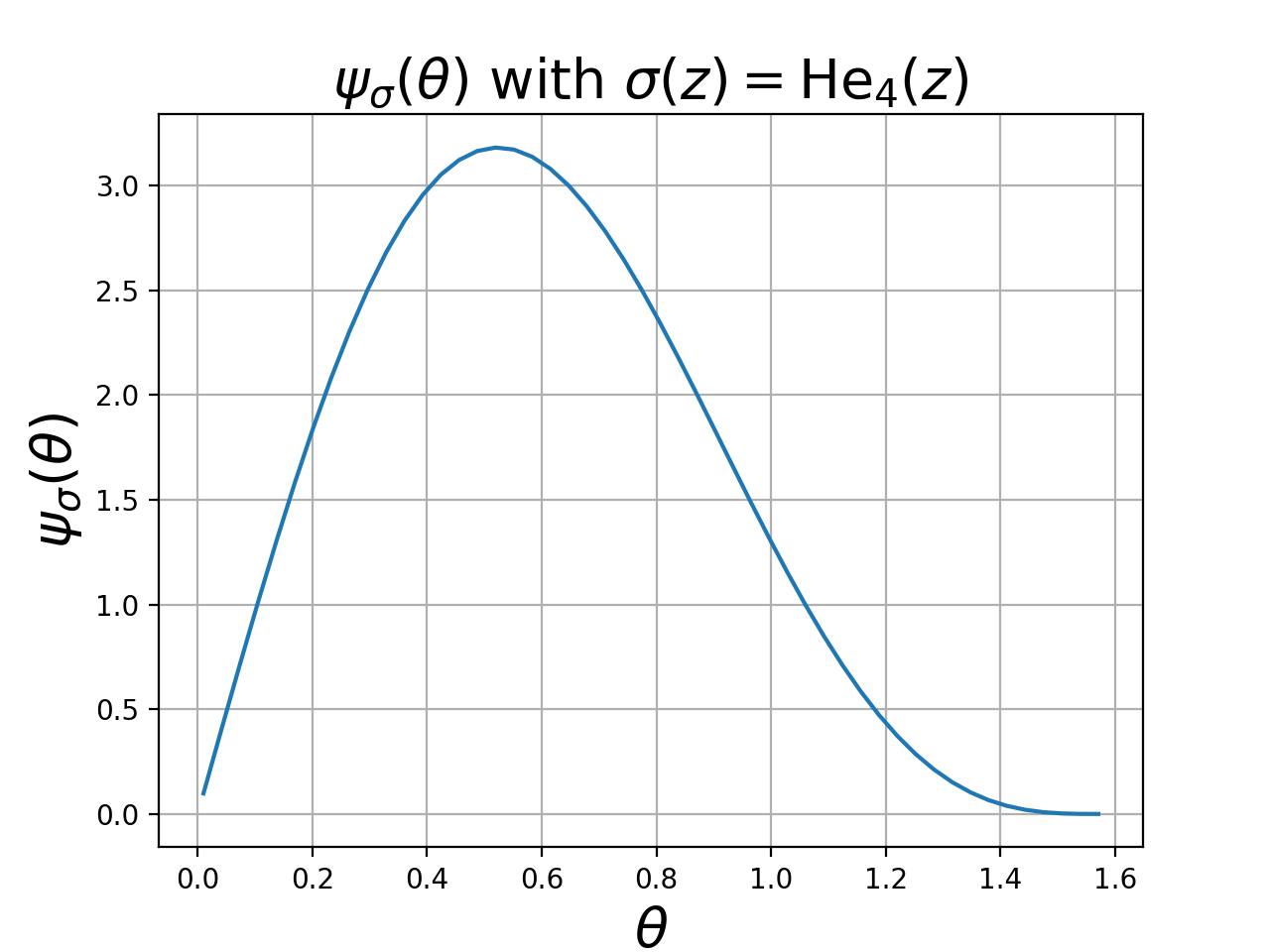}
\end{minipage}
  \vspace{1em}
  \begin{minipage}[b]{0.32\textwidth}
    \centering
    \includegraphics[width=\linewidth]{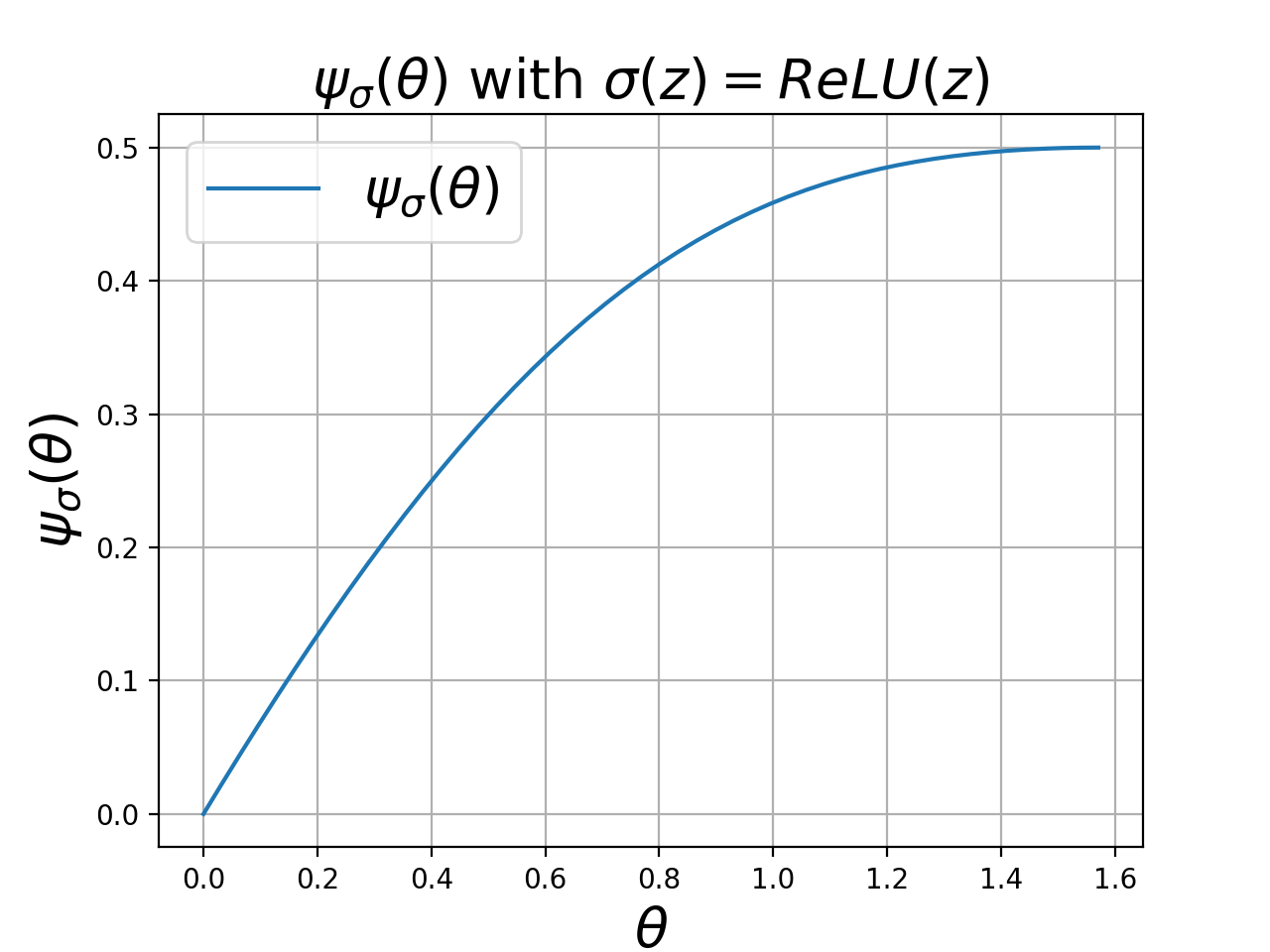}
\end{minipage}\hfill
  \begin{minipage}[b]{0.32\textwidth}
    \centering
    \includegraphics[width=\linewidth]{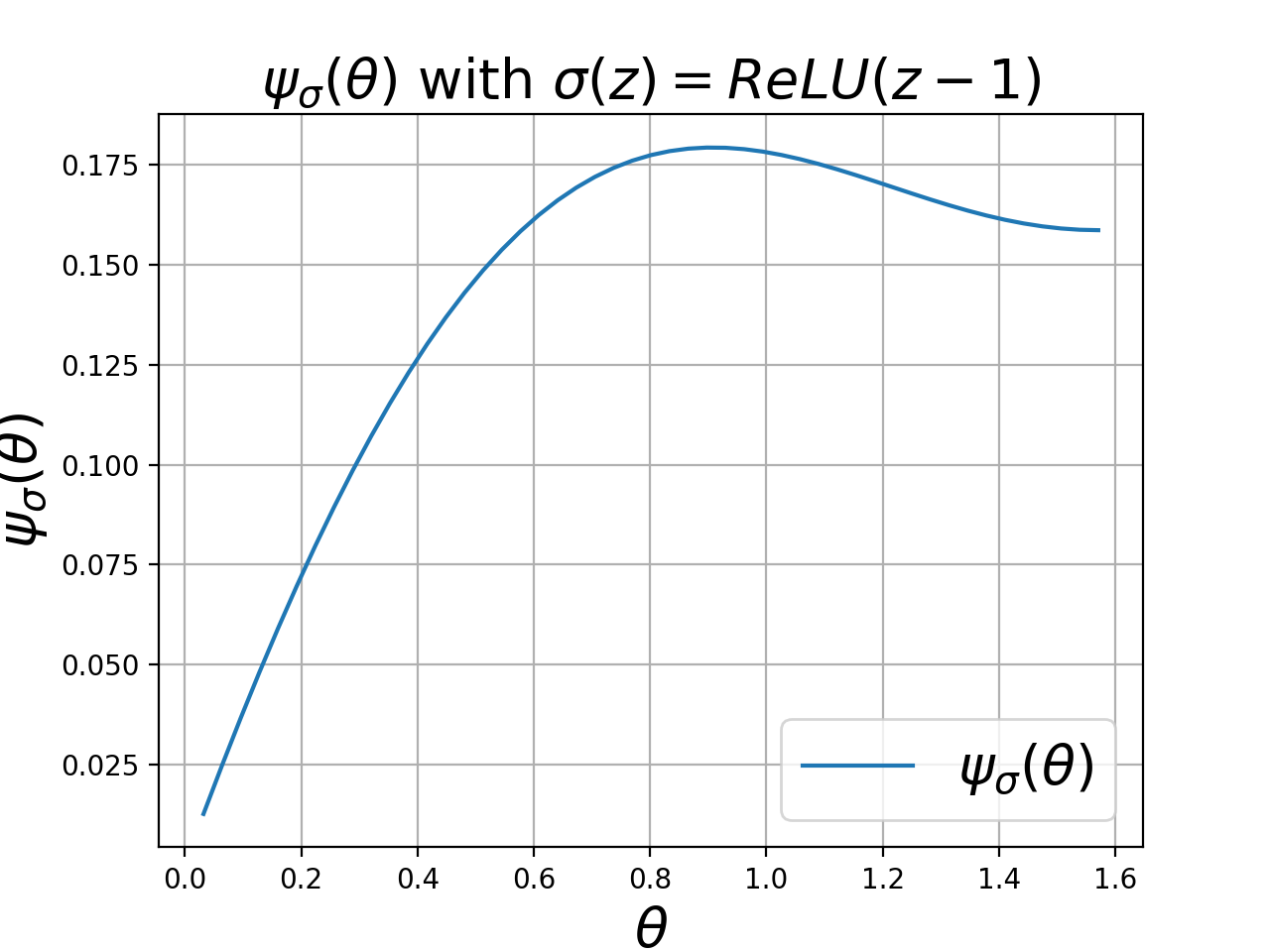}
\end{minipage}\hfill
  \begin{minipage}[b]{0.32\textwidth}
    \centering
    \includegraphics[width=\linewidth]{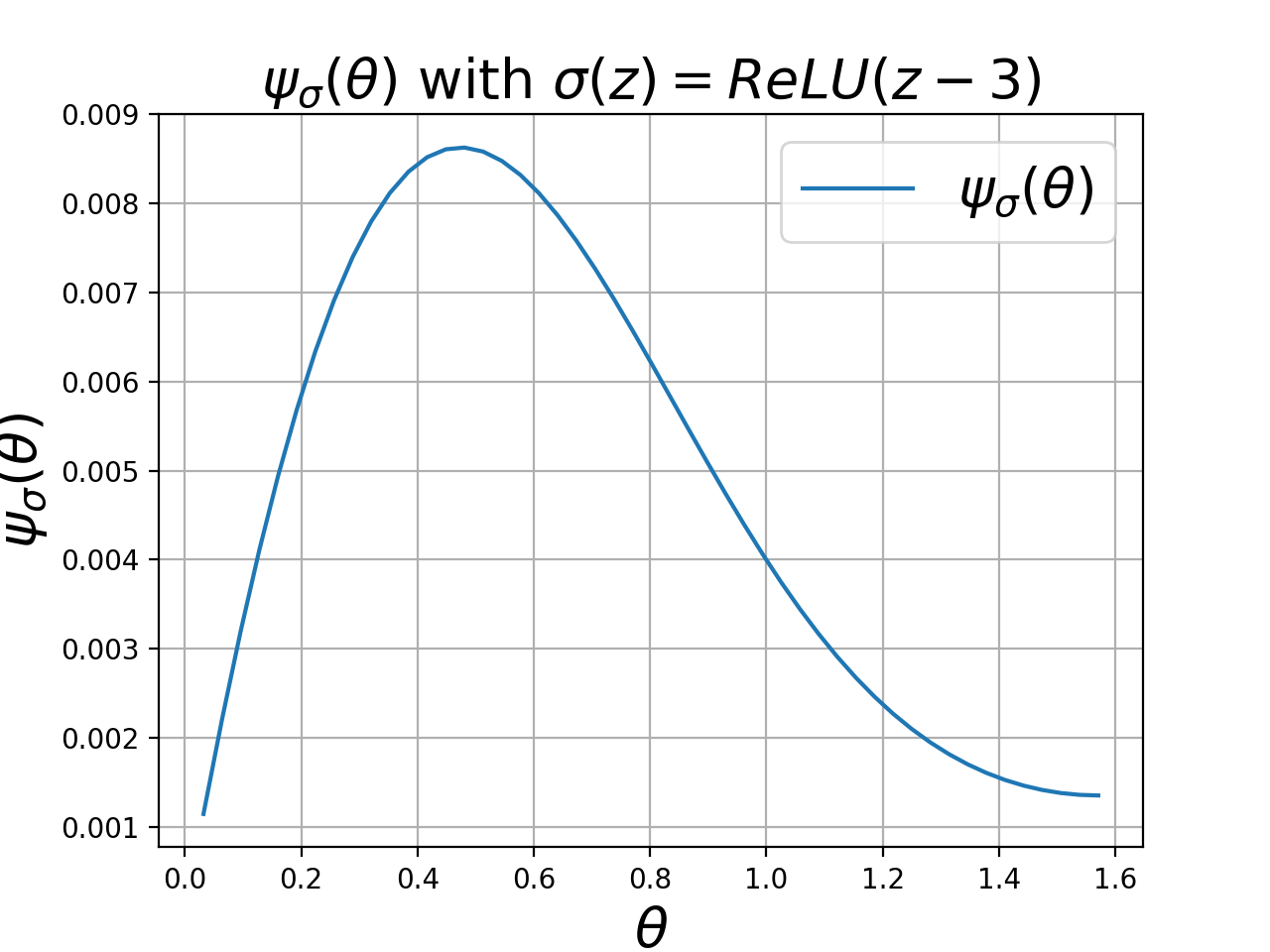}
\end{minipage}
  \caption{{\itshape (Up)}: The plot of $\psi_\sigma(\theta)$, where $\sigma(z) = \he_i(z)$, $i = 2,3,4$. {\itshape (Down)}: The plot of $\psi_\sigma(\theta)$, where $\sigma(z) = \mathrm{ReLU}(z-t)$, $t = 0,1,3$.}
  \label{app:fig:psi_sigma}
\end{figure}

\begin{proposition}[Critical Points and $L_2^2$ Error]\label{app:prop:error}
    Given $\sigma: \R\to\R$, $\sigma\in L_2(\mathcal N)$, and a distribution $\D$ of labeled examples
$(\x,y) \in \R^d \times \R$ such that 
$\D_\x = \mathcal N(\vec 0,\vec I)$, let $\wstar\in \R^d$ be such that $\Exy[(\sigma(\wstar\cdot\x)-y)^2] = \opt$. Then, for any unit vector $\vec w\in \R^d$ with $\theta=\theta(\vec w,\wstar)$ such that $\theta\leq \theta^*$, where $\theta^*$ is the $(\sigma,\theta_0,C\opt)$-Critical Point for some $\theta_0$ and $C>1$ an absolute constant, it holds that
$
        \Ey[(\sigma(\vec w\cdot\x)-y)^2]\leq O(\opt)
   +4\|\Pm{>(1/\theta^*)^2}\sigma\|_{L_2}^2.
$
\end{proposition}
\Cref{app:prop:error} provides a sufficient condition for proving that our algorithm converges to a region with the target approximation error. In particular, if we argue that the iterates of the algorithm we use land in the Critical Region, then \Cref{app:prop:error} gives us the target $L_2^2$ error.

To prove \Cref{app:prop:error}, we first prove the following technical lemma, which decomposes the error into $O(\opt)$ and error terms that depend on the properties of the activation $\sigma.$

\begin{lemma}[Error Decomposition]\label{app:prop:starpoints-l2error}Given $\sigma: \R\to\R$, $\sigma\in L_2(\mathcal N)$, and a distribution $\D$ of labeled examples
$(\x,y) \in \R^d \times \R$ such that 
$\D_\x = \mathcal N(\vec 0,\vec I)$, let $\wstar\in \R^d$ be such that $\Exy[(\sigma(\wstar\cdot\x)-y)^2] = \opt$. Then, for any unit vector $\vec w\in \R^d$ with $\theta=\theta(\vec w,\wstar)$ and any $k\in \Z_+$, it holds that
    \begin{align}\label{app:eq:error-decomposition-1}
        \Ey[(\sigma(\vec w\cdot\x)-y)^2]\leq 2\opt &+4\theta^2\|\Pm{k}\sigma'\|_{L_2}^2+4\|\Pm{>k}\sigma\|_{L_2}^2\;.
\end{align}
Furthermore, if $k\geq 2$, then for any $c \in [1, (k/2)^{1/4}]$,\begin{align}\label{app:eq:error-decomposition-2}
        \Ey[(\sigma(\vec w\cdot\x)-y)^2]\leq 2\opt &+8e^{c^2}\theta^2\|\Tre_{\sqrt{1-c^2/k}}\sigma'\|_{L_2}^2
   +4\|\Pm{>k}\sigma\|_{L_2}^2\;.
\end{align}
Finally, if $k = 0,1$, then for any $\rho\in(0,1)$ it holds
\begin{equation*}
    \Exy[(\sigma(\w\cdot\x) - y)^2]\leq 2\opt + \theta^2\|\Tr\sigma'\|_{L_2}^2  + 4\|\Pm{>k}\sigma\|_{L_2}^2.
\end{equation*}
\end{lemma}
\begin{proof}
First, we decompose the error into the minimum error (the one achieved by $\wstar$) and the alignment error (the one from the misalignment of $\vec w$ and $\wstar$). 
    By Young's inequality, we have that 
    \begin{align}\label{app:eq:prop-error-0}
         \Ey[(\sigma(\vec w\cdot\x)-y)^2)]&\leq 2 \Ey[(\sigma(\wstar\cdot\x)-y)^2)]
   +2\Ey[(\sigma(\wstar\cdot\x)-\sigma(\vec w\cdot\x))^2)]
   \nonumber\\&= 2\opt+2 \underbrace{\Exx[(\sigma(\wstar\cdot\x)-\sigma(\vec w\cdot\x))^2]}_{\text{Alignment Error}}\;.    
    \end{align}
In the rest of the proof, we bound above the alignment error.
\begin{claim}[Angle and Alignment Error]\label{app:clm:prop-error-1}
    Let $\vec w,\wstar\in \R^d$ be unit vectors and let $\theta\eqdef \theta(\vec w,\wstar)$. Then, for any $k \in \Z_+,$
    \[
   \Exx[(\sigma(\wstar\cdot\x)-\sigma(\vec w\cdot\x))^2]\leq 2\theta^2\|\Pm{k}\sigma'\|_{L_2}^2+2\Et[(\Pm{>k}\sigma(t))^2].\]
\end{claim}
\begin{proof}
By expanding the square, since $\w, \wstar$ are fixed vectors independent of $\x,$ we have that 
\begin{align}
    \Exx[(\sigma(\wstar\cdot\x)-\sigma(\vec w\cdot\x))^2]&=2 \bigg(\Et[(\sigma(t))^2]-\Exx[\sigma(\wstar\cdot\x)\sigma(\vec w\cdot\x)]\bigg)\notag\\
    &=2\bigg(\Et[(\sigma(t))^2]-\Et[\sigma(t)\Tre_{\cos\theta}\sigma(t)]\bigg), \label{app:eq:prop-error-1}
\end{align}
where in the second inequality we used that $t = \wstar\cdot\x \sim \normal(0, 1)$, decomposed $\w$ into components parallel and orthogonal to $\wstar,$ and applied the definition of $\Tre_{\cos \theta}.$ 

 Let $\sigma(t)\doteq \sum_{i=0}a_i \he_i(t)$ with $a_i = \Ez[\sigma(z)\he_i(z)]$ be the Hermite expansion of $\sigma$. Using the property that $\Tr \he_i(t)=\rho^i\he_i(t)$, for any integer $k\geq 1$ (see \Cref{lem:interchange-Tr-and-differentiation}, Part 3), we have that
 \begin{align}
     \Et[(\sigma(t))^2]-\Et[\sigma(t)\Tre_{\cos\theta}\sigma(t)]
     &= \sum_{i=1}^{+\infty} a_i^2(1-\cos^i\theta) \nonumber\\
     &=a_1^2(1 - \cos\theta) + \sum_{i=2}^{+\infty}a_i^2(1-(1-\sin^2\theta)^{i/2})\nonumber\\&\leq (1/2)a_1^2\theta^2+\theta^2\sum_{i=2}^k(i/2)a_i^2+\sum_{i=k+1}^{+\infty}a_i^2,\label{app:eq:prop-error-2}
 \end{align}
 where for the first term we used that $(1-\cos\theta) = 2\sin^2(\theta/2)\leq \theta^2/2$, and for the terms from $i=2,\ldots,k$, we used the Bernoulli inequality and that $\sin\theta\leq \theta$ for $\theta\geq 0$.

 Furthermore, note that if $\sigma(t)\doteq\sum_{i=0}a_i \he_i(t)$ is the Hermite expansion of $\sigma$, then $\sigma'(t)\doteq\sum_{i=1}a_i\sqrt{i} \he_{i-1}(t)$ is the Hermite expansion of $\sigma'$.
 Therefore, we have that
  \begin{align*}
     \Et[(\sigma(t))^2]-\Et[\sigma(t)\Tre_{\cos\theta}\sigma(t)]\leq \theta^2\|\Pm{k}\sigma'\|_{L_2}^2+\|\Pm{>k}\sigma\|_{L_2}^2,
 \end{align*}
 which, combined with \Cref{app:eq:prop-error-1}, completes the proof.
\end{proof}
To complete the proof, it remains to bound  $\|\Pm{k}\sigma'\|_{L_2}^2$ above, which is done in the following claim. 
 \begin{claim}\label{app:clm:prop-error-2}
When $k\geq 2$, for any $c \in [1, (k/2)^{1/4}]$, \[
     \|\Pm{k}\sigma'(t)\|_{L_2}^2\leq
     2e^{c^2}\|\Tre_{\sqrt{1-c^2/k}}\sigma'(t)\|_{L_2}^2.
     \]
 \end{claim}
 \begin{proof}
      Note that the $\Tre_{\sqrt{1-c^2/k}}\sigma'(t)\doteq \sum_{i=1} (1-c^2/k)^{(i-1)/2}\sqrt{i}a_i\he_{i-1}(t)$.
      Therefore, we have that
      \begin{align*}
          \|\Tre_{\sqrt{1-c^2/k}}\sigma'\|_{L_2}^2&=\sum_{i=1}^{+\infty} i (1-c^2/k)^{i-1} a_i^2\geq \sum_{i=1}^{k} i (1-c^2/k)^{i} a_i^2\;.
      \end{align*}
      By assumption, we have $c^4/k\leq 1/2$. Therefore, for any $i\leq k$, using the inequality 
      $(1-c^2/k)^i\geq (1-c^2/k)^k\geq e^{-c^2}(1-c^4/k)\geq e^{-c^2}/2$, we have that
            \[
\|\Tre_{\sqrt{1-c^2/k}}\sigma'\|_{L_2}^2\geq\sum_{i=1}^{k}ie^{-c^2}(1/2) a_i^2=e^{-c^2}(1/2) \|\Pm{k}\sigma'\|_{L_2}^2\;.
      \]
      This completes the proof.
 \end{proof}

Combining \Cref{app:clm:prop-error-2} with \Cref{app:clm:prop-error-1} and bring back the bounds to \Cref{app:eq:prop-error-0}, we have that when $k\geq 2$ and for any $c\in[1,(k/2)^{1/4}]$, it holds
\begin{equation*}
    \Exy[(\sigma(\w\cdot\x) - y)^2]\leq 2\opt + 8e^{c^2}\theta^2\|\Tre_{\sqrt{1-c^2/k}}\sigma'\|_{L_2}^2
   +4\|\Pm{>k}\sigma\|_{L_2}^2\;.
\end{equation*}

When $k = 0$, then according to \Cref{app:eq:prop-error-2}, for any $\rho\in (0,1)$, we have:
\begin{align*}
    \Exy[(\sigma(\w^*\cdot\x) - \sigma(\w\cdot\x))^2]\leq 2\|\Pm{>0}\sigma\|_{L_2}^2\leq \frac{1}{2}\theta^2\|\Tr\sigma'\|_{L_2}^2 + 2\|\Pm{>0}\sigma\|_{L_2}^2,
\end{align*}
since $\frac{1}{2}\theta^2\|\Tr\sigma'\|_{L_2}^2\geq 0$ for any $\rho\in (0,1)$.
When $k = 1$, similarly according to \Cref{app:eq:prop-error-2}, we have
\begin{equation*}
    \Exy[(\sigma(\w^*\cdot\x) - \sigma(\w\cdot\x))^2]\leq \frac{1}{2}\theta^2 a_1^2 + 2\|\Pm{>1}\sigma\|_{L_2}^2.
\end{equation*}
Observe that for any $\rho\in (0,1)$, we have $\Tr\sigma'(z)\doteq \sum_{i\geq 1} \rho^{i-1}\sqrt{i}a_i\he_{i-1}(z)$, therefore $\|\Tr\sigma'\|_{L_2}^2 = \sum_{i\geq 1} \rho^{2(i-1)} i a_i^2\geq a_1^2$. Thus, we further obtain
\begin{equation*}
    \Exy[(\sigma(\w^*\cdot\x) - \sigma(\w\cdot\x))^2]\leq \frac{1}{2}\theta^2\|\Tr\sigma'\|_{L_2}^2  + 2\|\Pm{>1}\sigma\|_{L_2}^2.
\end{equation*}
Plugging the above bounds  for the cases   $k = 0,1$ back into
\Cref{app:eq:prop-error-0} completes the proof.\end{proof}

Having proved \Cref{app:prop:starpoints-l2error}, it is not hard to see that \Cref{app:prop:error} follows as a direct corollary:

\smallskip

\begin{proof}[Proof of \Cref{app:prop:error}]
Since $\theta^*$ is the $(\sigma,\theta_0,C\opt)$-Critical point, we have $(\theta^*)^2\|\Tre_{\cos(c\theta^*)}\sigma'\|_{L_2}^2\leq C\opt$. 
We apply \Cref{app:prop:starpoints-l2error} with $k = \lfloor 1/(\theta^*)^2\rfloor$. Consider first $\theta^*\leq 1/\sqrt{2}$, which implies that $k\geq 2$. Then for any $c\in[1, (k/2)^{1/4}]$, since $c\theta^*\geq \sin(c\theta^*)$ and $c\geq 1$, it holds
\begin{align*}
    \sqrt{1 - \frac{c^2}{k}}\leq \sqrt{1 - (c\theta^*)^2}\leq \sqrt{1 - \sin^2(c\theta^*)} = \cos(c\theta^*)\leq \cos\theta^*.
\end{align*}
Thus as $\|\Tr\sigma'\|_{L_2}$ is non-decreasing with respect to $\rho$ (\Cref{fct:semi-group}), we further have $\|\Tre_{\sqrt{1 - c^2/k}}\sigma'\|_{L_2}^2\leq \|\Tre_{\cos(\theta^*)}\sigma'\|_{L_2}^2$. 
    Therefore, applying \Cref{app:prop:starpoints-l2error}, for any $\theta\leq \theta^*$, we obtain
    \begin{align*}
        \Ey[(\sigma(\vec w\cdot\x)-y)^2]&\leq 2\opt +8e^{c^2}\theta^2\|\Tre_{\sqrt{1-c^2/k}}\sigma'\|_{L_2}^2 +4\|\Pm{>k}\sigma\|_{L_2}^2\\
        &\leq 2\opt +8e^{c^2}(\theta^*)^2\|\Tre_{\cos(\theta^*)}\sigma'\|_{L_2}^2  +4\|\Pm{>(1/\theta^*)^2}\sigma\|_{L_2}^2\\
        &\leq C'\opt +4\|\Pm{>(1/\theta^*)^2}\sigma\|_{L_2}^2 \;,
    \end{align*}
    where $C'$ is an absolute constant. In particular, when $c = 1$, we have $C' = 2+ 8eC$. When $\theta^* > 1/\sqrt{2}$, then $k = 0, 1$. Choose $\rho = \cos(\theta^*)\in (0,1)$ in \Cref{app:prop:starpoints-l2error}, we have 
    \begin{align*}
        \Exy[(\sigma(\w\cdot\x)-y)^2]&\leq 2\opt + \theta^2\|\Tre_{\cos(\theta^*)}\sigma'\|_{L_2}^2 + 4\|\Pm{>k}\sigma\|_{L_2}^2 \leq (2 + C)\opt + 4\|\Pm{>(1/\theta^*)^2}\sigma\|_{L_2}^2
    \end{align*}
    In summary, for all $\theta^*\in (0,\pi/2)$, we have $\Ey[(\sigma(\vec w\cdot\x)-y)^2]\leq O(\opt)
   +4\|\Pm{>(1/\theta^*)^2}\sigma\|_{L_2}^2$.
\end{proof}

\section{Full Version of \Cref{sec:learn-sim-general}}\label{app:sec:learn-sim-general}
In this section, we present our main algorithm (\Cref{alg:GD-general-activation}) for learning GLMs under Gaussian marginals with adversarial corruptions, as stated in \Cref{def:agnostic-learning}.  
\Cref{alg:GD-general-activation} uses the main structural result of \Cref{sec:augmentation-and-landscape} (\Cref{prop:structural}) to update its iterates $\w^{(t)}$. In particular, 
for $\theta_t = \theta(\w\tth,\w^*)$, we show that 
after one gradient descent-style update, the angle $\theta_{t+1}$ shrinks by a factor $1 - c$, i.e., $\theta_{t+1}\leq (1 - c)\theta_t$, where $0<c<1$ is an absolute constant.  A crucial feature of \Cref{alg:GD-general-activation} is that 
in each iteration it carefully chooses a new value of $\rho_t$. This \textit{variable} update of $\rho_t$ ensures the `signal' of the gradient is present 
until $\w\tth$ reaches a small region centered at $\w^*$. Within this region, the agnostic noise corrupts the signal of the augmented gradient and convergence to $\w^*$ is no longer be guaranteed. However, the region that $\w\tth$ reaches is in fact the  Convergence Region $\mathcal{R}_{\sigma,\theta_0}(O(\opt))$, within which all points are solutions with the target approximation error. 
We will show in \Cref{sec:learn-monotone} that for any \textit{monotone} $(B,L)$-Regular activations, any point $\wh{\w}$ in $\mathcal{R}_{\sigma,\theta_0}(O(\opt))$ is a solution with error $C\opt + \eps$, provided that the initialized angle $\theta_0 = \theta(\w^{(0)},\w^*)$ is suitably small.

We now present our main algorithm.

\begin{algorithm}[h]
   \caption{$\mathrm{SGD-VA}$: SGD  with Variable Augmentation
}
   \label{app:alg:GD-general-activation}
\begin{algorithmic}[1]
\STATE {\bfseries Input:} Parameters $\eps, T$; Sample access to $\D$.
\STATE $[\w^{(0)}, \bar{\theta}] = \textbf{Initialization}[\sigma]$ (\Cref{app:subsec:initialization-for-monotone}); set $\rho_0 = \cos\bar{\theta}$.
\FOR{$t = 0,\dots, T$} 
\STATE Draw $n $ samples $\{(\tx\ith,y\ith)\}_{i=1}^n$ {from $\D_{\rho_t}$ using \Cref{app:alg:aug} and construct the empirical distribution $\widehat{\D}_{\rho_t}$.}
\STATE $\whg(\w^{(t)}) = -(1/\rho_t)\E_{(\tx,y)\sim\widehat{\D}_{\rho_t}}[y\sigma'(\w^{(t)}\cdot\tx)(\tx)^{\perp{\w^{(t)}}}]$ \label{app:line:grad-empirical}
\STATE $\eta_t = \sqrt{(1 - \rho_t)/2}/(4\|\wh{\g}(\w\tth)\|_2)$.
\STATE $\w^{(t+1)} = (\w^{(t)} - \eta_{t}\whg(\w^{(t)}))/\|\w^{(t)} - \eta_{t}\whg(\w^{(t)})\|_{2}$.
\STATE $\rho_{t+1} = 1 - (1 - 1/256)^2(1 - \rho_t)$\label{app:line:rho}
\ENDFOR
\STATE $\wh{\w} = \textbf{Test}[\w^{(0)},\w^{(1)},\dots, \w^{(T)}]$. (\Cref{app:alg:test})
\STATE{\bfseries Return:} $\wh{\w}$
\end{algorithmic}
\end{algorithm}

\begin{figure}[h]
    \centering
    
\begin{minipage}{\textwidth}
        \centering
          \scalebox{0.70}{
        \begin{tikzpicture}

\begin{scope}[xshift=0cm, yshift=0cm]
                \draw[thick] (0,1) rectangle (8.55,2); 

\fill[green!70] (0,1) rectangle (1.9,2);
                \fill[orange!70] (1.9,1) rectangle (6.175,2);
                \fill[black!70] (6.175,1) rectangle (8.55,2);
                
\draw[thick, red] (1.9,1) -- (1.9,2); \draw[thick, blue] (4.25,1) -- (4.25,2); \draw[thick, red] (6.175,1) -- (6.175,2); 

\node[red] at (1.9,2.2) {\(\theta^*\)};
                \node[blue] at (4.25,2.2) {\(\theta_t\)};
                \node[red] at (6.175,2.2) {\(\varphi_t\)};
                
\draw[dashed, blue] (3.775,1) -- (3.775,2);  
                \draw[dashed, red] (5.7,1) -- (5.7,2);      

\draw[->, blue, thick] (4.25,0.8) to[out=180,in=270] (3.775,1);
                \draw[->, red, thick] (6.175,0.8) to[out=180,in=270] (5.7,1);
                
\node[blue] at (3.775,0.6) {\(\theta_{t+1}\)};
                \node[red] at (5.7,0.6) {\(\varphi_{t+1}\)};
            \end{scope}

\begin{scope}[xshift=9cm, yshift=0cm] \draw[thick] (0,1) rectangle (8.55,2); 

\fill[green!70] (0,1) rectangle (1.9,2);
                \fill[orange!70] (1.9,1) rectangle (5.7,2);
                \fill[black!70] (5.7,1) rectangle (8.55,2);
                
\draw[thick, red] (1.9,1) -- (1.9,2); 

\node[red] at (1.9,2.2) {\(\theta^*\)};
                \node[blue] at (3.775,2.2) {\(\theta_{t+1}\)};
                \node[red] at (5.7,2.2) {\(\varphi_{t+1}\)};
                
\draw[thick, blue] (3.775,1) -- (3.775,2);  
                \draw[thick, red] (5.7,1) -- (5.7,2);

\end{scope}

        \end{tikzpicture}}
        \caption{Successfull Update}
    \end{minipage}

\begin{minipage}{\textwidth}
        \centering
         \scalebox{0.70}{
        \begin{tikzpicture}

\begin{scope}[xshift=0cm, yshift=-4cm] \draw[thick] (0,1) rectangle (8.55,2); 

\fill[green!70] (0,1) rectangle (1.9,2);
                \fill[orange!70] (1.9,1) rectangle (6.175,2);
                \fill[black!70] (6.175,1) rectangle (8.55,2);
                
\draw[thick, red] (1.9,1) -- (1.9,2); \draw[dashed, blue] (4.75,1) -- (4.75,2); \draw[thick, red] (6.175,1) -- (6.175,2); 

\node[red] at (1.9,2.2) {\(\theta^*\)};
                \node[blue] at (4.75,0.6) {\(\theta_{t+1}\)};
                \node[red] at (6.175,2.2) {\(\varphi_t\)};
                
\draw[thick, blue] (4.275,1) -- (4.275,2);  
                \draw[dashed, red] (5.7,1) -- (5.7,2);      

\draw[->, blue, thick] (4.275 ,0.8) to[out=180,in=270] (4.75,1);
                \draw[->, red, thick] (6.175,0.8) to[out=180,in=270] (5.7,1);
                
\node[blue] at (4.2,2.2) {\(\theta_{t}\)};
                \node[red] at (5.7,0.6) {\(\varphi_{t+1}\)};
            \end{scope}

\begin{scope}[xshift=9cm, yshift=-4cm] \draw[thick] (0,1) rectangle (8.55,2); 

\fill[green!70] (0,1) rectangle (1.9,2);
                \fill[orange!70] (1.9,1) rectangle (5.7,2);
                \fill[black!70] (5.7,1) rectangle (8.55,2);
                
\draw[thick, red] (1.9,1) -- (1.9,2); 

\node[red] at (1.9,2.2) {\(\theta^*\)};
                \node[blue] at (4.75,2.2) {\(\theta_{t+1}\)};
                \node[red] at (5.7,2.2) {\(\varphi_{t+1}\)};
                
\draw[thick, blue] (4.75,1) -- (4.75,2); \draw[thick, red] (5.7,1) -- (5.7,2);    
            \end{scope}

        \end{tikzpicture}}
        \caption{Wrong Update}
    \end{minipage}

    \caption{Illustration of \(\theta^*\), \(\theta_t\), and \(\varphi_t\) at different stages. The green region represents the Convergence Region, while the black region denotes the area that $\theta_t$ will never enter. Notably, the black region consistently expands, irrespective of whether the update is successful. The parameter $\theta_t$ is always guaranteed to never reach the black region.}\label{app:fig:alg}
\end{figure}

Our main result concerning the general setting of $(B,L)$-Regular activations is summarized in the following theorem.

\begin{theorem}\label{app:thm:main-general}Let $\eps>0$. Let $\sigma$ be a $(B,L)$-Regular activation. \Cref{app:alg:GD-general-activation} given initialization $\vec w^{(0)}$ with $\theta(\vec w^{(0)},\wstar)\leq \bar{\theta}$, runs at most  $T =O(\log(L/\eps))$ iterations, 
    draws $\tilde{\Theta}({dB^2}\log(L/\eps)/{\eps} + B^4\log(L/\eps)/\eps^2)$ samples, 
    and returns a vector $\wh{\w}$ such that with probability at least $2/3$, 
    {$\wh{\w}$ lies in the target region $\mathcal{R}_{\sigma,\theta_0}(O(\opt))$.}
Moreover,  $\calL(\wh{\w}) = O(\opt) + \eps + 4\|\Pm{>1/(\theta^*)^2}\sigma\|_{L_2}^2$.
\end{theorem}

Let us provide a roadmap of the proof.
Suppose for simplicity that we have taken enough samples so that we have access to the population gradient $\g(\w\tth)$. Furthermore, for the convenience of notation, let us use $\zeta(\rho)$ to denote the value $\sqrt{\opt}/\|\Tre_{\rho}\sigma'\|_{L_2}$. Our main tool is the structural result in \Cref{app:prop:structural}, which shows that when 
\begin{equation}\label{app:eq:condition-for-fast-convergence}
    \text{conditions for fast convergence:}\;\sin\theta_t\geq 3\zeta(\rho_t), \,\zeta(\rho_t)\eqdef \sqrt{\opt}/\|\Tre_{\rho_t}\sigma'\|_{L_2}, \rho_t\leq\cos\theta_t
\end{equation}
are satisfied, the gradient $\g(\w\tth)$ correlates strongly with the target vector $\w^*$, hence providing sufficient information about the direction of $\w^*$. This structural result enables us to decrease the angle $\theta_{t+1}$ efficiently so that $\theta_{t+1}\leq (1 - c)\theta_t$.
However, two critical problems arise:
\begin{enumerate}[leftmargin=*]
    \item If $\sin\theta_t\lesssim \zeta(\rho_t)$, then since the conditions in \Cref{app:eq:condition-for-fast-convergence} are not valid, we cannot guarantee that the angle $\theta_t$ contracts. On the other hand, since $\|\Tre_{\rho_t}\sigma'\|_{L_2}\leq \|\Tre_{\cos\theta_t}\sigma'\|_{L_2}$, it is not necessarily the case that $\sin\theta_t\lesssim \zeta(\cos\theta_t)$, therefore we also cannot assert that $\w\tth$ has reached the target region $\mathcal{R}_{\sigma,\theta_0}(C^2\opt)$. \item Suppose that the conditions in \Cref{app:eq:condition-for-fast-convergence} are satisfied, and we have contraction of angle $\theta_{t+1}\leq (1 - c)\theta_t$. Assume that $\w^{(t+1)}$ is still far away from $\w^*$ and $\theta_{t+1}\gtrsim \zeta(\cos\theta_{t+1})$, meaning that we still need to further decrease the angle between $\w^{(t+1)}$ and $\w^*$. However, it is possible that $\zeta(\cos\theta_{t+1}) \lesssim \theta_{t+1}\lesssim \zeta(\rho_t)$, because $\|\Tre_{\cos\theta_{t+1}}\sigma'\|_{L_2}\geq \|\Tre_{\rho_t}\sigma'\|_{L_2}$, as we have $\rho_t\leq \cos\theta_t\leq\cos\theta_{t+1}$ and $\|\Tr\sigma'\|_{L_2}$ is an increasing function of $\rho$ (see \Cref{fct:semi-group}). This implies that the fast convergence conditions (\Cref{app:eq:condition-for-fast-convergence}) might be invalid if we continue using $\rho_t$. Thus, we need to carefully increase $\rho_t$ to $\rho_{t+1}$ so that $\sin\theta_{t+1}\gtrsim \zeta(\rho_{t+1})$, while maintaining the other condition $\rho_{t+1}\leq \cos\theta_{t+1}$. This seems impossible since we do not have any lower bound on $\theta_{t+1}$.
\end{enumerate}
To overcome these hurdles, let us study the event $\evt_t\eqdef\{|\cos\theta_t-\rho_t|\leq \sin^2\theta_t,\, \sin\theta_t \leq C\zeta(\rho_t)\}$. We first observe that when $\evt_t$ is satisfied, then, since $\sin\theta_t \leq C\zeta(\rho_t)$ the algorithm may not be converging anymore, as discussed in Case 1 above. However, since $\evt_t$ also satisfies $|\cos\theta_t-\rho_t|\leq \sin^2\theta_t$, one can show that  $\zeta(\rho_t)\approx\zeta(\cos\theta_t)$, therefore, we have that $\sin\theta_t \leq C\zeta(\cos\theta_t)$. In other words, we can certify that $\w\tth$ lies in the target region $\mathcal{R}_{\sigma,\theta_0}(C^2\opt)$. This solves the first problem above. 

Now suppose $\evt_t$ is not satisfied. We use induction to show that updating $\rho_t$ by Line \ref{app:line:rho}, it always holds $\rho_{t+1}\leq \cos\theta_{t+1}$. To see this, suppose $\rho_{t}\leq \cos\theta_{t}$ holds at iteration $t$ and $\evt_t$ is not satisfied. Then if we have $\sin\theta_t \geq C\zeta(\rho_t)$, the conditions in \Cref{app:eq:condition-for-fast-convergence} are satisfied hence we have control of $\theta_{t+1}$. We can then show that $\rho_{t+1}\leq \cos\theta_{t}$ and $\sin\theta_{t+1}\gtrsim \zeta(\rho_{t+1})$ with $\rho_{t+1}$ defined by Line \ref{app:line:rho}. On the other hand, if $|\cos\theta_t-\rho_t|\geq \sin^2\theta_t$, then since $\rho_{t}\leq \cos\theta_{t}$ we know that $\rho_t$ is much smaller compared to $\cos\theta_t$. Thus, since we are taking  small gradient steps and making very small increments to $\rho_t$, we have that $\rho_{t+1}\leq \cos\theta_{t+1}$ and $\sin\theta_{t+1}\gtrsim \zeta(\rho_{t+1})$ continue to hold. This resolves the second problem.
See \Cref{app:fig:alg} for a visual illustration of the mechanism of \Cref{app:alg:GD-general-activation}.

We can now proceed to the proof of \Cref{app:thm:main-general}.

\smallskip 

\begin{proof}[Proof of \Cref{app:thm:main-general}]
In the proof, we denote the angle between $\w^{(t)}$ and $\w^*$ 
by $\theta_t = \theta(\w^{(t)},\w^*)$ and denote $\wh{\g}(\w\tth)$ by $\wh{\g}\tth$. {After initialization, 
it holds $\theta_0 = \theta(\w^{(0)},\w^*)\leq \bar{\theta}$.} 
Furthermore, the algorithm uses the following parameters: 
$\varphi_t = (1/\sqrt{2})(1 - \beta)^t\sin\bar{\theta}, \; 
\beta = 1/256;\;\rho_t \eqdef 1 - 2\varphi_t^2, \; 
\rho_0 = \cos\bar{\theta};\; 
\eta_t = \varphi_t/(4\|\wh{\g}\tth\|_{2})$.
Note that if $\eps\geq C \opt$, then we can run the algorithm 
with $\eps'=\eps/(2C)$ and assume that we have more 
noise of order $\opt'=2\eps'$. In this case, 
the final error bound will be $C\opt'\leq \eps/2\leq \opt+\eps$. 
So, without loss of generality, we can assume that $\eps\leq \opt$.
The goal of the algorithm is to converge to a vector in the region $\mathcal{R}_{\sigma,\theta_0}(C\opt)$ where $C>0$  is an absolute constant. For this reason, we consider the following event 
\begin{equation*} 
  \evt_t\eqdef\bigg\{|\cos\theta_t-\rho_t|\leq \sin^2\theta_t,\,  \sin\theta_t \leq \frac{C\sqrt{\opt}}{\|\Tre_{\rho_t}\sigma'\|_{L_2}}\bigg\},
\end{equation*}
where $C>0$ is an absolute constant.

We argue that if $\evt_t$ is satisfied at some iteration $t$, then the algorithm converges to a {vector that lies in the $\mathcal{R}_{\sigma,\theta_0}(C\opt)$ region.}
We consider two cases, the first case is if $\rho_t\geq \cos\theta_t$ and the second one if $\rho_t\leq \cos\theta_t$. Assume first that $1>\rho_t\geq \cos\theta_t$. Then, since $\|\Tr\sigma'\|_{L_2}$ is an increasing function with respect to the variable $\rho$ (see \Cref{fct:semi-group}), and $\evt_t$ implies $\sin\theta_t\leq C\sqrt{\opt}/\|\Tre_{\rho_t}\sigma'\|_{L_2}$, we also have that $\sin\theta_t\leq C\sqrt{\opt}/\|\Tre_{\cos\theta_t}\sigma'\|_{L_2}$ and therefore that means that $\w\tth$ is inside the region $\mathcal{R}_{\sigma,\theta_0}(C^2\opt)$. Next, we consider the case where $\rho_t\leq \cos\theta_t$. 
Since $\evt_t$ implies that $|\cos\theta_t - \rho_t|\leq\sin^2\theta_t$, we further have
\begin{align*}
    \rho_t\geq \cos\theta_t - \sin^2\theta_t\geq \cos^2\theta_t - \sin^2\theta_t = \cos(2\theta_t).
\end{align*}
Therefore, we have $\sin\theta_t\leq C\sqrt{\opt}/\|\Tre_{\rho_t}\sigma'\|_{L_2}\leq C\sqrt{\opt}/\|\Tre_{\cos(2\theta_t)}\sigma'\|_{L_2}$, i.e., $\psi_{\sigma}(2\theta_t)\leq 2C\sqrt{\opt}$.  Let $\theta^*$ be the $(\sigma,\theta_0, 4C^2\opt)$-Critical Point, we thus have $2\theta_t\leq \theta^*$ and $\w\tth$ is inside the region $\mathcal{R}_{\sigma,\theta_0}(4C^2\opt)$.
Now since $\theta_t\leq \theta^*$, applying \Cref{app:prop:error} yields
\begin{align*}
    \Exy[(\sigma(\w\tth\cdot\x) - y)^2]\leq O(\opt) + 4\|\Pm{>1/(\theta^*)^2}\sigma\|_{L_2}^2. \end{align*}
 indicating that $\w\tth$ is a solution that achieves the target error.

We proceed to show that the algorithm is guaranteed to generate a solution $\w^{(t^*)}$ that satisfies the event $\evt_{t^*}$ at some iteration $t^*\leq T= O(\log(\|\sigma'\|_{L_2}^2/\opt))$. Our strategy is to prove that in every iteration $t\leq t^*$, it holds that $\rho_t\leq \cos\theta_t$, due to the careful design of the algorithm. Furthermore, we guarantee that $\rho_t$ can grow geometrically; therefore, we obtain an exponentially growing lower bound on $\cos\theta_t$, which implies that $\sin\theta_t$ shrinks at a linear rate and hence the event $\evt_t$ will eventually be satisfied at some iteration $t^*$.

\begin{claim}\label{app:clm:induction}
   Let $t'$ be the maximum $t\in[0,T]$ such that for all $t=0,\ldots,t'$, $\evt_t$ is not satisfied. Then, for all $t\leq t'$, it holds that $\rho_t\leq\cos\theta_t$.
\end{claim}
\begin{proof}[Proof of \Cref{app:clm:induction}]
    We use induction to show the claim that $\rho_t\leq\cos\theta_t$ for all the iterations $t=0,\ldots,t'$ where the event $\evt_t$ is not satisfied.
    
\noindent\textbf{Base Case $t=0$. } Recall that:
\begin{gather*}
    \varphi_t = (1/\sqrt{2})(1 - \beta)^t\sin\bar{\theta}, \;\beta = 1/256;\;\rho_t \eqdef 1 - 2\varphi_t^2, \;\rho_0 = \cos\bar{\theta};\;\eta_t = \varphi_t/(4\|\wh{\g}\tth\|_{2}).
\end{gather*}
Therefore, since $\theta_0\leq \bar{\theta}$, $\rho_0 = \cos(\bar{\theta})\leq \cos\theta_0$ is satisfied in the base case. 

\noindent\textbf{Induction Step.} For the induction step, suppose that $\evt_t$ is not satisfied, in other words, we have either $|\cos\theta_t - \rho_t|\geq \sin^2\theta_t$ or $\sin\theta_t\geq C\sqrt{\opt}/\|\Tre_{\rho_t}\sigma'\|_{L_2}$. Assume that $\rho_t\leq \cos\theta_t$ for the iterations $0,\dots, t$. We argue that $\rho_{t+1}\leq \cos\theta_{t+1}$ continues to hold after one iteration.

\noindent\textbf{Case I. } Consider first the case where $\sin\theta_t\geq C\sqrt{\opt}/\|\Tre_{\rho_t}\sigma'\|_{L_2}$. We study the distance between $\w\tth$ and $\w^*$ after one iteration from $t$ to $t+1$. {Since $\wh{\g}\tth$ is orthogonal to $\w\tth$, it must be $\|\w\tth - \eta_t\wh{\g}\tth\|_2\geq 1$, therefore, $\w^{(t+1)} = \proj_\B(\w\tth - \eta_t\wh{\g}\tth)$.} By the non-expansiveness of the projection operator, we have
    \begin{align}\label{app:eq:general-activation-main-thm-decrease-0}
        \|\w^{(t+1)} - \w^*\|_{2}^2  &= \|\proj_{\B}(\w\tth - \eta_{t}\wh{\g}\tth) - \w^*\|_{2}^2 \leq \|\w\tth - \eta_{t}\wh{\g}\tth - \w^*\|_{2}^2 \nonumber\\
        &= \|\w\tth - \w^*\|_{2}^2 + \eta_t^2\|\wh{\g}\tth\|_{2}^2 - 2\eta_{t}\wh{\g}\tth(\w\tth - \w^*)\nonumber\\
        &= \|\w\tth - \w^*\|_{2}^2 + \eta_t^2\|\wh{\g}\tth\|_{2}^2 + 2\eta_{t}\wh{\g}\tth\cdot\w^*.
    \end{align} 
    Next, we use the following lemma about the concentration of $\wh{\g}$. 
    \begin{restatable}{lemma}{SampleComplexityAlg}\label{app:lem:sample}
    Suppose $\sigma$ is a $(B,L)$-Regular activation. If $0<\rho\leq \cos\theta<1$ and $\sin\theta\geq 4(\sqrt{\opt})/\|\Tr\sigma'\|_{L_2}$, then using
    \begin{equation*}
        n=\Theta\bigg(\frac{dB^2 }{\sin^2\theta\|\Tr\sigma'\|_{L_2}^2\delta}\bigg)
    \end{equation*}
   samples, with probability at least $1 - \delta$, we have
\begin{gather*}
    \|\wh{\g}(\w)\|_{2}\leq (3/2)\Ex[\Tre_{\cos\theta}\sigma'(\w\cdot\x)\Tr\sigma'(\w\cdot\x)]\sin\theta\;,\\
    \wh{\g}(\w)\cdot\w^*\leq -\frac{1}{2}\Ex[\Tre_{\cos\theta}\sigma'(\w\cdot\x)\Tr\sigma'(\w\cdot\x)]\sin^2\theta.
\end{gather*}
    \end{restatable}
    We defer the proof of \Cref{app:lem:sample} to \Cref{app:sec:proof-of-concentration-lemma}. 
    Using \Cref{app:lem:sample}, we know that 
with a batch size of
    \begin{equation*}
        n = \Theta\bigg(\frac{dB^2 }{\sin^2\theta\|\Tre_{\rho_t}\sigma'\|_{L_2}^2\delta}\bigg)\leq \Theta\bigg(\frac{dB^2 }{\eps\delta}\bigg)\end{equation*}
    and if the following conditions are satisfied
    \begin{gather}
        \rho_{t}\leq \cos\theta_{t}, \label{app:eq:general-activation-main-thm-condition-for-rho} \\
        \;\text{and } \sin\theta_{t}\geq C \sqrt{\opt}/\|\Tre_{\rho_t}\sigma'\|_{L_2}, \label{app:eq:general-activation-main-thm-condition-for-theta}
    \end{gather}
   then, with probability at least $1 - \delta$, we have that
    \begin{gather}
        \wh{\g}\tth\cdot\w^*\leq -(1/2)\Ex[\Tre_{\cos\theta_{t}}\sigma'(\w\tth\cdot\x)\Tre_{\rho_t}\sigma'(\w\tth\cdot\x)]\sin^2\theta_{t},\;\label{app:eq:general-activation-main-thm-sharpness}\\
        \|\wh{\g}\tth\|_{2}\leq 2 \Ex[\Tre_{\cos\theta_{t}}\sigma'(\w\tth\cdot\x)\Tre_{\rho_t}\sigma'(\w\tth\cdot\x)]\sin\theta_{t}. \label{app:eq:general-activation-main-thm-smoothness}
    \end{gather}
Combining \Cref{app:eq:general-activation-main-thm-sharpness} with \Cref{app:eq:general-activation-main-thm-smoothness}we get $\wh{\g}\tth\cdot\w^*\geq (1/4)\|\wh{\g}\tth\|_{2}\sin\theta_t$. Therefore, bringing in our choice of stepsize $\eta_t = \varphi_t/(4\|\wh{\g}\tth\|_{2})$, and noticing $\sin\theta_t\geq \sin(\theta_t/2)$ we obtain:
\begin{align}\label{app:eq:general-activation-main-thm-decrease-1}
    4\sin^2(\theta_{t+1}/2) &= \|\w^{(t+1)} - \w^*\|_{2}^2 \leq \|\w\tth - \w^*\|_{2}^2 + \frac{\varphi_t^2}{16} - \frac{\varphi_t \sin\theta_t}{8} \nonumber\\
    &\leq 4\sin^2(\theta_t/2) + \frac{\varphi_t}{16}(\varphi_t - 2\sin(\theta_t/2)).
\end{align}
Since, by assumption, we have $\rho_t\leq \cos\theta_t$, it holds $2\sin^2(\theta_t/2) = 1 - \cos\theta_t\leq 1 - \rho_t = 2\varphi_t^2$, in other words, $\sin(\theta_t/2)\leq \varphi_t$. Consider first the case that $\varphi_t\geq \sin(\theta_{t}/2)\geq (3/4)\varphi_t$. Then, according to \Cref{app:eq:general-activation-main-thm-decrease-1} we get
    \begin{align*}
        4\sin^2(\theta_{t+1}/2)\leq 4\sin^2(\theta_t/2) - \frac{1}{32}\varphi_t^2\leq 4(1 - 1/128)\varphi_t^2.
    \end{align*}
    Hence, since $1-1/128\leq (1 - 1/256)^2$, we get $\sin(\theta_{t+1}/2)\leq (1 - 1/256)\varphi_t = \varphi_{t+1}$.

    On the other hand, if $\sin(\theta_{t}/2)\leq (3/4)\varphi_t$, then by the triangle inequality and the non-expansiveness of the projection operator, we have
    \begin{align*}
        2\sin(\theta_{t+1}/2) &= \|\w^{(t+1)} - \w^*\|_{2} = \|\proj_{\B}(\w\tth - \eta_t\wh{\g}(\w\tth)) - \w^*\|_2\\
        &\leq \|\w\tth - \w^* - \eta_t\wh{\g}\tth\|_{2}\\
        &\leq \|\w\tth - \w^*\|_{2} + (1/4)\varphi_t = 2\sin(\theta_{t}/2) + (1/4)\varphi_t\\
        &\leq (3/2)(\varphi_t/2) + (1/4)\varphi_t \leq  (7/4)\varphi_t\;.
    \end{align*}
    Therefore, it holds that $\sin(\theta_{t+1}/2)\leq (7/8)\varphi_t\leq \varphi_{t+1}$.

    To conclude, we proved that if $\sin\theta_t\geq C\sqrt{\opt}/\|\Tre_{\rho_t}\sigma'\|_{L_2}$ and $\rho_t\leq \cos\theta_t$, we have $\sin(\theta_{t+1}/2)\leq \varphi_{t+1}$ after one step of the algorithm, which immediately implies that $\cos\theta_{t+1} = 1 - 2\sin^2(\theta_{t+1}/2)\geq 1 - 2\varphi_{t+1}^2 = \rho_{t+1}$. Note that when $\sin\theta_t\geq C\sqrt{\opt}/\|\Tre_{\rho_t}\sigma'\|_{L_2}$, our argument indicates that $\cos\theta_{t+1}\geq \rho_{t+1}$ holds regardless of whether $|\cos\theta_t - \rho_t|\leq \sin^2\theta_t$ or not.

\smallskip 

\noindent\textbf{Case II. }  It remains to consider the case where $|\cos\theta_t - \rho_t|\geq \sin^2\theta_t$. In fact, we consider the setting where $\cos\theta_t - \rho_t \geq \sin^2\theta_t$, because from the induction argument we have $\rho_t\leq \cos\theta_t$. Observe that this case only requires discussing the setting where $\sin\theta_t\leq C\sqrt{\opt}/\|\Tre_{\rho_t}\sigma'\|_{L_2}$, because if $\sin\theta_t\geq C\sqrt{\opt}/\|\Tre_{\rho_t}\sigma'\|_{L_2}$ then our previous argument already implies that $\rho_{t+1}\leq \cos\theta_{t+1}$ after one iteration. Therefore, assuming that $\sin\theta_t\leq C\sqrt{\opt}/\|\Tre_{\rho_t}\sigma'\|_{L_2}$, applying triangle inequality and the non-expansiveness of projection operator, it holds
    \begin{align*}
        2\sin(\theta_{t+1}/2)&=\|\w^{(t+1)}- \w^*\|_{2} = \|\proj_{\B}(\w\tth - \eta_t\wh{\g}(\w\tth)) - \w^*\|_{2}\leq \|\w\tth - \eta_t\wh{\g}\tth - \w^*\|_{2}\\
        &\leq \|\w\tth - \w^*\|_{2} + \eta_2\|\wh{\g}\tth\|_{2} = 2\sin(\theta_t/2) + \varphi_t/4.
    \end{align*}
    Using the assumption $\cos\theta_t - \rho_t\geq \sin^2\theta_t$, we observe that 
    \begin{align*}
        1- \sin^2(\theta_t/2) - (1 - 2\varphi_t^2)\geq \sin^2\theta_t\geq 2\sin^2(\theta_t/2),
    \end{align*}
    in other words, we have $\sin(\theta_t/2)\leq \sqrt{2/3}\varphi_t$. Hence, it holds
    \begin{equation*}
        \sin(\theta_{t+1}/2)\leq \sin(\theta_t/2) + \varphi_t/8\leq (\sqrt{2/3}+1/8)\varphi_t\leq (1 - 1/256)\varphi_t = \varphi_{t+1}.
    \end{equation*}
    Since $\sin(\theta_{t+1}/2)\leq \varphi_{t+1}$, using similar argument we have $\cos\theta_{t+1}\geq 1- 2\varphi_{t+1}^2 = \rho_{t+1}$, therefore the induction argument continues to hold at step $t+1$.
\end{proof}

    In conclusion, from \Cref{app:clm:induction}, we have $\cos\theta_{t}\geq \rho_t$ holds for all the iterations $t=0,\dots,t^*-1$. It remains to show that the event $\evt_{t^*}$ is satisfied at some iteration $t^*$.
    \begin{claim}\label{app:clm:max_steps}
       If $T = c\log(\|\sigma'\|_{L_2}^2/\opt)$, where $c>0$ is a sufficiently large absolute constant, then there exists $t^*\leq T$, so that the event $\evt_{t^*}$ is satisfied.
    \end{claim}
    \begin{proof}[Proof of \Cref{app:clm:max_steps}]
        Since $\sin\theta_t\leq \varphi_t \leq(1 - \beta)^t$ for all the iterations $0,1,\dots,t$ where the event $\evt_t$ is not satisfied, we have $\theta_t \to 0$. After at most $T = (1/\beta)\log(\|\sigma'\|_{L_2}/\sqrt{\opt}) = O(\log(\|\sigma'\|_{L_2}^2/\opt))$ iterations, it must hold that $\sin\theta_t\leq \sqrt{\opt}/\|\sigma'\|_{L_2}$. Note that $\sqrt{\opt}/\|\sigma'\|_{L_2}\leq \sqrt{\opt}/\|\Tr\sigma'\|_{L_2}$ for any $\rho\in(0,1)$ therefore there exists an iteration $t^*$ for which $\evt_{t^*}$ is satisfied.
    \end{proof}
  Therefore, we guarantee that $\evt_t$ will be satisfied in at most $T = O(\log(\|\sigma'\|_{L_2}^2/\opt))=O(\log(L/\eps))$ iterations. Setting $\delta = 2/(3T)$ and using a union bound, we have that at most 
    \begin{equation*}
        N_1 = nT = {\Theta}\bigg(\frac{dB^2 T^2}{\eps}\bigg) = \tilde{\Theta}\bigg(\frac{dB^2\log(L/\eps)}{\eps}\bigg)
    \end{equation*}
    samples suffices to guarantee that the algorithm generates a target solution with probability at least $2/3$. To pick out the target vector $\wh{\w}$ from the list, we can apply a testing procedure with $N_2 = \tilde{\Theta}(B^4/\eps^2)$ samples (see  \Cref{app:alg:test} and \Cref{app:lem:test} in \Cref{app:sec:testing}). Thus, in summary, the sample complexity is $N_1 + N_2 = \Tilde{\Theta}({dB^2\log(L/\eps)}/{\eps} + B^4\log(L/\eps)/\eps^2)$.
    \end{proof}

\subsection{Finding the Best Parameter}\label{app:sec:testing}
As we showed in \Cref{app:thm:main-general}, \Cref{app:alg:GD-general-activation} returns a list of vectors that contains at least one vector $\wh{\w}$ in the target region $\mathcal{R}_{\sigma,\theta_0}(O(\opt) + \eps)$. We now present a simple testing algorithm to identify one of such target vectors.

\begin{algorithm}[h]
   \caption{Testing}
   \label{app:alg:test}
\begin{algorithmic}[1]
\STATE {\bfseries Input:} Vectors $\{\w^{(0)}, \dots, \w^{(T)}\}$;  Number of Samples $m$
\STATE Sample $\{(\vec x^{(1)},y^{(1)}),\ldots,(\vec x^{(m)},y^{(m)})\}$ from $\D$ and construct the empirical distribution $\widehat{\D}$.
\STATE For $t=0,\dots,T$, let $\wh{\calL}(\w\tth) \gets \E_{(\x,y)\sim \widehat{D}}[(\sigma(\w\tth\cdot\x) - y)^2]$ 
\STATE $\wh{\w} = \argmin\{\w\tth,t\in[T]: \wh{\calL}(\w\tth)\}$.
\STATE {\bfseries Output:} $\wh{\w}$.
\end{algorithmic}
\end{algorithm}

\begin{lemma}[Testing]\label{app:lem:test}
    Let $\sigma$ be a $(B,L)$-Regular activation. Let $\{\w\tth\}_{t\in[T]}$ be the list of vectors generated by \Cref{app:alg:GD-general-activation} with $T = O(\log(L/\eps))$. Let $t^*\in[T]$ be the index such that $\calL(\w\tstrth) \in \argmin_{t\in[T]}\calL(\w\tth)$. We have 
    \begin{enumerate}[leftmargin=*]
        \item If $\|\Pm{>1/(\theta_{t^*})^2}\sigma\|_{L_2}^2\leq (\theta_{t^*})^2\|\Tre_{\cos(\theta_{t^*})}\sigma'\|_{L_2}^2$ for an absolute constant $C$, then using 
        \begin{equation*}
            m \leq \Theta\bigg(\frac{B^2\log(L/\eps)}{\eps}\bigg)
        \end{equation*}
        samples, \Cref{app:alg:test} finds a vector $\wh{\w}\in\{\w\tth\}_{t\in[T]}$ such that $\wh{\w}\in\mathcal{R}_{\sigma,\theta_0}(O(\opt) + \eps)$ and $\calL(\wh{\w})\leq O(\opt) + \eps$.
        \item Otherwise, using 
        \begin{equation*}
    m \leq \Theta\bigg(\frac{B^4\log\log(L/\eps)}{\eps^2}\bigg)
\end{equation*}
samples, \Cref{app:alg:test} outputs a vector $\wh{\w}\in\{\w\tth\}_{t\in[T]}$ such that $\wh{\w}\in\mathcal{R}_{\sigma,\theta_0}(O(\opt) + \eps)$ and $\calL(\wh{\w})\leq O(\opt) + \eps + 4\|\Pm{>1/(\theta_{t^*})^2}\sigma\|_{L_2}^2$.
    \end{enumerate}
\end{lemma}

\begin{proof}
    Let $t^* \in \argmin_{t\in[T]}\calL(\w\tth)$. Let $\ell(\w;\x,y)\eqdef (\sigma(\w\cdot\x)-y)^2$, and $\Delta(\w^1,\w^2)\eqdef \ell(\w^1;\x,y) - \ell(\w^2;\x,y)$. Given a data set $\{(\x\ith,y\ith)\}_{i=1}^m$, we denote by $\wh{\calL}(\w)$ the empirical version of $\calL(\w)$, i.e., $\wh{\calL}(\w) = (1/m)\sum_{i=1}^m \ell(\w;\x\ith,y\ith)$.

    Consider first the case when $\|\Pm{>1/(\theta_{t^*})^2}\sigma\|_{L_2}^2\leq C(\theta_{t^*})^2\|\Tre_{\cos(\theta_{t^*})}\sigma'\|_{L_2}^2$ for an absolute constant $C$. Our goal is to show that $\wh{\calL}(\w\tstrth)$ can be separated from all $\wh{\calL}(\w\tth)$ for all $\w\tth$ with large $L_2^2$ error.
As shown in \Cref{app:claim:bound-y}, when $\sigma$ is a $(B,L)$-Regular activation, labels $y$ can be assumed to be bounded above by $B$ without loss of generality. Therefore, the variance of $\Delta(\w^1,\w^2)$ is bounded by 
\begin{align*}
 \Exy[\Delta^2(\w^1,\w^2)]& =\Exy[(\sigma(\w^1\cdot\x) - \sigma(\w^2\cdot\x))^2(\sigma(\w^1\cdot\x) + \sigma(\w^2\cdot\x) - 2y)^2] \\
 &\leq 16B^2   \Ex[(\sigma(\w^1\cdot\x) - \sigma(\w^2\cdot\x))^2]\;.
\end{align*}
On the other hand, suppose without loss of generality that $\calL(\w^1)\geq \calL(\w^2)$; then, the expectation of $\Delta(\w^1,\w^2)$ can be bounded below by
\begin{align*}
    &\;\Exy[\Delta(\w^1,\w^2)]\\
    =\; &\Exy[(\sigma(\w^1\cdot\x) - \sigma(\w^2\cdot\x))(\sigma(\w^1\cdot\x) - \sigma(\w^2\cdot\x) + 2\sigma(\w^2\cdot\x) - 2y)]\\
    =\; &\Ex[(\sigma(\w^1\cdot\x) - \sigma(\w^2\cdot\x))^2] + 2\Exy[(\sigma(\w^1\cdot\x) - \sigma(\w^2\cdot\x))(\sigma(\w^2\cdot\x) - y)]\\
    \geq \; & \Ex[(\sigma(\w^1\cdot\x) - \sigma(\w^2\cdot\x))^2] - \frac{1}{2}\Ex[(\sigma(\w^1\cdot\x) - \sigma(\w^2\cdot\x))^2] - 2\calL(\w^2)\\
    =\; &\frac{1}{2}\Ex[(\sigma(\w^1\cdot\x) - \sigma(\w^2\cdot\x))^2] - 2\calL(\w^2).
\end{align*}
where in the last inequality we used Young's inequality $ab\geq -((1/4)a^2 + b^2)$.

Now consider first $\calL(\w^2)\leq (1/8)\Ex[(\sigma(\w^1\cdot\x) - \sigma(\w^2\cdot\x))^2]$. Then, we have
\begin{equation*}
    \Exy[\Delta(\w^1,\w^2)]\geq \frac{1}{4}\Ex[(\sigma(\w^1\cdot\x) - \sigma(\w^2\cdot\x))^2].
\end{equation*}
Therefore, using Markov's inequality, we have
\begin{align*}
    &\pr\bigg[\bigg|\frac{1}{m}\sum_{i=1}^m \bigg(\ell(\w^1;\x\ith,y\ith) - \ell(\w^2;\x\ith,y\ith)\bigg) - \Exy[\Delta(\w^1,\w^2)]\bigg|\geq \frac{1}{3}\Exy[\Delta(\w^1,\w^2)]\bigg]\\
    &\quad \leq \frac{\Exy[\Delta^2(\w^1,\w^2)]}{m ((1/3)\Exy[\Delta(\w^1,\w^2)])^2} \leq \frac{cB^2}{m \Ex[(\sigma(\w^1\cdot\x) - \sigma(\w^2\cdot\x))^2]}\leq  \frac{cB^2}{m \calL(\w^2)}\; .
\end{align*}
Let $\w^2 = \w\tstrth$. With $m \geq (cB^2/(\calL(\w\tstrth)\delta))$, the inequality above implies that when $\w^1 = \w\tth$, $t\in[T]$, such that $\calL(\w\tth) > \calL(\w\tstrth)$ and $\calL(\w\tstrth)\leq (1/8)\Ex[(\sigma(\w\tth\cdot\x) - \sigma(\w\tstrth\cdot\x))^2]$, it holds with probability at least $1 - \delta$:
\begin{equation*}
    |\wh{\calL}(\w\tth) - \wh{\calL}(\w\tstrth) - \Exy[\Delta(\w\tth, \w\tstrth)]|\leq \frac{1}{3}\Exy[\Delta(\w\tth, \w\tstrth)],
\end{equation*}
in other words, we have
\begin{equation*}
    \wh{\calL}(\w\tstrth)\leq \wh{\calL}(\w\tth) - \frac{2}{3}\Exy[\Delta(\w\tth, \w\tstrth)] \;.\end{equation*}
Let $\delta = 1/(3T)$. Applying a union bound to all $\w\tth$, $t\in[T]$, we have that with probability at least  $2/3$, using $m \geq (cB^2 T/\calL(\w\tstrth))$ samples suffices to distinguish $\w\tstrth$ from other vectors $\w\tth$ that have large $L_2^2$ error.

But what if $\calL(\w\tstrth)\geq (1/8)\Ex[(\sigma(\w\tth\cdot\x) - \sigma(\w\tstrth\cdot\x))^2]$? In this case, note that
\begin{align*}
    \calL(\w\tstrth)&\geq \frac{1}{8}\Ex[(\sigma(\w\tth\cdot\x) - \sigma(\w\tstrth\cdot\x))^2]\\
    &=\frac{1}{8}\bigg(\Exy[(\sigma(\w\tth\cdot\x) - y)^2] + \Exy[(\sigma(\w\tstrth\cdot\x)-y)^2] \\
    &\quad -2\Exy[(\sigma(\w\tth\cdot\x) - y)(\sigma(\w\tstrth\cdot\x)-y)]\bigg)\\
    &\geq \frac{1}{8}\bigg(\calL(\w\tth) + \calL(\w\tstrth) - \frac{1}{2}\calL(\w\tth) - 2\calL(\w\tstrth)\bigg),
\end{align*}
where in the last inequality we used Young's inequality $ab\leq (1/4)a^2 + b^2$. The inequality above indicates that $\calL(\w\tth)\leq 18\calL(\w\tstrth)$. 
{Note that according to \Cref{app:thm:main-general}, $\w^{(t^*)}$ is guaranteed to reside in the $\mathcal{R}_{\sigma,\theta_0}(C\opt + \eps)$ region, therefore, by definition of the region $\mathcal{R}_{\sigma,\theta_0}(C\opt + \eps)$, $$(\theta_{t^*})^2\|\Tre_{\cos(\theta_{t^*})}\sigma'\|_{L_2}^2\leq C\opt + \eps.$$
When $\|\Pm{>1/(\theta_{t^*})^2}\sigma\|_{L_2}^2\leq C(\theta_{t^*})^2\|\Tre_{\cos(\theta_{t^*})}\sigma'\|_{L_2}^2$, according to the error bound displayed in \Cref{app:prop:error}, we have $\calL(\w\tstrth)\leq O(\opt) + 4C(\theta_{t^*})^2\|\Tre_{\cos(\theta_{t^*})}\sigma'\|_{L_2}^2\leq O(\opt) + 4C\eps$. 
Therefore, any vector $\w\tth$ that satisfies $\calL(\w\tstrth)\geq (1/8)\Ex[(\sigma(\w\tth\cdot\x) - \sigma(\w\tstrth\cdot\x))^2]$ is in the $\mathcal{R}_{\sigma,\theta_0}(C_1\opt + \eps)$ region and can be output as a constant factor solution.

In summary, when $\|\Pm{>1/(\theta_{t^*})^2}\sigma\|_{L_2}^2\lesssim (\theta_{t^*})^2\|\Tre_{\cos(\theta_{t^*})}\sigma'\|_{L_2}^2$, using
\begin{equation*}
    m \leq \Theta\bigg(\frac{B^2\log(L/\eps)}{\eps}\bigg)
\end{equation*}
samples suffices for the testing algorithm.
}

Now consider the general case where $\|\Pm{>1/(\theta_{t^*})^2}\sigma\|_{L_2}^2\gtrsim (\theta_{t^*})^2\|\Tre_{\cos(\theta_{t^*})}\sigma'\|_{L_2}^2$. We still have $\calL(\w\tstrth)\leq O(\opt) + 4\|\Pm{>1/(\theta_{t^*})^2}\sigma\|_{L_2}^2$ but it is no longer acceptable to output a vector $\wh{\w}$ such that $\calL(\wh{\w}) = C\calL(\w\tstrth)$ for constant $C>1$. This is because in the worst case the $L_2^2$ error of $\wh{\w}$ can be as large as $\calL(\wh{\w}) = O(\opt) + 4C\|\Pm{>1/(\theta_{t^*})^2}\sigma\|_{L_2}^2$. When $\|\Pm{>1/(\theta_{t^*})^2}\sigma\|_{L_2}^2$ is large, this error bound does not imply that $\wh{\w}$ lies in the target region $\mathcal{R}_{\sigma,\theta_0}(O(\opt))$. Therefore, we need a different analysis.

To find a vector $\wh{\w}$ from the set $\{\w\tth\}_{t\in[T]}$ such that $\wh{\w}\in\mathcal{R}_{\sigma,\theta_0}(O(\opt) + \eps)$, we need to approximate $\calL(\w\tth)$ to error at most $O(\opt) + \eps$ for each $\w\tth$, $t\in[T]$. Since $|\ell(\w;\x,y)|\leq 4B^2$, we know that $\ell(\w;\x,y)$ is a sub-Gaussian random variable with $\|\ell(\w;\x,y)\|_{\psi_2}\leq cB^2$ for some absolute constant $c$. Then  using Hoeffding's inequality, we have
\begin{align*}
    \pr\bigg[\bigg|\sum_{i=1}^{m}\frac{1}{m} \ell(\w;\x\ith,y\ith) - \Exy[\ell(\w;\x,y)]\bigg|\geq (\opt + \eps)\bigg]&\leq\exp\bigg(-\frac{c'(\opt+\eps)^2}{m B^4}\bigg).
\end{align*}
Therefore, using $m \leq B^4\log(1/\delta)/\eps^2$ samples suffices to approximate $\calL(\w)$ to error $\opt + \eps$. Using a union bound on all $\w\tth$, $t\in[T]$, and set $\delta = 1/(3T)$, we obtain that with probability at least $2/3$, 
using
\begin{equation*}
    m \leq \Theta\bigg(\frac{B^4\log\log(L/\eps)}{\eps^2}\bigg)
\end{equation*}
samples,
it holds
\begin{equation*}
    |\wh{\calL}(\w\tth) - \calL(\w\tth)|\leq \opt + \eps
\end{equation*}
for all $\w\tth$, $t\in[T]$.
Therefore, since $\calL(\w\tstrth)\leq O(\opt) + \eps + 4\|\Pm{>1/(\theta_{t^*})^2}\sigma\|_{L_2}^2$, by outputting $\wh{\w} = \min\{\w\tth,t\in[T]: \wh{\calL}(\w\tth)\}$ we guarantee that $\calL(\wh{\w})\leq O(\opt) + \eps + 4\|\Pm{>1/(\theta_{t^*})^2}\sigma\|_{L_2}^2$, hence $\wh{\w}\in\mathcal{R}_{\sigma,\theta_0}(O(\opt) + \eps)$.
\end{proof}

\subsection{Proof of \Cref{app:lem:sample}}\label{app:sec:proof-of-concentration-lemma}

In this subsection, we provide technical lemmas that determine the number of samples required for each iteration. We start with \Cref{app:lem:upbd-||g||} that bounds the population gradient $\|\g(\w)\|_{2}$. 
Then in \Cref{app:lem:sample} we provide the sufficient batch size of samples per iteration, utilizing the bounds on $\|\g(\w)\|_{2}$ and the truncated upper bounds on the activation $\sigma(z)$ and labels $y$.

\begin{lemma}\label{app:lem:upbd-||g||}
    Let $\g(\w)\eqdef \Exy[y\Tre_{\rho}\sigma'(\w\cdot\x)\x^{\perp_\w}]$ and $\theta = \theta(\w, \wstar)$. Then, we have
    \begin{equation*}
        \|\g(\w)\|_{2}\leq\sqrt{\opt}\|\Tr\sigma'\|_{L_2} + \|\Tre_{\sqrt{\rho\cos\theta}}\sigma'\|_{L_2}^2\sin\theta.\end{equation*}
    If, in addition, $0<\rho\leq \cos\theta<1$ and $\sin\theta\geq 4\sqrt{\opt}/\|\Tr\sigma'\|_{L_2}$, then
    \begin{equation*}
        \|\g(\w)\|_{2}\leq (5/4)\|\Tre_{\sqrt{\rho\cos\theta}}\sigma'\|_{L_2}^2\sin\theta \end{equation*}
\end{lemma}
\begin{proof}
    By the variational definition of vector norms, we have
    \begin{align}\label{app:eq:gk-norm-1-gen}
        &\quad \|\g(\w)\|_{2} = \max_{\|\bu\|_{2} = 1} \Exy[y\Tre_{\rho}\sigma'(\w\cdot\x)\x^{\perp_\w}\cdot\bu] \nonumber\\
        &=\max_{\|\bu\|_{2} = 1} \Exy[(y - \sigma(\w^*\cdot\x))\Tre_{\rho}\sigma'(\w\cdot\x)\x^{\perp_\w}\cdot\bu] + \Ex[\sigma(\w^*\cdot\x)\Tre_{\rho}\sigma'(\w\cdot\x)\x^{\perp_\w}\cdot\bu].
    \end{align}
Observe here that the maximizing $\bu$ depends on the expectation defining $\g(\w)$ and is thus deterministic. 
    To bound the right-hand side in \Cref{app:eq:gk-norm-1-gen}, we fix an arbitrary unit vector $\bu$ and bound the two summands. 
    Using the Cauchy-Schwarz inequality, the first term in \Cref{app:eq:gk-norm-1-gen} above can be bounded by:
    \begin{align*}
        &\quad \Exy[(y - \sigma(\w^*\cdot\x))\Tre_{\rho}\sigma'(\w\cdot\x)\x^{\perp_\w}\cdot\bu]\\
        &\leq \sqrt{\Exy[(y - \sigma(\w^*\cdot\x))^2]\Ex[(\Tre_{\rho}\sigma'(\w\cdot\x))^2(\x^{\perp_\w}\cdot\bu)^2]}\\
        &\leq \sqrt{\opt}\sqrt{\Ex[(\Tre_{\rho}\sigma'(\w\cdot\x))^2]\Ex[(\x^{\perp_\w}\cdot\bu)^2]}\\
        &=\sqrt{\opt}\sqrt{\Ex[(\Tre_{\rho}\sigma'(\w\cdot\x))^2]},
    \end{align*}
    where in the second inequality we used the fact that $\w\cdot\x$ is independent of $ \x^{\perp_\w}\cdot\bu$, due to the Gaussianity. The last equality uses $\x^{\perp_\w}\cdot\bu \sim \normal(0,1),$ as $\bu$ is independent of $\x$ and $\x\sim \normal(\vec 0, \vec I)$.

    For the second term in \Cref{app:eq:gk-norm-1-gen}, observe that if $\bu\perp\w,\w^*$, then the expectation takes value zero due to the independence between Gaussian random variables $\x^{\perp_\w}\cdot\bu$ and $\w\cdot\x,\w^*\cdot\x$. Therefore, we only need to consider $\bu$ in the span of $\w, \wstar,$ which can be expressed as $\bu = \cos\alpha\w + \sin\alpha(\w^*)^{\perp_\w}/\|(\w^*)^{\perp_\w}\|_{2}$, for some $\alpha\in[0,2\pi]$. Thus, plugging this $\bu$ back into the second term in \Cref{app:eq:gk-norm-1-gen}, and setting $z_1 = \w\cdot\x$, $z_2 = (\w^*)^{\perp_\w}/\|(\w^*)^{\perp_\w}\|_{2}\cdot\x$ which are independent Gaussian random variables, we get
    \begin{align*}
        \Ex[\sigma(\w^*\cdot\x)\Tre_{\rho}\sigma'(\w\cdot\x)\x^{\perp_\w}\cdot\bu]
        &= \E_{z_1,z_2\sim\calN}[\sigma(\cos(\theta) z_1 + \sin(\theta) z_2)\Tre_{\rho}\sigma'(z_1) \sin(\alpha) z_2]\\
        &=\E_{z_1}\big[\,\E_{z_2}[\sigma(\cos(\theta) z_1 + \sin(\theta) z_2)z_2\,| \, z_1]\Tre_{\rho}\sigma'(z_1) \sin(\alpha)\big]\\
        &= \E_{z_1,z_2\sim\calN}[\sigma'(\cos(\theta) z_1 + \sin(\theta) z_2)\Tre_{\rho}\sigma'(z_1) \sin(\alpha) \sin(\theta)],
    \end{align*}
    {where in the last inequality we applied \Cref{fct:stein}. Moreover, recalling the definition of the \OU semigroup, we further have}
    \begin{align*}
    \Ex[\sigma(\w^*\cdot\x)\Tre_{\rho}\sigma'(\w\cdot\x)\x^{\perp_\w}\cdot\bu]
        &=\E_{z_1}[\, \E_{z_2}[\sigma'(\cos\theta z_1 + \sin\theta z_2)\, | \, z_1]\Tre_{\rho}\sigma'(z_1) \sin\alpha \sin\theta]\\
        &\leq \E_{z_1\sim\calN}[\Tre_{\cos\theta}\sigma'(z_1)\Tre_{\rho}\sigma'(z_1)]\sin\theta,
    \end{align*}
    where the last inequality holds since $\E_{z_1\sim\calN}[\Tre_{\cos\theta}\sigma'(z_1)\Tre_{\rho}\sigma'(z_1)]= \|\Tre_{\sqrt{\rho\cos\theta}}\sigma'\|_{L_2}^2\geq 0$. 
    Plugging in these bounds on the first and second terms of \Cref{app:eq:gk-norm-1-gen}, we get
    \begin{equation*}
        \|\g(\w)\|_{2}\leq \sqrt{\opt}\sqrt{\Ex[(\Tre_{\rho}\sigma'(\w\cdot\x))^2]} + \Ex[\Tre_{\cos\theta}\sigma'(\w\cdot\x)\Tr\sigma'(\w\cdot\x)]\sin\theta.
    \end{equation*}
    As we have argued in the proof of \Cref{app:prop:structural}, if $\rho\leq \cos\theta$, then $\Ex[\Tre_{\cos\theta}\sigma'(\w\cdot\x)\Tr\sigma'(\w\cdot\x)]\geq \Ex[(\Tre_{\rho}\sigma'(\w\cdot\x))^2]$, hence we further get that if in addition it holds
    \begin{equation*}
        \sin\theta\geq 4\sqrt{\opt}/\sqrt{\Ex[(\Tre_{\rho}\sigma'(\w\cdot\x))^2]},
    \end{equation*}
    then we obtain
    \begin{align*}
        \|\g(\w)\|_{2}&\leq (\Ex[\Tre_{\cos\theta}\sigma'(\w\cdot\x)\Tr\sigma'(\w\cdot\x)]+ (1/4)\Ex[(\Tre_{\rho}\sigma'(\w\cdot\x))^2])\sin\theta\\
        &\leq (5/4)\Ex[\Tre_{\cos\theta}\sigma'(\w\cdot\x)\Tr\sigma'(\w\cdot\x)]\sin\theta,
    \end{align*}
    completing the proof. 
\end{proof}

We now proceed to determine the sample complexities required to estimate the gradient. \Cref{app:lem:sample} provides the sample complexity to approximate the norm of the population gradient $\|\g(\w)\|_2$ and the inner product between the population gradient and $\w^*$.
{We restate and prove \Cref{app:lem:sample}:}
\SampleComplexityAlg*

\begin{proof}
Observe first that, by Chebyshev inequality,
\begin{align}\label{app:eq:norm-concentration-eq0}
    \pr[\|\wh{\g}(\w) - \g(\w)\|_{2}\geq t] &\leq \frac{\Exy[\|\g(\w;\x,y) - \g(\w)\|_{2}^2]}{nt^2}.
\end{align}
Now we proceed to bound the variance $\Exy[\|\g(\w;\x,y) - \g(\w)\|_{2}^2]$. Let $\e_1,\dots,\e_d$ be the standard basis of $\R^d$; we have
\begin{align*}
    \Exy[\|\g(\w;\x,y) - \g(\w)\|_{2}^2]&\leq\Exy[\|\g(\w;\x,y)\|_{2}^2] =\Exy\bigg[\sum_{j=1}^d(\g(\w;\x,y)\cdot\e_j)^2\bigg]\\
    &=\sum_{j=1}^d\Exy[(y\Tr\sigma'(\w\cdot\x)\x\cdot(\e_j)^{\perp_\w})^2]\\
    &\leq d B^2 \|\Tr\sigma'\|_{L_2}^2 \Ex[(\x\cdot(\e_1)^{\perp_\w})^2]\leq d B^2 \|\Tr\sigma'\|_{L_2}^2.
\end{align*}
In the second inequality above, we used $|y|\leq B$, which is w.l.o.g., as shown in \Cref{app:claim:bound-y}.
Therefore, plugging the upper bound on the variance back into \Cref{app:eq:norm-concentration-eq0}, we get
\begin{align*}
    \pr[\|\wh{\g}(\w) - \g(\w)\|_{2}\geq t]\leq \frac{d B^2 \|\Tr\sigma'\|_{L_2}^2}{nt^2}.
\end{align*}

Now choosing $t = \frac{1}{6}\|\Tre_{\sqrt{\rho\cos\theta}}\sigma'\|_{L_2}^2\sin\theta$ and setting
\begin{equation*}
    n = \Theta\bigg(\frac{dB^2\|\Tr\sigma'\|_{L_2}^2}{\sin^2\theta\|\Tre_{\sqrt{\rho\cos\theta}}\sigma'\|_{L_2}^4 \delta}\bigg),
\end{equation*}
we obtain that with probability at least $1- \delta$, it holds
\begin{equation}\label{app:eq:difference-empirical-and-population-grad}
    \|\wh{\g}(\w) - \g(\w)\|_2\leq \frac{1}{6}\|\Tre_{\sqrt{\rho\cos\theta}}\sigma'\|_{L_2}^2\sin\theta\;,
\end{equation}
and hence
\begin{equation*}
    \|\wh{\g}(\w)\|_{2}\leq \|\g(\w)\|_{2} + (1/6)\|\Tre_{\sqrt{\rho\cos\theta}}\sigma'\|_{L_2}^2\sin\theta.
\end{equation*}

Applying the upper bound on $\|\g(\w)\|_{2}$ we have provided in \Cref{app:lem:upbd-||g||}, we obtain
\begin{equation*}
    \|\wh{\g}(\w)\|_{2}\leq \sqrt{\opt}\|\Tr\sigma'\|_{L_2} + (7/6)\|\Tre_{\sqrt{\rho\cos\theta}}\sigma'\|_{L_2}^2\sin\theta.
\end{equation*}
In particular, if $0<\rho\leq \cos\theta<1$ and $\sin\theta\geq 4\sqrt{\opt}/\|\Tr\sigma'\|_{L_2}$, then we have
\begin{equation*}
    \|\wh{\g}(\w)\|_{2}\leq (3/2)\|\Tre_{\sqrt{\rho\cos\theta}}\sigma'\|_{L_2}^2\sin\theta.
\end{equation*}
Since $0<\rho\leq \cos\theta<1$, it must be $\|\Tr\sigma'\|_{L_2}\leq \|\Tre_{\sqrt{\rho\cos\theta}}\sigma'\|_{L_2}$, and thus using 
\begin{equation}\label{app:eq:batchsize-1}
    n=\Theta\bigg(\frac{dB^2}{\sin^2\theta\|\Tr\sigma'\|_{L_2}^2\delta}\bigg)
\end{equation}
samples suffices to guarantee that $\|\wh{\g}(\w)\|_{2}\leq (3/2)\|\Tre_{\sqrt{\rho\cos\theta}}\sigma'\|_{L_2}^2\sin\theta$.

For the inner product between $\wh{\g}(\w)$ and $\w^*$, let us denote $\bv = (\w^*)^{\perp_\w}/\|(\w^*)^{\perp_\w}\|_{2}$, and $\w^* = \sin\theta\bv + \cos\theta\w$. Then, since $\wh{\g}(\w)$ is orthogonal to $\w$, we have $\wh{\g}(\w)\cdot\w^* =  \wh{\g}(\w)\cdot\bv \sin\theta$. Therefore, using \Cref{app:eq:difference-empirical-and-population-grad}, we obtain that when the batch size $n$ satisfies \Cref{app:eq:batchsize-1}, with probability at least $1 - \delta$, we have
\begin{align}\label{app:eq:batch-size-2}
    \wh{\g}(\w)\cdot\w^* &=  (\wh{\g}(\w) - \g(\w))\cdot\bv \sin\theta + \g(\w)\cdot\bv\sin\theta \nonumber\\
    &\leq \|\wh{\g}(\w) - \g(\w)\|_2\sin\theta + \g(\w)\cdot\w^* \nonumber\\
    &\leq (1/6)\|\Tre_{\sqrt{\rho\cos\theta}}\sigma'\|_{L_2}^2\sin^2\theta + \g(\w)\cdot\w^*.
\end{align}
Now applying \Cref{app:prop:structural} we get
\begin{equation*}
    \wh{\g}(\w)\cdot\w^*\leq -\frac{5}{6}\|\Tre_{\sqrt{\rho\cos\theta}}\sigma'\|_{L_2}^2\sin^2\theta + \sqrt{\opt}\|\Tr\sigma'\|_{L_2}\sin\theta.
\end{equation*}
In particular, when $\sin\theta\geq 3\sqrt{\opt}/\|\Tr\sigma'\|_{L_2}$, in \Cref{app:prop:structural} we showed that 
$$\g(\w)\cdot\w^*\leq -(2/3)\|\Tre_{\sqrt{\rho\cos\theta}}\sigma'\|_{L_2}^2\sin^2\theta.$$ 
Thus, when $\sin\theta\geq 3\sqrt{\opt}/\|\Tr\sigma'\|_{L_2}$, using \Cref{app:eq:batch-size-2} we have that with probability at least $1 - \delta$, 
\begin{equation*}
    \wh{\g}(\w)\cdot\w^*\leq -\frac{1}{2}\|\Tre_{\sqrt{\rho\cos\theta}}\sigma'\|_{L_2}^2\sin^2\theta,
\end{equation*}
completing the proof. 
\end{proof}

\section{Full Version of \Cref{sec:learn-monotone}}\label{app:sec:learn-monotone}

We have shown in \Cref{app:sec:learn-sim-general} that \Cref{app:alg:GD-general-activation} converges to a parameter vector $\vec w$ with an $L_2^2$ error  bounded above by $O(\opt)+\|\Pm{>1/(\theta^*)^2}\sigma\|_{L_2}^2$, where $\theta^*$ is a Critical Point.
 One of the technical difficulties is that in general we cannot bound $\|\Pm{>1/(\theta^*)^2}\sigma\|_{L_2}^2$ by $\opt$. One such example is when $\sigma(t)=\he_{(1/(\theta^*)^2+1)}(t)$; in this case $\|\Pm{>1/(\theta^*)^2}\sigma\|_{L_2}^2=\|\sigma\|_{L_2}^2$, which can be   $\omega(\opt)$. 
In this section, we show that if the activation is also monotone, then given that $\theta^*$ is sufficiently small, we can bound $\|\Pm{>1/(\theta^*)^2}\sigma\|_{L_2}^2$ by the \OU semigroup of $\sigma'$. Specifically, we provide an initialization method that along with \Cref{app:alg:GD-general-activation} gives an algorithm that guarantees error $O(\opt)$. Formally, we show the following. 
\begin{theorem}[Learning Monotone $(B,L)$-Regular Activations]\label{app:thm:main-monotone}
    Let $\eps>0$, and let $\sigma$ be a monotone $(B,L)$-Regular activation. 
    Then, \Cref{app:alg:GD-general-activation} 
    draws $N = \tilde{\Theta}({dB^2\log(L/\eps)}/{\eps} + d/\eps^2)$ samples, 
    runs in $\poly(d, N)$ time 
    and returns a vector $\wh{\w}$ such that with probability at least $2/3$,
    $\wh{\w}\in\mathcal{R}_{\sigma,\theta_0}(O(\opt) + \eps)$, and 
    it holds that $\Exy[(\sigma(\wh{\w}\cdot\x) - y)^2] \leq C\opt + \eps$, where $C$ is an absolute constant independent of $\eps, d, B, L$.
\end{theorem}

The main result of this section is an initialization routine that allows us to bound the higher coefficients of the spectrum, $\|\Pm{>1/(\theta^*)^2}\sigma\|_{L_2}^2$. In particular, we prove the following.

\begin{proposition}[Initialization]\label{app:prop:initialization}
     Let $\sigma:\R\to\R$, $\sigma\in L_2(\mathcal N)$, be a monotone $(B,L)$-Regular activation. Let $\D$ be a distribution of labeled examples
$(\x,y) \in \R^d \times \R$ such that 
$\D_\x=\mathcal N(\vec 0,\vec I)$. Fix a unit vector $\wstar\in \R^d$ such that $\Exy[(\sigma(\wstar\cdot\x)-y)^2] = \opt$. There exists an algorithm that draws $N=\widetilde O(d/\eps^2)$ samples, runs in $\poly(N,d)$ time, and with probability at least $2/3$, returns a unit vector $\w^{(0)}\in \R^d$ such that for any unit $\w'\in \R^d$ with $\theta=\theta(\w',\wstar)\leq \theta(\w^{(0)},\wstar)$, it holds that
\[
\|\Pm{>1/\theta^2}\sigma\|_{L_2}^2\lesssim \sin^2\theta\|\Tre_{\cos\theta}\sigma'\|_{L_2}^2\;.
\]
\end{proposition}

Combining \Cref{app:thm:main-general} with \Cref{app:prop:initialization}, we can the prove  \Cref{app:thm:main-monotone}.

\begin{proof}[Proof of \Cref{app:thm:main-monotone}]
\Cref{app:thm:main-general} implies that \Cref{app:alg:GD-general-activation} generates a vector $\wh{\w}\in\mathcal{R}_{\sigma,\theta_0}(C\opt+\eps)$ where $C$ is an absolute constant. This implies that $\theta^2\|\Tre_{\cos(\theta)}\sigma'\|_{L_2}^2\leq C\opt + \eps$. Since $\theta(\wh{\w},\w^*)\leq \theta_0$, combining with  \Cref{app:prop:initialization}, i.e., $\|\Pm{>1/\theta^2}\sigma\|_{L_2}^2\lesssim \sin^2\theta\|\Tre_{\cos\theta}\sigma'\|_{L_2}^2$, we further have $\|\Pm{>1/\theta^2}\sigma\|_{L_2}^2\lesssim \opt + \eps$. Finally, using the error bound on $\calL(\wh{\w})$ developed in \Cref{app:prop:error}, we get $\calL(\wh{\w})\leq C\opt + \eps$. 
    
    As displayed in \Cref{app:thm:main-general}, the main algorithm uses $N_1 = \tilde{\Theta}(dB^2/\eps + B^2/\eps)$ samples (since according to \Cref{app:lem:test}, when $\|\Pm{>1/\theta^2}\sigma\|_{L_2}^2\lesssim \sin^2\theta\|\Tre_{\cos\theta}\sigma'\|_{L_2}^2$, using $m= \tilde{\Theta}(B^2/\eps)$ samples suffices for testing \Cref{app:alg:test}), and in \Cref{app:prop:initialization} we showed that the initialization procedure requires $\tilde{\Theta}(d/\eps^2)$ samples. Thus, in summary, for monotone $(B,L)$-Regular activations, \Cref{app:alg:GD-general-activation} uses $N = \tilde{\Theta}({dB^2}/{\eps} + d/\eps^2)$ samples and runs in $\poly(d,N)$ times.
\end{proof}

    For monotone $b$-Lipschitz activations $\sigma$, we know from \Cref{app:lem:activation-truncation-lip-and-mon-jump}  that $\sigma$ is an $\eps$-Extended $(b\log^{1/2}(b/\eps), b)$-Regular activation, meaning that there exists a truncated activation $\bar{\sigma}$ that such that $\Ez[(\bar{\sigma}(z) - \sigma(z))^2]\leq \eps$ and $\bar{\sigma}$ is $(b\log^{1/2}(b/\eps), b)$-Regular. Hence applying \Cref{app:thm:main-monotone} to $\bar{\sigma}$, we obtain the following corollary:
    
    \begin{corollary}[Learning Monotone \& Lipschitz Activations]\label{app:cor:main-monotone-lip}
        Let $\eps,b>0$, and let $\sigma$ be a monotone $b$-Lipschitz activation. 
    Then, \Cref{app:alg:GD-general-activation} 
    draws $N = \tilde{\Theta}({db^2}/{\eps} + d/\eps^2)$ samples, 
    runs in $\poly(d, N)$ time, 
    and returns a vector $\wh{\w}$ such that with probability at least $2/3$, 
    it holds that $\calL(\wh{\w}) \leq C\opt + \eps$, where $C$ is an absolute constant independent of $\eps, d, b$.
    \end{corollary}

Similarly, if $\sigma$ has bounded $2+\zeta$ moment $\Ez[\sigma^{2+\zeta}(z)] \leq B_{\sigma}$, then according to \Cref{app:lem:activation-truncation-2+zeta} we know that $\sigma$ is an $\eps$-Extended $((B_{\sigma}/\eps)^{1/\zeta},(B_{\sigma}/\eps)^{4/\zeta}/\eps^2)$-Regular activation. Therefore, replacing $B$ with $(B_{\sigma}/\eps)^{1/\zeta}$ and replace $L$ with $(B_{\sigma}/\eps)^{4/\zeta}/\eps^2$ in \Cref{app:thm:main-monotone}, we obtain:
    \begin{corollary}[Learning Monotone Activations With Bounded $2+\zeta$  Moments]\label{app:cor:main-monotone-4mom}
        Let $\eps>0$, and let $\sigma$ be a monotone activation that satisfies $\Ez[\sigma^{2+\zeta}(z)] \leq B_{\sigma}$. 
    Then, \Cref{app:alg:GD-general-activation} 
    draws $N = \tilde{\Theta}(d(B_{\sigma}/\eps)^{2/\zeta}\log(B_{\sigma}/\eps)/{\eps} + d/\eps^2)$ samples, 
    runs in $\poly(d, N)$ time, 
    and returns a vector $\wh{\w}$ such that with probability at least $2/3$, 
    it holds that $\calL(\wh{\w}) \leq C\opt + \eps$, where $C$ is an absolute constant independent of $\eps, d, B_{\sigma}, L$.
    \end{corollary}

The main contents of this section are the following:
To prove \Cref{app:prop:initialization}, we need to combine two main technical pieces: (1) proving that there exists a threshold $\theta_0$ such that for any $\theta\leq \theta_0$,  $\|\Pm{>1/\theta^2}\sigma\|_{L_2}^2\lesssim \sin^2\theta\|\Tre_{\cos\theta}\sigma'\|_{L_2}^2$; (2) proving that there exists an efficient algorithm that finds a parameter $\w^{(0)}$ such that $\theta(\w^{(0)},\w^*)\leq \theta_0$.

\Cref{app:subsec:bound-high-order-coefficients} is devoted to the proof of (1), i.e., that there exists $\theta_0$ such that for $\theta \leq \theta_0,$ $\|\Pm{>1/\theta^2}\sigma\|_{L_2}^2\lesssim \sin^2\theta\|\Tre_{\cos\theta}\sigma'\|_{L_2}^2$ (\Cref{app:prop:error-bound-smoothing-tails}). Unfortunately, it was technically hard to prove this claim directly for all monotone functions due to the versatility of such functions. Hence, the natural idea is that if we can prove (1) for a sequence of simple and `nice' functions $\Phi_k$ that can converge to $\sigma$, then by the convergence theorems the desired claim will also hold true for $\sigma$. In particular, let $\Phi_k$ be a sequence of functions; then, one can show that the higher order coefficients can be bounded by
\begin{equation*}
    \|\Pm{>1/\theta^2}\sigma\|_{L_2}^2\leq 2\|\sigma-\Phi_k\|_{L_2}^2+4\|\Phi_k-\Tr \Phi_k\|_{L_2}^2+4\theta^2\|\Tr \Phi_k'\|_{L_2}^2.
\end{equation*}
If $\Phi_k$ converges to $\sigma$ pointwise, one can show that the first term above goes to $0$ and the second term converges to $\theta^2\|\Tre_{\cos\theta}\sigma'\|_{L_2}^2$, which is the bound we are looking for. Thus, it remains to show that $\|\Phi_k-\Tr \Phi_k\|_{L_2}^2\lesssim \theta^2\|\Tr \Phi_k'\|_{L_2}^2$.

\Cref{app:subsection:bounding-aug-error} proves the claim that $\|\Phi-\Tr \Phi\|_{L_2}^2\lesssim \theta^2\|\Tr \Phi'\|_{L_2}^2$ (\Cref{app:prop:smoothing-error-bound}), for any $\Phi(z)$ that is a {\itshape monotonic staircase function}:
\begin{definition}[Monotonic Staircase Functions]\label{app:def:staircase}
For simplicity, denote the indicator function $\1\{z\geq t\}$ by $\phi(z;t)$. Let $m$ be a positive integer and let $M > 0$. The monotonic 
 staircase functions (of $M$-bounded support) are defined by
    \begin{equation*}
    \calF\eqdef \bigg\{\sum_{i=1}^m A_i\phi(z;t_i) + A_0: A_0\in\R; A_i > 0, |t_i|\leq M, \forall i\in[m]; m<\infty\bigg\}\;.
\end{equation*}
\end{definition}
These staircase functions constitute a dense subset of the monotone function class and have a simple and easy-to-analyze form, therefore they serve well for our purpose. However, though the staircase function $\Phi$ already takes a concise and simple expression, many technical difficulties arise when analyzing $\Tr\Phi(z) - \Phi(z)$, mainly due to the complicated form of $\Tr\Phi(z)$. Our workaround is to introduce a new type of smoothing/augmentation method, which we call centered augmentation, defined by $\Tr(\Phi(z/\rho))$. This recentered augmentation takes a much simpler form compared to $\Tr\Phi(z)$. In particular, we show that when the smoothing parameter $\rho$ is not too small, namely, when $1 - \rho^2\leq O(1/\log(1/\eps))$, then: (i) the $L_2^2$ distance between $\Tr\Phi(z/\rho)$ and $\Phi(z)$ can be bounded above by $(1 - \rho^2)\|\Tr\Phi'(z/\rho)\|_{L_2}^2$ (\Cref{app:lem:upper-bound-centered-smoothed-error}); (ii) the $L_2^2$ distance between $\Tr\Phi(z/\rho)$
 and $\Tr\Phi(z)$ can be bounded above by $(1 - \rho^2)(\|\Tr\Phi'(z/\rho)\|_{L_2}^2 + \|\Tr\Phi'(z)\|_{L_2}^2)$ (\Cref{app:lem:bound-the-differece-TrPhi-TrPhi(z/rho)}); (iii) finally, choosing the smoothing strength $\rho_1$ slightly larger than $\rho$, we have $\|\Tre_{\rho_1}\Phi'(z/\rho_1)\|_{L_2}^2\lesssim \|\Tr\Phi'(z)\|_{L_2}^2$ (\Cref{app:lem:upper-bound-centered-smoothed-derivative-squared}). Combining these 3 results on the relations between $\Tr\Phi(z/\rho)$ and $\Tr\Phi(z)$, we prove \Cref{app:prop:smoothing-error-bound} 
 in \Cref{app:subsec:proof-of-smoothing-error-bound}, completing the last piece of the puzzle in the proof of \Cref{app:prop:error-bound-smoothing-tails}.

Finally, in \Cref{app:subsec:initialization-for-monotone}, we prove (2) by providing an {\itshape SQ} initialization algorithm. The main idea is to transform the labels $y$ to $\tilde{y} = \mathcal{T}(y)\eqdef \1\{y\geq t'\}$ for a carefully chosen threshold $t'$. Then, we show that there exists a halfspace $\phi(\w^*\cdot\x; t) = \1\{\w^*\cdot\x\geq t\}$ such that the transformed labels $\tilde{y}$ can be viewed as the corrupted labels of $\phi(\w^*\cdot\x;t)$. Then, utilizing the algorithm for learning halfspaces~\cite{DKTZ22b}, we can obtain an initial vector $\w^{(0)}$ such that $\theta(\w^{(0)},\w^*)\leq \theta_0$. 

Finally, combining (1) and (2), we prove \Cref{app:prop:initialization}.

\subsection{Bounding Higher Order Hermite Coefficients of Monotone Activations}\label{app:subsec:bound-high-order-coefficients}
The main result of this section is the following:\begin{proposition}[From Hermite Tails to \OU Semigroup]\label{app:prop:error-bound-smoothing-tails}
    Let $\sigma:\R\to\R$ be a monotone activation and $\sigma\in L_2(\mathcal N)$.
    Let $M$ be the upper bound for the support of $\sigma'(z)$,\footnote{In \Cref{app:claim:bounded-support-sigma'}, we show that for any $\sigma\in\mathcal{H}(B,L)$, the support of $\sigma'$ can always be bounded by $M\lesssim \sqrt{\log(B/\eps) - \log\log(B/\eps)}$.} i.e., $\forall z\in\R$ such that $|z|\geq M$, we have $\sigma'(z)= 0$. For any $\theta\in[0,\pi]$ such that $1 - C/M^2<\cos^2\theta$ with $C>0$ an absolute constant, it holds that $\|\Pm{>1/\theta^2}\sigma\|_{L_2}^2\lesssim \sin^2\theta\|\Tre_{\cos\theta}\sigma'\|_{L_2}^2$. \end{proposition}
\begin{proof}
Instead of proving \Cref{app:prop:error-bound-smoothing-tails} directly  for the activation $\sigma$, we chose another function $\Phi$ that works as a surrogate for $\sigma$ and satisfies certain regularity properties.   Let $\Phi$ be any function in $L_2(\normal)$, then by Young's inequality we have that
    \begin{align*}
        \|\Pm{>1/\theta^2}\sigma\|_{L_2}^2&\leq 2\|\Pm{>1/\theta^2}(\sigma-\Phi)\|_{L_2}^2+2\|\Pm{>1/\theta^2}\Phi\|_{L_2}^2
        \\&\leq 2\|\Pm{>1/\theta^2}(\sigma-\Phi)\|_{L_2}^2+4\|\Pm{>1/\theta^2}(\Phi-\Tr \Phi)\|_{L_2}^2+4\|\Pm{>1/\theta^2}\Tr \Phi\|_{L_2}^2\;.
    \end{align*}
    Observe that $\Pm{>m}$ is a non-expansive operator since for any $f\in L_2(\calN)$, $f(z)\doteq \sum_{i\geq 0} a_i\he_i(z)$ it holds
    \begin{equation*}
        \|\Pm{>m}f\|_{L_2}^2= \sum_{i > m} a_i^2\leq \sum_{i\geq 0} a_i^2 = \|f\|_{L_2}^2.
    \end{equation*}
    Therefore, $\|\Pm{>1/\theta^2}(\sigma-\Phi)\|_{L_2}^2\leq \|\sigma - \Phi\|_{L_2}^2$. In addition, note that we have the following inequality for any $f,f'\in L_2(\calN)$:
    \begin{align*}
        \|\Pm{>m} f\|_2^2 = \sum_{i > m} a_i^2\leq \sum_{i > m} (i/m) a_i^2\leq \sum_{i=1}^m (i/m) a_i^2 +  \sum_{i > m} (i/m) a_i^2 = (1/m)\| f'\|_{L_2}^2, 
    \end{align*}
    therefore $\|\Pm{>1/\theta^2}\Tr\Phi\|_{L_2}^2\leq \theta^2\|(\Tr\Phi)'\|_{L_2}^2$. Finally by \Cref{fct:semi-group} we have $\|(\Tr\Phi)'\|_{L_2}^2 = \|\rho\Tr\Phi'\|_{L_2}^2\leq \|\Tr\Phi'\|_{L_2}^2$ since $\rho < 1$, thus, it holds
    
\begin{align}
      \|\Pm{>1/\theta^2}\sigma\|_{L_2}^2\leq 2\|\sigma-\Phi\|_{L_2}^2+4\|\Phi-\Tr \Phi\|_{L_2}^2+4\theta^2\|\Tr \Phi'\|_{L_2}^2\;.\label{app:eq:limit-theorem}
  \end{align}  
  Let $\Phi_k$ be any sequence of functions such that $\lim_{k \to \infty}\|\Phi_k-\sigma\|_{L_2}=0$. For this sequence we have that \Cref{app:eq:limit-theorem} becomes
    \begin{align}
      \|\Pm{>1/\theta^2}\sigma\|_{L_2}^2\leq 2\|\sigma-\Phi_k\|_{L_2}^2+4\|\Phi_k-\Tr \Phi_k\|_{L_2}^2+4\theta^2\|\Tr \Phi_k'\|_{L_2}^2\;.\label{app:eq:limit-theorem-sequence}
  \end{align} 
    In particular, let $\Phi_k$ be a sequence of staircase monotonic functions (see \Cref{app:def:staircase}) that converges to $\sigma$ uniformly; then, for $\rho^2 \geq 1 - C/M^2$ where $M$ is the upper bound on the support of $\sigma'$ (which is also the upper bound on the support of all $\Phi_k'$'s) and $C$ is an absolute constant, from \Cref{app:prop:smoothing-error-bound}, we conclude that $\|\Phi_k-\Tr\Phi_k\|_{L_2}^2\lesssim (1-\rho^2)\|\Tr \Phi_k'\|_{L_2} $ and therefore we have that
        \begin{align}
      \|\Pm{>1/\theta^2}\sigma\|_{L_2}^2\leq 2\|\sigma-\Phi_k\|_{L_2}^2+4((1-\rho^2)+\theta^2)\|\Tr \Phi_k'\|_{L_2}^2\;.\label{app:eq:limit-theorem-sequence-final}
  \end{align} 
  Our next goal is to show that the sequence of smoothed derivatives $\Tr \Phi_k'$ also converge to $\sigma'$, as stated in the following lemma.

\begin{lemma}[Convergence of Derivatives]\label{app:lem:norm-derivative-sequence} Let $\sigma:\R\to\R$, $\sigma\in L_2(\normal)$, and let $\Phi_k:\R\to\R$ be a sequence of functions such that $\|\sigma-\Phi_k\|_{L_2}\to 0$
as $k\to \infty$. Then, for any $\rho\in(0,1)$, it holds that
\[
\|\Tr\Phi'_k-\Tr \sigma'\|_{L_2}\to 0,\; \text{ as } k\to \infty.
\]
\end{lemma}
\begin{proof}
    For any function $f\in L_2(\normal)$, we have that (by the definition of \OU semigroup and Stein's lemma, stated in  \Cref{fct:stein})\begin{align*}
     \Tr f'(z)=\frac{1}{\sqrt{1-\rho^2}}\E_{t\sim \normal}[f(\rho z+\sqrt{(1-\rho^2)}t)t]\;.
 \end{align*}
 Therefore, we have that
 \begin{align*}
     \|\Tr\Phi'_k-\Tr \sigma'\|_{L_2}^2&=\frac{1}{{1-\rho^2}}\E_{z\sim \normal}\left[ \left(\E_{t\sim \normal}\left[\left(\Phi_k(\rho z+\sqrt{(1-\rho^2)}t)-\sigma(\rho z+\sqrt{(1-\rho^2)}t) \right)t\right]\right)^2\right]
     \\&\leq\frac{1}{{1-\rho^2}}\E_{z\sim \normal}\left[ \E_{t\sim \normal}\left[\left(\Phi_k(\rho z+\sqrt{(1-\rho^2)}t)-\sigma(\rho z+\sqrt{(1-\rho^2)}t) \right)^2\right] \E_{t\sim \normal}\left[t^2\right]\right]
     \\&=\frac{1}{{1-\rho^2}}\E_{z\sim \normal}\left[ \E_{t\sim \normal}\left[\left(\Phi_k(\rho z+\sqrt{(1-\rho^2)}t)-\sigma(\rho z+\sqrt{(1-\rho^2)}t) \right)^2\right]\right]
     \\&=\frac{1}{{1-\rho^2}}\E_{z\sim \normal}\left[ \left(\Phi_k(z)-\sigma(z) \right)^2\right]\;,
 \end{align*}
where the inequality is by the \CS inequality and the last inequality is by $\rho z+\sqrt{(1-\rho^2)}t \sim \normal(0,1)$ for independent $z\sim \normal(0, 1),\, t\sim \normal(0, 1)$. In remains to take the limit with $k\to\infty$.
 \end{proof}
Combining \Cref{app:lem:norm-derivative-sequence} with \Cref{app:eq:limit-theorem-sequence-final}, and letting $\rho = \cos\theta$ now completes the proof of \Cref{app:prop:error-bound-smoothing-tails}.  
\end{proof}

We recall that the assumption that the support of $\sigma'(z)$ is bounded by $M<+\infty$ is without loss of generality, as we have proved in \Cref{app:claim:bounded-support-sigma'}.

\subsection{Bounding the Augmentation Error}\label{app:subsection:bounding-aug-error}

In this subsection, we prove the main technical result, which provides an upper bound on the smoothing error of piecewise staircase functions using the $L_2(\calN)$ norm of the smoothed derivative. We recall the class of the piecewise staircase functions $\calF$ below:
\begin{equation*}
    \calF\eqdef \bigg\{\sum_{i=1}^m A_i\phi(z;t_i) + A_0: A_0\in\R; A_i > 0, |t_i|\leq M, \forall i\in[m]; m<\infty\bigg\}\;.
\end{equation*}
Our result is the following proposition:
\begin{restatable}{proposition}{SmoothgingErrorBound}\label{app:prop:smoothing-error-bound}
    Let $\Phi\in\calF$ be any staircase function that is consists of $m$ indicator functions with thresholds $t_i$, $i\in[m]$, and suppose $|t_i| \leq M$ for all $i\in[m]$, where $1< M<+\infty$. For any $\rho\in(0,1)$ such that $\rho^2\geq 1 - C/M^2$ where $C<M^2/4$ is an absolute constant, we have 
    \begin{equation*}
        \Ez[(\Tr\Phi(z)-\Phi(z))^2]\lesssim (1-\rho^2) \Ez[(\Tr\Phi'(z))^2].
    \end{equation*}
\end{restatable}

As we have remarked in the comment after the proof of \Cref{app:lem:activation-truncation-lip-and-mon-jump}, since $M$ is an upper bound on the support of $\sigma'$, we will assume without loss of generality throughout the rest of the paper that $M^2$ is larger than constant $4C$.

Some remarks about the staircase functions are in order. Observe first that according to \Cref{app:claim:bounded-support-sigma'}, when $\sigma$ is $(B,L)$-Regular, we can always bound $M$ by $\sqrt{2\log(4B^2/\eps) - \log\log(4B^2/\eps)}$. Next, for any function $\Phi\in\calF$, its derivative can be written as:
\begin{equation*}
    \Phi'(z) = \sum_{i=1}^m A_i\phi'(z;t_i) = \sum_{i=1}^m A_i\delta(z - t_i),
\end{equation*}
where $\delta(z - t_i)$ is the Dirac delta function. Certainly, when $|z|\geq M$ we have $\Phi'(z) = 0$. Also note that for any non-decreasing function $\sigma$ with the support of its derivative $\sigma'(z)$ bounded by $M$, there exists a sequence of staircase functions $\Phi_k\in\calF$ such that $\Phi_k$ converges to $\sigma$ uniformly. To prove this claim, we note that since for any $|z|\geq M$, $\sigma'(z) = 0$, therefore $\sigma(z)=\sigma(M)$ when $z\geq M$ and $\sigma(z) = \sigma(-M)$ for all $z\leq -M$. Hence, let
\begin{equation*}\begin{split}
    &\Phi_k(z) = \sum_{i = 1}^m \frac{1}{k}\phi(z; t_i) + \sigma(-M), \;\\ &\text{where } \begin{cases}
        m = \lceil{\sigma(M) - \sigma(-M)}/k \rceil + 1,\\
        t_i = \min_{t\in[-M,M]}\{\sigma(t)\geq (i - 1)(1/k) + \sigma(-M)\}, \,i=1,\dots,m-1; \, t_m = M. 
    \end{cases}
\end{split}
\end{equation*}
By construction,  we have $|\Phi_k(z) - \sigma(z)|\leq 1/k$ for all $z\in\R$, therefore $\Phi_k$ converges to $\sigma$ uniformly.

To prove \Cref{app:prop:smoothing-error-bound}, we decompose $\Ez[(\Tr\Phi(z)-\Phi(z))^2]$ into the following terms and provide upper bounds on each term respectively:
\begin{align}\label{app:eq:bound-error-0}
    &\quad \Ez[(\Tr\Phi(z)-\Phi(z))^2] \nonumber\\
    &\leq 2\Ez[(\Tr\Phi(z) - \Tre_{\rho_1}\Phi(z/\rho_1))^2] + 2\Ez[(\Tre_{\rho_1}\Phi(z/\rho_1) - \Phi(z))^2] \nonumber\\
    &\lesssim \Ez[(\Tr\Phi(z) - \Tre_{\rho_1}\Phi(z))^2] + \Ez[(\Tre_{\rho_1}\Phi(z) - \Tre_{\rho_1}\Phi(z/\rho_1))^2]+ \Ez[(\Tre_{\rho_1}\Phi(z/\rho_1) - \Phi(z))^2],
\end{align}
where we repeatedly used the inequality $(a+b)^2\leq 2a^2 + 2b^2$. 
As we have discussed at the beginning of \Cref{app:sec:learn-monotone}, we introduced this `recentered smoothing' operator $\Tr\Phi(z/\rho)$ to overcome the difficulty of analyzing $\Tr\Phi(z)-\Phi(z)$, since  $\Tr\Phi(z/\rho)$ takes a more simple and easy-to-analyze form. Here, $\rho_1\in(0,1)$ is a carefully chosen smoothing parameter that is slightly larger than $\rho$, so that we can bound $\|\Tre_{\rho_1}\Phi'(z/\rho_1)\|_{L_2}^2$ above using $\|\Tr\Phi'(z)\|_{L_2}^2$ (\Cref{app:lem:upper-bound-centered-smoothed-derivative-squared}). 

Coming back to \Cref{app:eq:bound-error-0}, we show that: 
(1) the first term $\Ez[(\Tr\Phi(z) - \Tre_{\rho_1}\Phi(z))^2]$ can be bounded above by $(1 - \rho)\|\Tre_{\rho_1}\Phi'(z)\|_{L_2}^2$, using \Cref{clm:difference}; 
(2) the second term $\Ez[(\Tre_{\rho_1}\Phi(z) - \Tre_{\rho_1}\Phi(z/\rho_1))^2]$ is bounded above by $(1 - \rho)(\|\Tre_{\rho_1}\Phi'(z/\rho_1)\|_{L_2}^2 + \|\Tre_{\rho_1}\Phi'(z)\|_{L_2}^2)$, using \Cref{app:lem:bound-the-differece-TrPhi-TrPhi(z/rho)}; and 
(3) the third term $\Ez[(\Tre_{\rho_1}\Phi(z/\rho_1) - \Phi(z))^2]$  is bounded above by $(1 - \rho)\|\Tre_{\rho_1}\Phi'(z/\rho_1)\|_{L_2}^2$, using \Cref{app:lem:upper-bound-centered-smoothed-error}.

Thus, in summary, we have $\Ez[(\Tr\Phi(z)-\Phi(z))^2]\lesssim (1 - \rho)(\|\Tre_{\rho_1}\Phi'(z/\rho_1)\|_{L_2}^2 + \|\Tre_{\rho_1}\Phi'(z)\|_{L_2}^2)$. Since $\rho_1$ is chosen so that $\|\Tre_{\rho_1}\Phi'(z/\rho_1)\|_{L_2}^2\lesssim \|\Tre_{\rho}\Phi'(z)\|_{L_2}^2$ (see \Cref{app:lem:upper-bound-centered-smoothed-derivative-squared}), and furthermore, since it holds that $\|\Tre_{\rho_1}\Phi'(z)\|_{L_2}^2\lesssim \|\Tr\Phi'(z)\|_{L_2}$,  combining these results we prove that $\Ez[(\Tr\Phi(z)-\Phi(z))^2]\lesssim (1 - \rho)\|\Tre_{\rho}\Phi'(z)\|_{L_2}^2$.

We first derive an explicit expression for $\Ez[(\Tr\Phi'(z))^2]$, for any $\Phi\in\calF$.
\begin{lemma}\label{app:lem:expression-E[(Trho-Phi'(z))**2]}
    For any $\Phi\in\calF$, it holds that
    \begin{equation*}
        \Ez[(\Tr\Phi'(z))^2] = \sum_{i,j = 1}^m \frac{A_iA_j}{2\pi\sqrt{1 - \rho^4}}\exp\bigg(-\frac{t_i^2 + t_j^2}{2(1 - \rho^4)} + \frac{\rho^2 t_it_j}{1 - \rho^4}\bigg).
    \end{equation*}
\end{lemma}
\begin{proof}
    By the linearity of the \OU semigroup, we have $\Tr\Phi'(z) = \sum_{i=1}^m A_i\Tr\phi'(z;t_i)$. In fact, each summand in this summation has an explicit expression, which we derive in the following:
    \begin{align}\label{app:eq:expression-Tr-phi'(z;t)}
        \Tr\phi'(z;t_i)&=\int_{-\infty}^{+\infty} \frac{1}{\sqrt{2\pi}} \phi'(\rho z + \sqrt{1 - \rho^2}u;t_i)\exp(-u^2/2)\diff{u} \nonumber\\
        &=\int_{-\infty}^{+\infty} \frac{1}{\sqrt{2\pi}} \delta(\rho z + \sqrt{1 - \rho^2}u - t_i)\exp(-u^2/2)\diff{u} \nonumber\\
        & = \frac{1}{\sqrt{2\pi(1 - \rho^2)}} \exp\bigg(-\frac{(\rho z - t_i)^2}{2(1 - \rho^2)}\bigg),
    \end{align}
    where we have used that $\delta$ is the Dirac delta function, and so  $\delta(u)$ satisfies $\delta(au) = \delta(u)/a$ for any real positive number $a$. 
Therefore, we get that
    \begin{align*}
    &\quad \Ez[(\Tr\Phi'(z))^2]\\
    & = \Ez\bigg[\sum_{i,j = 1}^m \frac{A_iA_j}{2\pi(1 - \rho^2)}\exp\bigg(-\frac{(\rho z - t_i)^2 + (\rho z - t_j)^2}{2(1 - \rho^2)}\bigg)\bigg]\\
        & = \sum_{i,j = 1}^m \frac{A_iA_j}{2\pi(1 - \rho^2)}\exp\bigg(-\frac{t_i^2 + t_j^2}{2(1 - \rho^2)}\bigg)\int_{-\infty}^{+\infty} \frac{1}{\sqrt{2\pi}} \exp\bigg(-\frac{\rho^2}{1 - \rho^2}z^2 + \frac{(t_i + t_j)\rho}{1 - \rho^2}z - \frac{z^2}{2}\bigg)
        \diff{z}\\
        &\overset{(i)}{=}\sum_{i,j = 1}^m \frac{A_iA_j}{2\pi(1 - \rho^2)}\exp\bigg(-\frac{t_i^2 + t_j^2}{2(1 - \rho^2)}\bigg)\sqrt{\frac{1 - \rho^2}{1 + \rho^2}}\exp\bigg(\frac{(t_i + t_j)^2\rho^2}{2(1 - \rho^4)}\bigg)\\
        & = \sum_{i,j = 1}^m \frac{A_iA_j}{2\pi\sqrt{1 - \rho^4}}\exp\bigg(-\frac{t_i^2 + t_j^2}{2(1 - \rho^4)} + \frac{\rho^2 t_it_j}{1 - \rho^4}\bigg),
    \end{align*}
    where in $(i)$ we used the fact that $\int \exp(-az^2 + bz)\diff{z} = \sqrt{\pi/a}\exp(b^2/4a)$.
\end{proof}

A byproduct of the above proof is that: 
\begin{claim}\label{app:claim:hermite-expansion-Trho-Phi'(z)}
    For any $\Phi\in\calF$, it holds 
\begin{equation*}
        \Tr\Phi'(z)=\sum_{k\geq 0}\rho^k\alpha_k\he_k(z),\;\text{where }\, \alpha_k = \sum_{i=1}^m \frac{A_i}{\sqrt{2\pi}}\exp(-t_i^2/2)\he_k(t_i).
    \end{equation*}
    Furthermore, the function $\zeta(\rho)\eqdef \Ez[(\Tr\Phi'(z))^2]$ is a non-decreasing function of $\rho\in(0,1)$.
\end{claim}
\begin{proof}
It is easy to see that $\Tr\phi'(z;t_i)$ is square-integrable under the Gaussian measure, therefore the Hermite expansion of $\Tr\phi'(z;t_i)$ exists. In particular, using Mehler's formula (\Cref{fct:hermite}), we can derive the Hermite expansion of $\Tr\phi'(z;t_i)$ immediately:
    \begin{equation*}
        \Tr\phi'(z;t_i)= \frac{1}{\sqrt{2\pi}}\sum_{k\geq 0} \rho^k\exp(-t_i^2/2)\he_k(t_i)\he_k(z),
    \end{equation*}
    which then implies that it holds
    \begin{align}\label{app:eq:hermite-expansion-Trho-Phi'(z)}
        \Tr\Phi'(z)&=\sum_{i=1}^m \frac{A_i}{\sqrt{2\pi}} \sum_{k\geq 0} \rho^k\exp(-t_i^2/2)\he_k(t_i)\he_k(z)\nonumber\\
        &= \sum_{k\geq 0}\rho^k\bigg(\sum_{i=1}^m \frac{A_i}{\sqrt{2\pi}}\exp(-t_i^2/2)\he_k(t_i)\bigg)\he_k(z).
    \end{align}
For the monotonicity of $\zeta(\rho)$, observe that by the Hermite expansion of $\Tr\Phi'(z)$, we have 
    \begin{align*}
        \zeta(\rho) = \Ez[(\Tr\Phi'(z))^2] = \sum_{k\geq 0}\rho^{2k}\alpha_k^2,
    \end{align*}
    which is an increasing function of $\rho\in(0,1)$.
\end{proof}
\Cref{app:claim:hermite-expansion-Trho-Phi'(z)} implies that though $\Phi'(z)$ is not in $L_2(\calN)$ (since the square of the Dirac delta function $\delta^2(z)$ is not integrable), $\Tr\Phi'(z)$ is well-defined and is continuous and smooth. Consequently, all the facts presented in \Cref{fct:semi-group} apply to $\Tr\Phi'(z)$ as well.

Proceeding to the analysis of $\Ez[(\Tr\Phi(z)-\Phi(z))^2]$, 
however, technical difficulties arise when we try to relate $\Ez[(\Tr\Phi(z)-\Phi(z))^2]$ with $\Ez[(\Tr\Phi'(z))^2]$. The main obstacle is that it is hard to analyze $\Tr\phi(z;t)-\phi(z;t)$, since 
\begin{align*}
    \Tr\phi(z;t)-\phi(z;t) = \pr_{u\sim\calN}[u\geq (t - \rho z)/\sqrt{1- \rho^2}] - \1\{z\geq t\},
\end{align*}
and the probability term does not have a close form.
The workaround is to study the centered augmentation (centered smoothing), and then translate the upper bound on the centered augmentation error back to the upper bound on the standard augmentation error.
\subsubsection{Centered Augmentation}\label{app:subsub:centered-aug}
We define the centered augmentation as the following:
\begin{equation*}
    \Tr\sigma(z/\rho) = \E_{u\sim\calN}[\sigma(z + (\sqrt{1 - \rho^2}/\rho) u)].
\end{equation*}
Note that for the staircase functions $\Phi\in\calF$, it holds
\begin{align*}
    \Tr\Phi(z/\rho) &= \sum_{i = 1}^m A_i\E_{u\sim\calN}[\1\{z + (\sqrt{1 - \rho^2}/\rho) u\geq t\}] \\
    &= \sum_{i = 1}^m A_i\E_{u\sim\calN}[\1\{\rho z + \sqrt{1 - \rho^2} u\geq \rho t_i\}] =\sum_{i=1}^m A_i\Tr\phi(z;\rho t_i). 
\end{align*}

We first provide explicit expressions for $\Tr\Phi'(z/\rho)$ and $\Ez[(\Tr\Phi'(z/\rho))^2]$.
\begin{lemma}\label{app:lem:expression-E[(Trho-Phi'(z/rho))**2]}
    For any $\Phi(z) = \sum_{i=1}^m A_i\phi(z;t_i) + A_0 \in\calF$, we have 
    \begin{align*}
        \Tr\Phi'(z/\rho) &= \sum_{i=1}^m \frac{\rho A_i}{\sqrt{2\pi(1 - \rho^2)}} \exp\bigg(-\frac{\rho^2(z - t_i)^2}{2(1 - \rho^2)}\bigg), \text{ and }\\
        \Ez[(\Tr\Phi'(z/\rho))^2] &= \sum_{i,j = 1}^m \frac{\rho^2 A_iA_j}{2\pi\sqrt{1 - \rho^4}}\exp\bigg(-\frac{\rho^2(t_i^2 + t_j^2)}{2(1 - \rho^4)} + \frac{\rho^4 t_it_j}{1 - \rho^4}\bigg).
    \end{align*}
\end{lemma}
\begin{proof}
The proof follows similar steps as the proof of \Cref{app:lem:expression-E[(Trho-Phi'(z))**2]}. Observe first that by the definition of $\Phi(z)$, the derivative of $\Phi$ equals
    \begin{equation*}
        \Phi'(z) = \sum_{i=1}^m A_i\phi'(z;t_i) = \sum_{i=1}^m A_i\delta(z - t_i),
    \end{equation*}
    where $\delta$ is the Dirac delta function. As $\delta(u)$ satisfies $\delta(au) = \delta(u)/a$ for any real positive number $a$, 
    \begin{equation*}
        \Phi'(z/\rho) = \sum_{i=1}^m A_i\delta((z - \rho t_i)/\rho) = \sum_{i=1}^m \rho A_i\delta(z - \rho t_i) = \rho\sum_{i=1}^m A_i\phi'(z;\rho t_i).
    \end{equation*}
    This implies that
    \begin{equation*}
        \Tr\Phi'(z/\rho) = \rho\sum_{i=1}^m A_i\Tr\phi'(z;\rho t_i),
    \end{equation*}
    which leads to the first claim in the statement after combining with \Cref{app:eq:expression-Tr-phi'(z;t)}. 
 The second claim now follows from \Cref{app:lem:expression-E[(Trho-Phi'(z))**2]}, by replacing $t_i,t_j$ with $\rho t_i$ and $\rho t_j$. \end{proof}

We now show that the centered augmentation error can be bounded  above by $\Ez[(\Tr\Phi'(z/\rho))^2]$.

\begin{lemma}\label{app:lem:upper-bound-centered-smoothed-error}
    Let $\Phi\in\calF$.  
    Then, for any $\rho \in (0, 1)$,  
    \begin{equation*}
        \Ez[(\Tr\Phi(z/\rho) - \Phi(z))^2]\leq 4((1 - \rho^2)/\rho^2) \Ez[(\Tr\Phi'(z/\rho))^2].
    \end{equation*} 
\end{lemma}
\begin{proof}
Observe that after augmentation, the indicator function $\1\{z\geq t\} = \phi(z;t)$ becomes $\Tr\phi(z/\rho;t) = \Tr\phi(z; \rho t) = \pr_{u\sim\calN}[u\geq \rho(t - z)/\sqrt{1 - \rho^2}]$. Therefore,  $\Tr\phi(z/\rho;t) - \phi(z;t)$ can be expressed as:
    \begin{align*}
        \Tr\phi(z/\rho;t) - \phi(z;t)  = \begin{cases}
            \pr_{u\sim\calN}[u\geq \rho(t - z)/\sqrt{1 - \rho^2}] & z < t,\\
            -\pr_{u\sim\calN}[u\leq \rho(t - z)/\sqrt{1 - \rho^2}] & z\geq t.
        \end{cases}
    \end{align*}
    Hence, $\Ez[(\Tr\Phi(z/\rho) - \Phi(z))^2]$  equals:
    \begin{align*}
        &\quad \Ez[(\Tr\Phi(z/\rho) - \Phi(z))^2]\\
        &=\Ez\bigg[\sum_{i,j=1}^m A_iA_j(\Tr\phi(z/\rho;t_i) - \phi(z;t_i))(\Tr\phi(z/\rho;t_j) - \phi(z;t_j))\bigg]\\
        &=\sum_{i,j=1}^m \frac{A_i A_j}{\sqrt{2\pi}}\int_{-\infty}^{\min\{t_i,t_j\}} \pr_{u\sim\calN}\bigg[u\geq \frac{\rho (t_i - z)}{\sqrt{1 - \rho^2}}\bigg]\pr_{u\sim\calN}\bigg[u\geq \frac{\rho (t_j - z)}{\sqrt{1 - \rho^2}}\bigg]e^{-z^2/2}\diff{z}\\
        &\quad - \sum_{i,j=1}^m \frac{A_i A_j}{\sqrt{2\pi}}\int_{\min\{t_i,t_j\}}^{\max\{t_i,t_j\}} \pr_{u\sim\calN}\bigg[u\geq \frac{\rho (\max\{t_i,t_j\} - z)}{\sqrt{1 - \rho^2}}\bigg]\pr_{u\sim\calN}\bigg[u\leq \frac{\rho (\min\{t_i,t_j\} - z)}{\sqrt{1 - \rho^2}}\bigg]e^{-z^2/2}\diff{z}\\
        &\quad + \sum_{i,j=1}^m \frac{A_i A_j}{\sqrt{2\pi}}\int_{\max\{t_i,t_j\}}^{+\infty} \pr_{u\sim\calN}\bigg[u\leq \frac{\rho (t_i - z)}{\sqrt{1 - \rho^2}}\bigg]\pr_{u\sim\calN}\bigg[u\leq \frac{\rho (t_j - z)}{\sqrt{1 - \rho^2}}\bigg]e^{-z^2/2}\diff{z}.
    \end{align*}
    When $z\leq \min\{t_i,t_j\}$, since both $\rho(t_i - z)$ and $\rho(t_j - z)$ are positive, by standard Gaussian concentration,  
    \begin{equation*}
        \pr_{u\sim\calN}\bigg[u\geq \frac{\rho (t_i - z)}{\sqrt{1 - \rho^2}}\bigg]\leq \frac{1}{2}\exp\bigg(-\frac{\rho^2(t_i - z)^2}{2(1 - \rho^2)}\bigg),\; \pr_{u\sim\calN}\bigg[u\geq \frac{\rho (t_j - z)}{\sqrt{1 - \rho^2}}\bigg]\leq \frac{1}{2}\exp\bigg(-\frac{\rho^2(t_j - z)^2}{2(1 - \rho^2)}\bigg).
    \end{equation*}
 The same inequalities hold for $\pr[u\leq \rho(t_i - z)/\sqrt{1 - \rho^2}]$ and $\pr[u\leq \rho(t_j - z)/\sqrt{1 - \rho^2}]$ when $z\geq \max\{t_i,t_j\}$. Thus, we can further upper bound $\Ez[(\Tr\Phi(z/\rho) - \Phi(z))^2]$ by
    \begin{align*}
       &\quad \Ez[(\Tr\Phi(z/\rho) - \Phi(z))^2]\\
       &\leq \sum_{i,j=1}^m \frac{A_i A_j}{\sqrt{2\pi}}\int_{-\infty}^{\min\{t_i,t_j\}} \pr_{u\sim\calN}\bigg[u\geq \frac{\rho (t_i - z)}{\sqrt{1 - \rho^2}}\bigg]\pr_{u\sim\calN}\bigg[u\geq \frac{\rho (t_j - z)}{\sqrt{1 - \rho^2}}\bigg]e^{-z^2/2}\diff{z}\\
        &\quad + \sum_{i,j=1}^m \frac{A_i A_j}{\sqrt{2\pi}}\int_{\max\{t_i,t_j\}}^{+\infty} \pr_{u\sim\calN}\bigg[u\leq \frac{\rho (t_i - z)}{\sqrt{1 - \rho^2}}\bigg]\pr_{u\sim\calN}\bigg[u\leq \frac{\rho (t_j - z)}{\sqrt{1 - \rho^2}}\bigg]e^{-z^2/2}\diff{z}\\
        &\leq \sum_{i,j = 1}^{m} \frac{1}{2}\frac{A_iA_j}{\sqrt{2\pi}} \bigg(\int_{-\infty}^{\min\{t_i,t_j\}} + \int_{\max\{t_i,t_j\}}^{+\infty}\bigg)\exp\bigg( - \frac{\rho^2((t_i - z)^2 + (t_j - z)^2)}{2(1 - \rho^2)} - \frac{z^2}{2}\bigg)\diff{z}\\
        &\leq \sum_{i,j = 1}^{m} \frac{1}{2}\frac{A_iA_j}{\sqrt{2\pi}} \int_{-\infty}^{+\infty}\exp\bigg( - \frac{\rho^2((t_i - z)^2 + (t_j - z)^2)}{2(1 - \rho^2)} - \frac{z^2}{2}\bigg)\diff{z}\\
        &= \sum_{i,j = 1}^{m} \frac{1}{2} A_iA_j\sqrt{\frac{1 - \rho^2}{1 + \rho^2}}\exp\bigg(-\frac{\rho^2(t_i^2 + t_j^2)}{2(1 - \rho^4)} + \frac{\rho^4 t_it_j}{1 - \rho^4}\bigg),
    \end{align*}
    where in the last inequality we used the definition of Gaussian pdf with variance $\frac{1 - \rho^2}{1 + \rho^2}$ and the fact that its integral over the real line is equal to one. 
    Comparing with the expression for $\Ez[(\Tr\Phi'(z/\rho))^2]$ from \Cref{app:lem:expression-E[(Trho-Phi'(z/rho))**2]}, we immediately get the claimed bound on $\Ez[(\Tr\Phi'(z/\rho))^2]$. 
\end{proof}

Our next result shows that when $\rho$ is close to 1, the centered augmentation $\Tr\Phi(z/\rho)$ does not differ much  from the uncentered augmentation $\Tr\Phi(z)$, as stated below.
\begin{lemma}\label{app:lem:bound-the-differece-TrPhi-TrPhi(z/rho)}
Let $\Phi\in\calF$. Suppose $1>\rho^2\geq 1-C/M^2$ for an absolute constant $C \in (0, M^2/2]$. Then: 
    \begin{equation*}
        \Ez[(\Tr\Phi(z) - \Tr\Phi(z/\rho))^2] \leq C'(1 - \rho^2)(\|\Tr\Phi'(z/\rho)\|_{L_2}^2 + \|\Tr\Phi'(z)\|_{L_2}^2),
    \end{equation*}
    where $C'$ is an absolute constant.
\end{lemma}

\begin{proof}
    We first observe that since $\Tre_{\rho}$ is a linear operator on functionals, we have $\Tre_{\rho}\Phi(z) - \Tre_{\rho}\Phi(z/\rho) =  \Tre_{\rho}(\Phi(z) - \Phi(z/\rho))$. Given a staircase function $\Phi\in\calF$, $\Phi(z) = \sum_{i=1}^m A_i\phi(z;t_i) + A_0$, let $I_+ = \{i:t_i>0\}$ and $I_- = \{i:t_i<0\}$. Expressing $\Tre_{\rho}(\Phi(z) - \Phi(z/\rho))$ in terms of the sum of indicator functions we get
    \begin{align*}
        \big|\Tre_{\rho}(\Phi(z) - \Phi(z/\rho))\big|&= \bigg| \Tre_{\rho}\bigg(\sum_{i=1}^m A_i(\phi(z;t_i) - \phi(z/\rho;t_i)) \bigg) \bigg|\\
        &= \bigg| \Tre_{\rho}\bigg(\sum_{i=1}^m A_i(-\1\{\rho t_i\leq z\leq t_i, t_i\geq 0\} + \1\{t_i\leq z\leq \rho t_i, t_i\leq 0\})\bigg) \bigg|\\
        &\leq \sum_{i=1}^m A_i\bigg|\Tr\bigg(-\1\{\rho t_i\leq z\leq t_i, t_i\geq 0\} + \1\{t_i\leq z\leq \rho t_i, t_i\leq 0\}\bigg)\bigg|\\
        &=\sum_{i\in I_+} A_i\Tr(\1\{\rho t_i\leq  z\leq t_i\}) + \sum_{i\in I_-} A_i\Tr(\1\{t_i\leq z\leq \rho t_i\})\\
\end{align*}
    Suppose first $t_i\geq 0$. Then by the definition of \OU semigroup, we have
    \begin{equation*}
        g_i(z)\eqdef \Tre_{\rho} (\1\{\rho t_i\leq z\leq t_i\})= \E_{u\sim\calN}[\1\{\rho t_i\leq \rho z + \sqrt{1 - \rho^2}u \leq t_i\}] =\int_{\rho (t_i - z)/\sqrt{1 - \rho^2}}^{(t_i - \rho z)/\sqrt{1 - \rho^2}} \frac{e^{-u^2/2}}{\sqrt{2\pi}} \diff{u}.
    \end{equation*}
    When $z\leq t_i$ or $z \geq t_i/\rho$, $t_i - \rho z$ and $\rho(t_i - z)$ are both positive or negative, therefore, when $z\in(-\infty,t_i]\cup[t_i/\rho,+\infty)$, the function $g_i(z)$ can be bounded by
    \begin{align*}
        g_i(z)&=\int_{\rho (t_i - z)/\sqrt{1 - \rho^2}}^{(t_i - \rho z)/\sqrt{1 - \rho^2}} \frac{e^{-u^2/2}}{\sqrt{2\pi}} \diff{u}\\
        &\leq \frac{1}{\sqrt{2\pi}}\bigg(\frac{t_i - \rho z}{\sqrt{1 - \rho^2}} - \frac{\rho(t_i - \rho z)}{\sqrt{1 -\rho^2}}\bigg)\exp\bigg(-\frac{1}{2}\min\bigg\{\frac{(t_i - \rho z)^2}{1 - \rho^2}, \frac{\rho^2 (t_i - z)^2}{1 - \rho^2}\bigg\}\bigg)\\
        &\leq \frac{(1 - \rho)t_i}{\sqrt{2\pi(1 - \rho^2)}}\bigg(\exp\bigg(-\frac{(t_i - \rho z)^2}{2(1 - \rho^2)}\bigg) + \exp\bigg(-\frac{\rho^2(t_i - z)^2}{2(1 - \rho^2)}\bigg)\bigg).
    \end{align*}
    Comparing the right-hand side of the inequality above with the expressions for $\Tre_{\rho}\phi'(z;t_i)$ and $\Tre_{\rho}\phi'(z;\rho t_i)$ displayed in \Cref{app:eq:expression-Tr-phi'(z;t)} and \Cref{app:lem:expression-E[(Trho-Phi'(z/rho))**2]}, 
    \begin{gather*}
        \Tre_{\rho}\phi'(z;t_i) = \frac{1}{\sqrt{2\pi(1 - \rho^2)}} \exp\bigg(-\frac{(\rho z - t_i)^2}{2(1 - \rho^2)}\bigg)\\
        \Tre_{\rho}\phi'(z;\rho t_i) = \frac{1}{\sqrt{2\pi(1 - \rho^2)}} \exp\bigg(-\frac{\rho^2(z - t_i)^2}{2(1 - \rho^2)}\bigg)
    \end{gather*}
    we obtain that 
    \begin{align*}
        g_i(z)\1\{z\leq t_i \text{ or } z\geq t_i/\rho\} \leq {(1-\rho)t_i}(\Tre_{\rho}\phi'(z;t_i) + \Tre_{\rho}\phi'(z;\rho t_i)).
    \end{align*}
    On the other hand, when $z\in[t_i,t_i/\rho]$, since $0\in[\rho (t_i - z)/\sqrt{1 - \rho^2}, (t_i - \rho z)/\sqrt{1 - \rho^2}]$ we can bound $g_i(z)$  above by
    \begin{equation*}
        g(z)\leq \int_{\rho (t_i - z)/\sqrt{1 - \rho^2}}^{(t_i - \rho z)/\sqrt{1 - \rho^2}} \frac{1}{\sqrt{2\pi}} \diff{u}\leq \frac{(1 - \rho)t_i}{\sqrt{2\pi(1 - \rho^2)}}.
    \end{equation*}
    Thus, in summary, $g_i(z)$ is bounded above by
    \begin{align*}
        g_i(z)&= g_i(z)\1\{z\leq t_i \text{ or } z\geq t_i/\rho\} + g_i(z)\1\{t_i\leq z \leq t_i/\rho\}\\
        &\leq {(1-\rho)t_i}(\Tre_{\rho}\phi'(z;t_i) + \Tre_{\rho}\phi'(z;\rho t_i)) + \frac{(1 - \rho)t_i}{\sqrt{2\pi(1 - \rho^2)}}\1\{t_i\leq z \leq t_i/\rho\}.
    \end{align*}
    Similarly, for $i\in I_-$, with the same arguments we obtain that
    \begin{align*}
        g_i(z)\leq {(1-\rho)|t_i|}(\Tre_{\rho}\phi'(z;t_i) + \Tre_{\rho}\phi'(z;\rho t_i)) + \frac{(1 - \rho)|t_i|}{\sqrt{2\pi(1 - \rho^2)}}\1\{t_i/\rho\leq z \leq t_i\}.
    \end{align*}

Therefore, the $L_2^2$ difference between $\Tr\Phi(z)$ and $\Tr\Phi(z/\rho)$ can be bounded by (note that $A_i, g_i(z)>0$ for all $i\in [m]$)
\begin{align}\label{app:eq:bound-Trphi(z)-TrPhi(z/r)-eq-0}
&\quad \Ez[(\Tr\Phi(z) - \Tr\Phi(z/\rho))^2] = \Ez\bigg[\bigg(\sum_{i=1}^m A_i g_i(z)\bigg)^2\bigg] \nonumber\\
&\leq \Ez\bigg[\bigg(\sum_{i=1}^m { A_i(1-\rho)|t_i|}(\Tre_{\rho}\phi'(z;t_i) + \Tre_{\rho}\phi'(z;\rho t_i)) \nonumber\\
&\quad \quad + \frac{ A_i (1 - \rho)|t_i|}{\sqrt{2\pi(1 - \rho^2)}}\sum_{i=1}^m (\1\{z\in[t_i, t_i/\rho]\} + \1\{z\in [t_i/\rho, t_i]\})\bigg)^2\bigg] \nonumber\\
&\leq 2\underbrace{\Ez\bigg[\bigg(\sum_{i=1}^m { A_i (1-\rho)|t_i|}(\Tre_{\rho}\phi'(z;t_i) + \Tre_{\rho}\phi'(z;\rho t_i))\bigg)^2\bigg]}_{(Q_1)} \nonumber\\
&\quad \quad + 2\underbrace{\Ez\bigg[\bigg(\frac{ (1 - \rho)}{\sqrt{2\pi(1 - \rho^2)}}\sum_{i=1}^m A_i |t_i| (\1\{z\in[t_i, t_i/\rho]\} + \1\{z\in [t_i/\rho, t_i]\})\bigg)^2\bigg]}_{(Q_2)}.
\end{align}
Note that in $(Q_2)$ above, we used the convention that if $a>b$ then $[a,b] = \emptyset$ and $\1\{z\in\emptyset\} = 0$. 
For $(Q_1)$, using Young's inequality again yields:
\begin{align*}
    (Q_1) &\leq 2\Ez\bigg[\bigg(\sum_{i=1}^m  A_i (1-\rho)|t_i|\Tre_{\rho}\phi'(z;t_i)\bigg)^2\bigg] + 2\Ez\bigg[\bigg(\sum_{i=1}^m  A_i (1-\rho)|t_i|\Tre_{\rho}\phi'(z;\rho t_i)\bigg)^2\bigg]\\
    &\leq 2(1-\rho)^2\bigg(\max_{i\in[m]}\{|t_i|\}\bigg)^2\bigg(\Ez\bigg[\bigg(\sum_{i=1}^m  A_i \Tre_{\rho}\phi'(z; t_i)\bigg)^2\bigg] + \Ez\bigg[\bigg(\sum_{i=1}^m  A_i \Tre_{\rho}\phi'(z;\rho t_i)\bigg)^2\bigg]\bigg)\\
    &\leq 2(1 - \rho)^2M^2(\|\Tr\Phi'(z)\|_{L_2}^2 + \|\Tr\Phi'(z/\rho)\|_{L_2}^2),
\end{align*}
where in the last inequality, we use the fact that since $\Phi\in\calF$, we have $|t_i|\leq M$ for all $i\in [m]$. Now, by our assumption, $\rho^2\geq 1- C/M^2$, therefore, $(1 - \rho)(1 + \rho)\leq  C/M^2$ and hence $1 - \rho\leq C/M^2$, which implies
\begin{equation*}
(Q_1)\leq 2C(1-\rho)    (\|\Tr\Phi'(z)\|_{L_2}^2 + \|\Tr\Phi'(z/\rho)\|_{L_2}^2).
\end{equation*}

For $(Q_2)$, since $|t_i|\leq M$ for all $i\in [m]$ and $1 - \rho\leq C/M^2$, expanding the square yields
\begin{align*}
    (Q_2)&\leq \frac{M^2(1 - \rho)}{2\pi(1 + \rho)}\Ez\bigg[\bigg(\sum_{i=1}^m A_i (\1\{z\in[t_i, t_i/\rho]\} + \1\{z\in [t_i/\rho, t_i]\})\bigg)^2\bigg]\\
    &\leq \frac{C}{4\pi} \Ez\bigg[\sum_{i,j\in I_+}A_iA_j \1\{z\in[t_i,t_i/\rho]\cap[t_j,t_j/\rho]\} + \sum_{i,j\in I_-}A_iA_j \1\{z\in[t_i/\rho, t_i]\cap[t_j/\rho,t_j]\}\bigg]
\end{align*}
By the symmetry of Gaussian distribution, we have
\begin{equation*}
    \pr[z\in[t_i/\rho, t_i]\cap[t_j/\rho,t_j], t_i,t_j\leq 0] = \pr[z\in[|t_i|,|t_i|/\rho]\cap[|t_j|,|t_j|/\rho]],
\end{equation*}
therefore, it suffices to discuss only the case where $t_i,t_j\in I_+$. Suppose without loss of generality that $0<t_i\leq t_j$. Observe that $\Ez[\1\{z\in[t_i,t_i/\rho]\cap[t_j,t_j/\rho]\}]\neq 0$ if and only if $0<t_i\leq t_j<t_i/\rho\leq t_j\rho$, therefore, the expectation of the indicator is bounded by:
\begin{align}\label{app:eq:bound-Trphi(z)-TrPhi(z/r)-eq-1}
    \Ez[\1\{z\in[t_i,t_i/\rho]\cap[t_j,t_j/\rho]\}]&=\pr[z\in[t_j, t_i/\rho]] = \int_{t_j}^{t_i/\rho} \frac{e^{-u^2/2}}{\sqrt{2\pi}} \diff{u}\nonumber\\
    &\leq \frac{\exp(-t_j^2/2)}{\sqrt{2\pi}}(t_i/\rho - t_j)\leq \frac{(1 - \rho)t_i}{\rho\sqrt{2\pi}}\exp\bigg(-\frac{t_j^2}{2}\bigg).
\end{align}
Recall that in \Cref{app:lem:expression-E[(Trho-Phi'(z/rho))**2]}, we proved
\begin{equation*}
        \Ez[(\Tr\Phi'(z/\rho))^2] = \sum_{i,j = 1}^m \frac{\rho^2 A_iA_j}{2\pi\sqrt{1 - \rho^4}}\exp\bigg(-\frac{\rho^2(t_i^2 + t_j^2)}{2(1 - \rho^4)} + \frac{\rho^4 t_it_j}{1 - \rho^4}\bigg),
    \end{equation*}
hence our strategy is to show that:
\begin{equation}\label{app:eq:bound-Trphi(z)-TrPhi(z/r)-eq-2}
    \exp\bigg(-\frac{t_j^2}{2}\bigg)\leq \exp\bigg(-\frac{\rho^2(t_i^2 + t_j^2)}{2(1 - \rho^4)} + \frac{\rho^4 t_it_j}{1 - \rho^4}\bigg), \;\text{for }0<t_i\leq t_j<t_i/\rho\leq t_j/\rho.
\end{equation}
We show:
\begin{restatable}{claim}{HelperTrPhiMinusTrPhisquare}\label{app:claim:helper-[TrPhi-TrPhi(z/r)]**2}
    Let $t_i,t_j>0$ satisfy $t_i\leq t_j\leq t_i/\rho$. Then, for any $\rho\in(0,1)$, it holds 
    \begin{equation*}
        -\frac{t_j^2}{2} \leq - \frac{\rho^2(t_i^2 + t_j^2)}{2(1 - \rho^4)} + \frac{\rho^4 t_it_j}{1 - \rho^4}.
    \end{equation*}
\end{restatable}
The proof of \Cref{app:claim:helper-[TrPhi-TrPhi(z/r)]**2} is deferred to \Cref{app:subsec:proof-supplementary-claims-in-monotone}.
Therefore, for each $t_i,t_j$, $i,j\in [m]$, the expectation in \Cref{app:eq:bound-Trphi(z)-TrPhi(z/r)-eq-1} is bounded above by
\begin{align*}
    &\quad \Ez[\1\{z\in[t_i,t_i/\rho]\cap[t_j,t_j/\rho], t_i t_j > 0\}]\\
    &\leq \frac{(1 - \rho)\sqrt{2\pi(1 - \rho^4)}|t_i|}{\rho^3}\frac{\rho^2}{2\pi\sqrt{1 - \rho^4}}\exp\bigg(-\frac{\rho^2(t_i^2 + t_j^2)}{2(1 - \rho^4)} + \frac{\rho^4 t_it_j}{1 - \rho^4}\bigg),
\end{align*}
which, combining with the fact that $\sqrt{1 -\rho}\leq \sqrt{C}/M$ and $|t_i|\leq M$, yields
\begin{align*}
    (Q_2)\leq\;&  \frac{C}{4\pi}\bigg(\sum_{i,j\in I_+}A_iA_j\Ez[\1\{z\in[t_i,t_i/\rho]\cap[t_j,t_j/\rho]\}]\\
    &\quad + \sum_{i,j\in I_-} A_iA_j\Ez[\1\{z\in [t_i/\rho, t_i]\cap [t_j\rho,t_j]\}]\bigg)\\
    \leq\; & \sum_{i,j=1}^m \frac{C(1 - \rho)\sqrt{2\pi(1 - \rho^4)}|t_i|}{4\pi\rho^3}\frac{\rho^2}{2\pi\sqrt{1 - \rho^4}}\exp\bigg(-\frac{\rho^2(t_i^2 + t_j^2)}{2(1 - \rho^4)} + \frac{\rho^4 t_it_j}{1 - \rho^4}\bigg)\\
    \leq\; & C'(1 - \rho)\|\Tr\Phi'(z/\rho)\|_{L_2}^2.
\end{align*}

Plugging the bounds on $(Q_1)$, $(Q_2)$ back to \Cref{app:eq:bound-Trphi(z)-TrPhi(z/r)-eq-0}, we finally obtain:
\begin{equation*}
    \Ez[(\Tr\Phi(z) - \Tr\Phi(z/\rho))^2] \leq C''(1 - \rho)(\|\Tr\Phi'(z/\rho)\|_{L_2}^2 + \|\Tr\sigma'(z)\|_{L_2}^2).
\end{equation*}
Since $1 - \rho\leq 1 - \rho^2$, we complete the proof of \Cref{app:lem:bound-the-differece-TrPhi-TrPhi(z/rho)}.
\end{proof}

Our last step is to show that $\Ez[(\Tr\Phi'(z/\rho))^2]$ is not much larger than $\Ez[(\Tr\Phi'(z))^2]$  when $\rho$ is close to 1.
\begin{lemma}\label{app:lem:upper-bound-centered-smoothed-derivative-squared}
    Let $\Phi\in\calF$ be any staircase function that is constructed from $m$ indicator functions with thresholds $t_i$, $i\in[m]$, and suppose that $|t_i| \leq M$, for all $i\in[m]$, where $1< M<+\infty$. For any $\rho\in(0,1)$ such that $\rho^2\geq 1 - C/M^2$ where $C<M^2$ is an absolute constant, let $\rho_1 = \sqrt{\rho^2 + C(1 - \rho^2)/M^2}$. Then, \begin{equation*}
        \Ez[(\Tre_{\rho_1}\Phi'(z/\rho_1))^2]\leq 2e^C  \Ez[(\Tr\Phi'(z))^2].
    \end{equation*}
\end{lemma}

\begin{proof}
    Observe first that $1 -\rho_1^2 = (1 - \rho^2)(1 - C/M^2) \in (0, 1)$, hence $\rho_1 \in (0, 1)$ and $\Tre_{\rho_1}\Phi'(z/\rho_1)$ is well-defined. To proceed, we compare each term of $\Ez[(\Tre_{\rho_1}\Phi'(z/\rho_1))]$ and $\Ez[(\Tr\Phi'(z))^2]$ that are given in \Cref{app:lem:expression-E[(Trho-Phi'(z/rho))**2]} and \Cref{app:lem:expression-E[(Trho-Phi'(z))**2]} separately.
    
    Since $\rho_1^2/(1 - \rho_1^4)$ appears in the exponential terms of $\Ez[(\Tre_{\rho_1}\Phi'(z/\rho_1))]$ while the coefficient in the exponential terms of $\Ez[(\Tr\Phi'(z))^2]$ is $1/(1 - \rho^4)$, we first need to compare these two factors. The proof of \Cref{app:claim:helper-rho1-rho-ratio} is deferred to \Cref{app:subsec:proof-supplementary-claims-in-monotone}.
\begin{restatable}{claim}{helperRhoOneRhorationOne}\label{app:claim:helper-rho1-rho-ratio}
        Let $\rho_1^2 = \rho^2 + C(1 - \rho^2)/M^2$. If $1 > \rho^2\geq 1 - C/M^2$, then $\rho_1^2/(1 -\rho_1^4)\geq 1/(1 - \rho^4)$.
    \end{restatable}
Observe that for any $t_i,t_j\in\R$ and $\rho\in(0,1)$, we have $t_i^2 + t_j^2 - 2\rho_1^2 t_i t_j\geq (1 - \rho^2)(t_i^2 + t_j^2)\geq 0$, and recalling the expression for $\Ez[(\Tre_{\rho_1}\Phi'(z/\rho_1))^2]$  given in \Cref{app:lem:expression-E[(Trho-Phi'(z/rho))**2]}, we thus obtain
    \begin{align}\label{app:ineq:upbd-Trho1-Phi-(z/rho1)-1}
        \Ez[(\Tre_{\rho_1}\Phi'(z/\rho_1))^2]&=\sum_{i,j=1}^m \frac{ A_iA_j}{2\pi}\sqrt{\frac{\rho_1^4}{1 - \rho_1^4}}\exp\bigg(-\frac{\rho_1^2(t_i^2 + t_j^2 - 2\rho_1^2 t_i t_j)}{2(1 - \rho_1^4)}\bigg)\nonumber\\
        &\overset{(i)}{\leq} \sum_{i,j=1}^m \frac{ A_iA_j}{2\pi}\sqrt{\frac{\rho_1^4}{1 - \rho_1^4}} \exp\bigg(-\frac{t_i^2 + t_j^2}{2(1 - \rho^4)} + \frac{\rho_1^2 t_i t_j}{1 - \rho^4}\bigg)\nonumber\\
        &\overset{(ii)}{=}\sum_{i,j=1}^m \frac{ A_iA_j}{2\pi}\sqrt{\frac{\rho_1^4}{1 - \rho_1^4}} \exp\bigg(-\frac{t_i^2 + t_j^2}{2(1 - \rho^4)} + \frac{\rho^2 t_i t_j}{1 - \rho^4} + \frac{C(1 - \rho^2) t_i t_j}{(1 - \rho^4)M^2}\bigg)\nonumber\\
        &=\sum_{i,j=1}^m \frac{ A_iA_j}{2\pi}\sqrt{\frac{\rho_1^4}{1 - \rho_1^4}} \exp\bigg(-\frac{t_i^2 + t_j^2}{2(1 - \rho^4)} + \frac{\rho^2 t_i t_j}{1 - \rho^4}\bigg)\exp\bigg(\frac{Ct_i t_j}{(1 + \rho^2)M^2}\bigg),
    \end{align}
    where in $(i)$ we plugged in \Cref{app:claim:helper-rho1-rho-ratio} and in $(ii)$ we used the definition that $\rho_1^2 = \rho^2 + C(1 - \rho^2)/M^2$.

    Since $M$ is an upper bound on $|t_i|, i\in[m]$, we can assume without loss of generality that $M^2\geq 2C$. We next observe that when $M^2\geq 2C$, we have the following inequality, whose proof is relocated to \Cref{app:subsec:proof-supplementary-claims-in-monotone}.
\begin{restatable}{claim}{helperRhoOneRhoRatio}\label{app:claim:helper2-rho1-rho-ratio}
        If $M^2\geq 2C$ then  $\rho_1^4/(1 - \rho_1^4)\leq 4/(1 - \rho^4).$
    \end{restatable}

    Thus, plugging in \Cref{app:claim:helper2-rho1-rho-ratio} into \Cref{app:ineq:upbd-Trho1-Phi-(z/rho1)-1}, and recalling that it is assumed $|t_i|^2\leq M^2$ for any $i\in[m]$, we further get
    \begin{align*}\Ez[(\Tre_{\rho_1}\Phi'(z/\rho_1))^2]
        &\leq \sum_{i,j=1}^m \frac{ A_iA_j}{\pi}\frac{1}{\sqrt{1 - \rho^4}} \exp\bigg(-\frac{t_i^2 + t_j^2}{2(1 - \rho^4)} + \frac{\rho^2 t_i t_j}{1 - \rho^4}\bigg)\exp\bigg(\frac{C t_i t_j}{(1 + \rho^2)M^2}\bigg) \nonumber\\
        &\leq 2e^C\sum_{i,j=1}^m \frac{A_iA_j}{2\pi\sqrt{1 - \rho^4}} \exp\bigg(-\frac{t_i^2 + t_j^2}{2(1 - \rho^4)} + \frac{\rho^2 t_i t_j}{1 - \rho^4}\bigg) \nonumber\\
        &=2e^C \Ez[(\Tr\Phi'(z))^2],
    \end{align*}
    which completes the proof.
\end{proof}

\subsubsection{Proof of \Cref{app:prop:smoothing-error-bound}}\label{app:subsec:proof-of-smoothing-error-bound}
We can now restate \Cref{app:prop:smoothing-error-bound} and present its proof.
\SmoothgingErrorBound*
\begin{proof}[Proof of \Cref{app:prop:smoothing-error-bound}] 
Observe  that $\Ez[(\Tr\Phi(z) - \Phi(z))^2]$ can be bounded as 
    \begin{align}\label{app:eq:upper-bound-smoothed-error-0}
        \Ez[(\Tr\Phi(z) - \Phi(z))^2]& \leq 2\Ez[(\Tr\Phi(z) - \Tre_{\rho_1}\Phi(z/\rho_1))^2] + 2\Ez[(\Tre_{\rho_1}\Phi(z/\rho_1) - \Phi(z))^2].
    \end{align}
    We first bound the second term in \Cref{app:eq:upper-bound-smoothed-error-0}.
    Since we have assumed that $\rho^2\geq 1 - C/M^2$, where $C$ is an absolute constant, using \Cref{app:lem:upper-bound-centered-smoothed-error} and \Cref{app:lem:upper-bound-centered-smoothed-derivative-squared} and plugging in $\rho_1 = \sqrt{\rho^2 + C(1 - \rho^2)/M^2}$, we get
    \begin{align}\label{app:eq:upper-bound-smoothed-error-1}
        \Ez[(\Tre_{\rho_1}\Phi(z/\rho_1) -\Phi(z))^2]&\leq \frac{4(1 - \rho_1^2)}{\rho_1^2}\Ez[(\Tre_{\rho_1}\Phi'(z/\rho_1))^2] \nonumber\\
        &\leq \frac{8e^C (1 - \rho^2 - C(1 - \rho^2)/M^2)}{(\rho^2 + C(1 - \rho^2)/M^2)}\Ez[(\Tr\Phi'(z))^2] \nonumber\\
        &\lesssim e^C (1 - \rho^2)\Ez[(\Tr\Phi'(z))^2].
    \end{align}

    Now we turn to bounding the first term in \Cref{app:eq:upper-bound-smoothed-error-0}. First, we add and subtract $\Tre_{\rho_1}\Phi'(z)$ in the squared parentheses to obtain:
    \begin{equation}\label{app:eq:upper-bound-smoothed-error-2}
        \Ez[(\Tr\Phi(z) - \Tre_{\rho_1}\Phi(z/\rho_1))^2]\leq 2\underbrace{\Ez[(\Tr\Phi(z) - \Tre_{\rho_1}\Phi(z))^2]}_{Q_1} + 2\underbrace{\Ez[(\Tre_{\rho_1}\Phi(z) - \Tre_{\rho_1}\Phi(z/\rho_1))^2]}_{Q_2}.
    \end{equation}
    For $Q_1$, observe that since $\rho<\rho_1<1$, using the property of \OU semigroup presented in \Cref{fct:semi-group}, we have $\Tr\Phi(z) = \Tre_{\rho_1(\rho/\rho_1)}\Phi(z) = \Tre_{\rho/\rho_1}(\Tre_{\rho_1}\Phi(z))$. Therefore, using \Cref{clm:difference} with $f(z) = \Tre_{\rho_1}\Phi(z)$ we have
    \begin{align}\label{app:eq:upper-bound-smoothed-error-3}
        \Ez[(\Tre_{\rho/\rho_1}(\Tre_{\rho_1}\Phi(z)) - \Tre_{\rho_1}\Phi(z))^2]&\leq 3 (1 - \rho/\rho_1)\Ez\bigg[\bigg(\frac{\diff{}}{\diff{z}}\Tre_{\rho_1}\Phi(z)\bigg)^2\bigg] \nonumber\\
        &\overset{(i)}{=}\frac{3(\rho_1 - \rho)}{\rho_1}\rho_1^2\Ez[(\Tre_{\rho_1}\Phi'(z))^2] \nonumber\\
        &=\frac{3(\rho_1^2 - \rho^2)\rho_1}{\rho_1 + \rho}\Ez[(\Tre_{\rho_1}\Phi'(z))^2] \nonumber\\
        &\overset{(ii)}{=} \frac{3C(1 - \rho^2)\rho_1}{M^2(\rho_1 +\rho)}\Ez[(\Tre_{\rho_1}\Phi'(z))^2],
    \end{align}
   note that in $(i)$ we applied \Cref{lem:interchange-Tr-and-differentiation}, Part 2(g), and in $(ii)$ we brought in the definition of $\rho_1^2 = \rho^2 + C(1 -\rho^2)/M^2$. It remains to bound $\Ez[(\Tre_{\rho_1}\Phi'(z))^2]$ above by $\Ez[(\Tre_{\rho}\Phi'(z))^2]$. The proof of \Cref{app:clm:||Tr rho1 Phi'||2<= eC ||Tr rho Phi'||2} is deferred to \Cref{app:subsec:proof-supplementary-claims-in-monotone}.
   
   \begin{restatable}{claim}{HelperTrPhiSquare}\label{app:clm:||Tr rho1 Phi'||2<= eC ||Tr rho Phi'||2}
       Let $\rho^2\geq 1 - C/M^2$ and $\rho_1^2 = \rho^2 + C(1 - \rho^2)/M^2$. Then, $\|\Tre_{\rho_1}\Phi'(z)\|_{L_2}^2\leq e^C\|\Tre_{\rho}\Phi'(z)\|_{L_2}^2$.
   \end{restatable}
    
    Plugging the upper bound on $\Ez[(\Tre_{\rho_1}\Phi'(z))^2]$ from \Cref{app:clm:||Tr rho1 Phi'||2<= eC ||Tr rho Phi'||2} back into \Cref{app:eq:upper-bound-smoothed-error-3} yields
    \begin{align*}
        Q_1 = 2\Ez[(\Tre_{\rho/\rho_1}(\Tre_{\rho_1}\Phi(z)) - \Tre_{\rho_1}\Phi(z))^2]&\leq \frac{C(1 - \rho^2)\rho_1}{M^2(\rho_1 +\rho)}4e^C\Ez[(\Tre_{\rho}\Phi'(z))^2].
    \end{align*}
    Since $C/M^2\leq 1/4$ we have $1>\rho_1>\rho\geq 1/2$, thus we finally get
    \begin{equation*}
        Q_1\leq 2e^C (1 - \rho^2)\|\Tr\Phi'\|_{L_2}^2.
    \end{equation*}

    We now turn to bounding the term $Q_2$ in \Cref{app:eq:upper-bound-smoothed-error-1}. Applying \Cref{app:lem:bound-the-differece-TrPhi-TrPhi(z/rho)} with $\rho_1$, we obtain
    \begin{equation*}
        Q_2\leq C''(1 -\rho^2)(\|\Tre_{\rho_1}\Phi'(z)\|_{L_2}^2 + \|\Tre_{\rho_1}\Phi'(z/\rho_1)\|_{L_2}^2).
    \end{equation*}
    Applying \Cref{app:clm:||Tr rho1 Phi'||2<= eC ||Tr rho Phi'||2} and \Cref{app:lem:upper-bound-centered-smoothed-derivative-squared} again, we obtain
    \begin{equation*}
        Q_2\leq 4C''(1 -\rho^2)e^C\|\Tr\Phi'\|_{L_2}^2.
    \end{equation*}
    
    Plugging the upper bounds on $Q_1$ and $Q_2$ back into \Cref{app:eq:upper-bound-smoothed-error-2} yields $\Ez[(\Tr\Phi(z) - \Tre_{\rho_1}\Phi(z/\rho_1))^2]\lesssim e^C(1 - \rho^2)\|\Tr\Phi'(z)\|_{L_2}^2$. Finally, combining with \Cref{app:eq:upper-bound-smoothed-error-1}, we obtain
    \begin{equation*}
        \Ez[(\Tr\Phi(z) - \Phi(z))^2]\lesssim e^C (1 - \rho^2)\|\Tr\Phi'\|_{L_2}^2.
    \end{equation*}
    Since $e^C$ is an absolute constant, this completes the proof of \Cref{app:prop:smoothing-error-bound}.
\end{proof}

\subsubsection{Proof of Supplementary Claims}\label{app:subsec:proof-supplementary-claims-in-monotone}
Below, we provide proofs for the supplementary claims appeared in \Cref{app:sec:learn-monotone}. 
\HelperTrPhiMinusTrPhisquare*
\begin{proof}[Proof of \Cref{app:claim:helper-[TrPhi-TrPhi(z/r)]**2}]
Simple algebraic calculation yields
\begin{align*}
    -\frac{t_j^2}{2} + \frac{\rho^2(t_i^2 + t_j^2)}{2(1 - \rho^4)} - \frac{\rho^4 t_it_j}{1 - \rho^4} & =\frac{1}{2(1 - \rho^4)}(-t_j^2  + \rho^2 t_i^2 + \rho^2 t_j^2 + \rho^4 t_j^2 - 2\rho^4t_it_j)\\
    &=\frac{1}{2(1 - \rho^4)}(-(1 - \rho^2)t_j^2 + \rho^2 t_i^2 +  \rho^4 (t_j - t_i)^2 - \rho^4 t_i^2)\\
    &= \frac{1}{2(1 - \rho^4)}(-(1 - \rho^2)t_j^2 + \rho^2 (1 - \rho^2) t_i^2 +  \rho^4 (t_j - t_i)^2).
\end{align*}
Now since $0< t_j-t_i < t_i/\rho - t_i = (1-\rho)t_i/\rho$, we further have
\begin{align*}
    &\quad -\frac{t_j^2}{2} + \frac{\rho^2(t_i^2 + t_j^2)}{2(1 - \rho^4)} - \frac{\rho^4 t_it_j}{1 - \rho^4} \\
    &\leq \frac{1}{2(1 - \rho^4)}(-(1 - \rho^2)t_j^2 + \rho^2 (1 - \rho^2) t_i^2 +  \rho^2 (1 - \rho)^2 t_i^2)\\
    &\leq \frac{1}{2(1 - \rho^4)}(-(1 - \rho^2) + \rho^2 (1 - \rho^2)  +  \rho^2 (1 - \rho)^2 )t_j^2\\
    &=\frac{t_j^2}{2(1 + \rho)(1 + \rho^2)}(-(1 + \rho) + \rho^2 (1 + \rho)  +  \rho^2 (1 - \rho) ) = \frac{-t_j^2(1-\rho)(1+2\rho)}{2(1 + \rho)(1 + \rho^2)}<0. 
\end{align*}
Thus, indeed we have $-t_j^2/2\leq -\rho^2(t_i^2+t_j^2-2\rho^2t_it_j)/(2(1-\rho^4))$. \end{proof}

\helperRhoOneRhorationOne*
\begin{proof}[Proof of \Cref{app:claim:helper-rho1-rho-ratio}]
    Since $\rho_1,\rho<1$, we only need to show that $\rho_1^2(1- \rho^4)\geq 1 - \rho_1^4$. Plugging in the value of $\rho_1$, we have
    \begin{gather*}
        \rho_1^2(1 - \rho^4) = (\rho^2 + C(1 - \rho^2)/M^2)(1 - \rho^4) = (\rho^2 + C(1 - \rho^2)/M^2)(1 - \rho^2)(1 + \rho^2);\\
         1 - \rho_1^4 = 1 - (\rho^2 + C(1 - \rho^2)/M^2)^2 =(1 + \rho^2 + C(1 - \rho^2)/M^2)(1- \rho^2)(1 - C/M^2).
    \end{gather*}
    Therefore, our goal is to prove that
    \begin{equation*}
        (\rho^2 + C(1 - \rho^2)/M^2)(1 - \rho^2)(1 + \rho^2)\geq (1 + \rho^2 + C(1 - \rho^2)/M^2)(1- \rho^2)(1 - C/M^2).
    \end{equation*}
    Dividing both sides of the inequality above by $1- \rho^2 > 0$ yields that it is sufficient to show the following inequality:
    \begin{align*}
        \rho^2 + C(1 - \rho^2)/M^2 + (\rho^2 + C(1 - \rho^2)/M^2)\rho^2 \geq (1 - C/M^2) + (\rho^2 + C(1 - \rho^2)/M^2)(1 - C/M^2).
    \end{align*}
    Since $\rho^2\geq 1 - C/M^2$, we have $\rho^2 + C(1 - \rho^2)/M^2\geq 1 - C/M^2$ and 
    \begin{equation*}
        (\rho^2 + C(1 - \rho^2)/M^2)\rho^2 \geq (\rho^2 + C(1 - \rho^2)/M^2)(1 - C/M^2).
    \end{equation*}
    Thus, it holds that $\rho_1^2(1 - \rho^2)\geq 1 - \rho_1^4$.
\end{proof}

\helperRhoOneRhoRatio*
\begin{proof}[Proof of \Cref{app:claim:helper2-rho1-rho-ratio}]
    For any fixed $\rho\in(0,1)$, let us define
    \begin{equation*}
        h(M) = \frac{(\rho^2 + C(1 - \rho^2)/M^2)(1 - \rho^4)}{1 - (\rho^2 + C(1 - \rho^2)/M^2)^2}.
    \end{equation*}
    It is easy to see that $h(M)$ is a decreasing function with respect to $M>0$, therefore, for any fixed $\rho\in(0,1)$ and any $M^2\geq 2C$, we have
    \begin{equation*}
        h(M) = \frac{(\rho^2 + C(1 - \rho^2)/M^2)(1 - \rho^4)}{1 - (\rho^2 + C(1 - \rho^2)/M^2)^2}\leq h(\sqrt{2C}) = \frac{((1 + \rho^2)/2)^2(1- \rho^4)}{1 - ((1 + \rho^2)/2)^2}\leq 4.
    \end{equation*}
    Therefore, for any $\rho\in(0,1)$, it holds $\rho_1^4/(1 - \rho_1^4)\leq 4/(1 - \rho^2)$.
\end{proof}

\HelperTrPhiSquare*
\begin{proof}[Proof of \Cref{app:clm:||Tr rho1 Phi'||2<= eC ||Tr rho Phi'||2}]
   To prove this claim, we recall the explicit expression of $\Ez[(\Tre_{\rho_1}\Phi'(z))^2]$ using the formula displayed in \Cref{app:lem:expression-E[(Trho-Phi'(z))**2]}:
   \begin{align*}
       &\quad \Ez[(\Tre_{\rho_1}\Phi'(z))^2] = \sum_{i,j = 1}^m \frac{A_iA_j}{2\pi\sqrt{1 - \rho_1^4}}\exp\bigg(-\frac{t_i^2 + t_j^2}{2(1 - \rho_1^4)} + \frac{\rho_1^2 t_it_j}{1 - \rho_1^4}\bigg)\\
       &=\sum_{i,j = 1}^m \frac{A_iA_j}{2\pi\sqrt{1 - \rho_1^4}}\exp\bigg(-\frac{(t_i^2 + t_j^2-2\rho_1^2 t_i t_j)}{2(1 - \rho_1^4)}\bigg)\\
       &\overset{(i)}{\leq} \sum_{i,j = 1}^m \frac{A_iA_j}{2\pi\sqrt{(1 - \rho_1^2)(1 + \rho_1^2)}}\exp\bigg(-\frac{t_i^2 + t_j^2}{2(1 - \rho^4)} + \frac{\rho_1^2 t_i t_j}{1 - \rho^4}\bigg)\\
       &\overset{(ii)}{\leq} \sum_{i,j = 1}^m \frac{A_iA_j}{2\pi\sqrt{(1 - \rho^2)(1 - C/M^2)(1 + \rho_1^2)}}\exp\bigg(-\frac{t_i^2 + t_j^2}{2(1 - \rho^4)} + \frac{\rho^2 t_i t_j}{1 - \rho^4} + \frac{(\rho_1^2 - \rho^2) t_i t_j}{1 - \rho^4} \bigg)\\
       &\overset{(iii)}{\leq} 2\sum_{i,j = 1}^m \frac{A_iA_j}{2\pi\sqrt{1 - \rho^4}}\exp\bigg(-\frac{t_i^2 + t_j^2}{2(1 - \rho^4)} + \frac{\rho_1^2 t_i t_j}{1 - \rho^4}\bigg)\exp\bigg(\frac{(\rho_1^2 - \rho^2) t_i t_j}{1 - \rho^4}\bigg)\\
       &\leq  2\Ex[(\Tr\Phi'(z))^2]\exp\bigg(\frac{(\rho_1^2 - \rho^2) M^2}{1 - \rho^4}\bigg).
   \end{align*}
    Inequality $(i)$ is due to the facts that $(t_i^2 + t_j^2-2\rho_1^2 t_i t_j)\geq 0$ for any $t_i,t_j\in\R$ and that $1/(1 - \rho_1^4)\geq 1/(1 - \rho^2)$ since $\rho<\rho_1<1$; in $(ii)$ we plugged in the definition $\rho_1$; $(iii)$ comes from the assumption that $C/M^2\leq 1/4$ and the fact that $1 + \rho_1^2\geq 1+ \rho^2$ since $\rho_1>\rho$. Finally, for the term $\exp((\rho_1^2 - \rho^2) t_i t_j/(1 - \rho^4))$, bringing in the definition of $\rho_1^2$ and the fact that $|t_i|\leq M,i\in[m]$ we get
    \begin{align*}
        \exp\bigg(\frac{(\rho_1^2 - \rho^2) M^2}{1 - \rho^4}\bigg)&=\exp\bigg(\frac{(1 - \rho^2)(C/M^2) M^2}{(1 - \rho^2)(1 + \rho^2)}\bigg)\leq e^C.
    \end{align*}
    Thus, in summary, we have
    $\Ez[(\Tre_{\rho_1}\Phi'(z))^2]\leq e^C \Ez[(\Tre_{\rho}\Phi'(z))^2].
    $
    \end{proof}

\subsection{Initialization Algorithm for Monotone Activations}\label{app:subsec:initialization-for-monotone}

In this section, we provide an initialization algorithm for $\sigma$  that is an $\eps$-Extended monotone $(B,L)$-Regular activation. The algorithm generates a vector $\w^{(0)}$ satisfying $\theta(\w^{(0)},\w^*)\leq C/M$, where $C$ is an absolute constant and $M$ is at most $\sqrt{\log(B/\eps)-\log\log(B/\eps)}$. Our key idea is to convert the regression problem to the problem of robustly learning halfspaces via {\itshape data transformation}. In particular, we transform $y$ to $\tilde{y}\in\{0,1\}$ by truncating the labels $y$ to $\tilde{y} = \1\{y\geq t'\}$, where this $t'$ is a  carefully chosen threshold. Then we utilize a previous algorithm from~\cite{DKTZ22b} to robustly learn $\w^*$. 

As the main result of this subsection, we prove 
the following proposition, which suffices to 
establish \Cref{app:prop:initialization}. 

\begin{proposition}\label{app:prop:initial2}
Let $\sigma$ be a non-decreasing $(B,L)$-Regular function. Let $M$ be defined as in \Cref{app:claim:bounded-support-sigma'}. Then there exists an algorithm that draws $O(d/\eps^2\log(B/\delta))$ samples, it runs in $\poly(d,N)$ time, and, with probability at least $1-\delta$, it outputs a vector $\vec w$ such that $\theta(\vec w,\wstar)\leq \min\{\pi/6, C/M\}$, where $C>0$ is a universal constant, independent of any problem parameters.
\end{proposition}

The proof of \Cref{app:prop:initialization}
follows from \Cref{app:prop:initial2} and \Cref{app:prop:error-bound-smoothing-tails}.
\begin{proof}[Proof of \Cref{app:prop:initialization}]
    \Cref{app:prop:initial2} implies that there exists an algorithm using $O(d/\eps^2)$ samples and outputs a vector $\w^{(0)}$ such that $\theta_0 = \theta(\w^{(0)},\w^*)\leq C/M$. Now for any $\theta\leq \theta_0$, it holds $\cos\theta^2\geq 1 - \theta_0^2\geq 1 - C^2/M^2$. Thus, using \Cref{app:prop:error-bound-smoothing-tails} we know that $\|\Pm{>1/\theta^2}\sigma\|_{L_2}^2\lesssim \sin^2\theta\|\Tre_{\cos\theta}\sigma'\|_{L_2}^2$. This finishes the proof of \Cref{app:prop:initialization}.
\end{proof}

Since we are only aiming for a constant factor approximate solution, it 
is sufficient to truncate the activation $\sigma$ to $\tilde{\sigma}$ so 
that $\|\sigma - \tilde{\sigma}\|_{L_2}^2\leq C_1\opt$ for some absolute 
constant $C_1$ in \Cref{app:claim:bounded-support-sigma'}. Hence, given 
an activation $\sigma\in\mathcal{H}(B,L)$, the parameter $M$ is defined as follows. 
Fix an absolute constant $C_1\geq 1$. 
There exists a function $\tilde{\sigma}\in\mathcal{H}(B,L)$ 
satisfying $\|\tilde{\sigma} - \sigma\|_{L_2}^2\leq 2C_1\eps$ 
such that $\sigma'(z) = 0$ for all $|z|\geq M$. 
In fact, in the proof of \Cref{app:claim:bounded-support-sigma'}, 
we chose $\tilde{\sigma}(z) = \sigma(z)\1\{-M_-\leq z \leq M_+\} + 
\sigma(M_+)\1\{z\geq M_+\} + \sigma(-M_-)\1\{z\leq -M_-\}$, 
such that $\Ez[(\sigma(z)- \sigma(M_+))^2\1\{z\geq M_+\}] \leq C_1\eps$ 
and $\Ez[(\sigma(z)- \sigma(-M_-))\1\{z\leq -M_-\}] \leq C_1\eps$. 
Then the upper bound $M$ on the support of $\sigma'$ is chosen as $M = \max\{M_+,M_-\}\leq \sqrt{\log(B/\eps)-\log\log(B/\eps)}$. 
In the following, let us assume without loss of generality 
that $M = M_+ \geq M_-$, since if $M_+\leq M_-$ 
we can instead consider $-\sigma(-z)$. 

Our goal is to show that there exists $M^*\geq M$ such that 
the following holds: 
$$\pr[\tilde{y}\neq \1\{\w^*\cdot\x\geq M^*\}]\leq (4/\sqrt{C_1})\pr[\w^*\cdot\x\geq M^*] \;,$$ 
where $\tilde{y}=\mathcal{T}(y) = \1\{y\geq \sigma(M^*)\}$. 
Then, we will use the following fact from~\cite{DKTZ22b}, which states:  
\begin{fact}[\cite{DKTZ22b}, Corollary of Lemma C.3 and Theorem C.1] \label{app:fct:agnostic-general-halfspaces} 
There is an algorithm that for any halfspace 
$\phi(\w^*\cdot\x;t)$ and sample access to a 
distribution $(\x,\ty)\sim\D$ 
of labeled examples with standard Gaussian $\x$ 
and $\opt'$-adversarial noise---meaning 
that $\pr[\phi(\w^*\cdot\x;t)\neq \ty]\leq \opt'$---it 
draws $O(d/\eps^2\log(1/\delta))$ samples from 
$\D$, it runs in polynomial in time, and 
with probability at least $1 - \delta$ 
has the following performance guarantee: 
if $\exp(-t^2/2)/t\geq C_2\opt'$, where $C_2>1$ is a 
large universal constant, the algorithm 
returns $\vec w$ such that 
$\theta(\vec w,\wstar)\exp(-t^2/2)\leq C_3\opt'$ 
and $\theta(\vec w,\wstar)\leq \pi/6$, 
where $C_3$ is a universal constant.
\end{fact}

With the error $\opt' \leq \pr[z\geq M^*]$ and $t = M^*$ in \Cref{app:fct:agnostic-general-halfspaces}, we obtain a vector $\w$ that satisfies $\theta(\w,\w^*)\leq C_3\exp((M^*)^2/2)\pr[z\geq M^*]\leq C_3/M^*\leq C_3/M$, where we used the fact that $\pr[z\geq M^*]\approx \exp(-(M^*)^2/2)/M^*$. This will complete our initialization argument.

To proceed, we prove the following key lemma:
\begin{lemma}\label{app:lem:initialziation-error-small-than-bias}
    Fix $C>1$. Let $f$ be a monotone function such that $f\geq 0$ and $\|f-y\|_{L_2}^2\leq \eps$. Assume that for all $q > 0$ it holds that  $\E[|\1\{f(z)\geq q\}-\1\{y\geq q\}|]\geq \pr[f(z)\geq q]/C$. Then, it holds that
    $\|f\|_{L_2}^2\leq 5C^2\eps$.
\end{lemma}
\begin{proof}
    Let $T(q)=\pr[f(z)\geq q]$ and  $\Delta(q) = \E[|\1\{f(z) \geq q\} - \1\{y \geq q\}|]$. From the assumption we have that $T(q)\leq C \Delta(q)$. Therefore, we have that

\begin{align*}
\E[f^2] &= \int_0^\infty 2q T(q) dq \leq \int_0^\infty 2q (C \Delta(q)) dq \\
&= 2C \int_0^\infty q (\pr[f \geq q, y < q] + \pr[f < q, y \geq q]) dq \\
&= 2C \left( \E\left[ \frac{f^2 - y^2}{2} \1\{f > y\} \right] + \E\left[ \frac{y^2 - f^2}{2} \1\{y > f\} \right] \right) \\
&= 2C \E\left[ \frac{|f^2 - y^2|}{2} \right] = C \E[|f - y| |f + y|] \\
&\le C \sqrt{\E[(f - y)^2] \E[(f + y)^2]} \le C \sqrt{\eps \E[(f + y)^2]} . 
\end{align*}

Note that $\E[(f + y)^2]\leq 4\E[f^2]+4\eps$. 
Therefore, we have that 
\[
\E[f^2]\leq  C \sqrt{4\eps (\E[f^2]+\eps)}\;.
\]
Letting $\tau=\E[f^2]$, the above becomes
\[
\tau^2\leq 4C^2\eps \tau + 4C^2\eps^2\;.
\]
Maximizing over $\tau$, we have that 
$\tau\leq 5 C^2 \epsilon$ provided that $C>1$. 
Therefore, $\E[f^2]\leq 5C^2\eps$.
\end{proof}

We can now prove \Cref{app:prop:initial2}.
\begin{proof}[Proof of \Cref{app:prop:initial2}]
Let $\sigma\in\mathcal{H}(B,L)$.
Throughout the proof, we make the following assumptions that are without loss of generality:
\begin{enumerate}
    \item There exists $\bar{M}<\infty$ such that $\sigma(z) = \sigma(\bar{M})$ when $z\geq \bar{M}$ and $\sigma(z) = \sigma(-\bar{M})$ when $z\leq -\bar{M}$. This is without loss of generality, as follows from \Cref{app:claim:bounded-support-sigma'}.
    \item $B = \max\{|\sigma(\bar{M})|,|\sigma(-\bar{M})|\} = \sigma(\bar{M})$, since we can always shift $\sigma(z)$ to $\sigma(z) + |\sigma(-\bar{M})|$ without affecting any of the results. 
    \item By \Cref{app:claim:bound-y}, it holds $|y|\leq B = \sigma(\bar{M})$ without loss of generality.
    \item $\Exy[(y - \sigma(\w^*\cdot\x))^2]\leq \eps$, and $h(z) = \sigma(0)$ is not an approximate solution, i.e., for any absolute constant $C$ we have $\Exy[(y-h(\w^*\cdot\x))^2]\geq C\eps$.
     \item It holds that $\Exy[(\sigma(z) - \sigma(0))^2\1\{z\geq 0\}]\geq \Exy[(\sigma(z) - \sigma(0))^2\1\{z\leq 0\}]$, because if this does not hold, we can use $\tilde{\sigma}(z) = -\sigma(-z)$.
    \item There exists $M\in[0,\bar{M}]$ such that $\Ez[(\sigma(z) - \sigma(M))^2\1\{z\geq M\}] \geq C_1\eps$, where $C_1>1$ is a large absolute constant. In the rest of the proof, 
    we will denote by $M$ the smallest value in $[0,\bar{M}]$ such that 
    $\Ez[(\sigma(z) - \sigma(M))^2\1\{z\geq M\}] \geq C_1\eps$.
   Such an $M$ exists. To see this, we first observed that $\Ez[(\sigma(z) - \sigma(0))^2\1\{z\leq 0\}]\geq C_1\eps$, because otherwise, according to assumption 5 above, we will have $\Ez[(\sigma(z) - \sigma(0))^2]\leq 2C_1\eps$, indicating that $\Ez[(y - h(z))^2]\leq 2\Ez[(y-\sigma(z))^2] + 2\Ez[(\sigma(z)-h(z))^2]\leq (2C_1 +2)\eps$. This implies $h(z) = \sigma(0)$ is a constant factor solution, contradicting to assumption 4.
    Let $g(t)\eqdef \Ez[(\sigma(z) - \sigma(t))^2\1\{z\geq t\}]$, which is a decreasing function from $t=0$ to $t=+\infty$. Since $g(0)\geq C_1\eps$ and $g(\bar{M}) = 0$, we know there must exists a minimum real number $M$ such that $g(M)\geq C_1\eps$.
\end{enumerate}

Given the assumptions above, we claim that there must exist $q'\in[0,B - \sigma(M)]$ such that
\begin{equation}\label{eq:small-error}
    \pr[\1\{y\geq \sigma(M) + q'\}\neq \1\{\sigma(z)\geq \sigma(M) + q'\}]\leq 4/\sqrt{C_1}\pr[\sigma(z)\geq \sigma(M) + q'].
\end{equation}
Suppose for the sake of contradiction that for any $q\in[0,B - \sigma(M)]$ it holds $\pr[\1\{y\geq \sigma(M) + q\}\neq \1\{\sigma(z)\geq \sigma(M) + q\}]\geq 4/\sqrt{C_1}\pr[\sigma(z)\geq \sigma(M) + q]$. 
Note that for $q\geq B-\sigma(M)$, since $\sigma(z)\leq B$ and $y\leq B$, we have $\1\{y\geq \sigma(M) + q\} = \1\{\sigma(z)\geq \sigma(M) + q\} = 0$. Thus, we have $\pr[\1\{y\geq \sigma(M) + q\}\neq \1\{\sigma(z)\geq \sigma(M) + q\}]\geq 4/\sqrt{C_1}\pr[\sigma(z)\geq \sigma(M) + q]$ for all $q\geq 0$ under the assumption. 

Now let $f(z)= (\sigma(z) - \sigma(M))\1\{z\geq M\}$, $y' = (y-\sigma(M))\1\{y\geq \sigma(M)\}$. Then, for any $q\geq 0$,
\begin{align*}
    4/\sqrt{C_1}\pr[\sigma(z)\geq \sigma(M) + q]&=4/\sqrt{C_1}\pr[f(z)\geq q]\\
    &\leq \pr[\1\{y\geq \sigma(M) + q\}\neq \1\{\sigma(z)\geq \sigma(M) + q\}]\\
    &=\pr[\1\{y'\geq q\}\neq\1\{f(z)\geq q\}].
\end{align*}
Furthermore, we have
\begin{align*}
    \Ez[(f(z) - y')^2] &= \Ez[((\sigma(z)-\sigma(M))\1\{z\geq M\} - (y-\sigma(M))\1\{y\geq \sigma(M)\})^2]\\
    &\leq 2\Ez[(\sigma(z) - y)^2\1\{z\geq M\}] + 2\Ez[(y - \sigma(M))^2(\1\{y\geq \sigma(M)\} - \1\{z\geq M\})^2]\\
    &\leq 2\eps + 2\Ez[(y - \sigma(M))^2(\1\{y\geq \sigma(M)\} - \1\{z\geq M\})^2].
\end{align*}
Note that it holds $0\leq (y - \sigma(M))\1\{y\geq \sigma(M), z<M\}\leq (y-\sigma(z))\1\{y\geq \sigma(M), z<M\}$ and $0\leq (\sigma(M) - y)\1\{y<\sigma(M),z\geq M\}\leq (\sigma(z) - y)\1\{y<\sigma(M),z\geq M\}$.
Therefore,
\begin{align*}
    &\quad \Ez[(y - \sigma(M))^2(\1\{y\geq \sigma(M)\} - \1\{z\geq M\})^2] \\
    &= \Ez[(y - \sigma(z))^2\1\{y\geq \sigma(M), z<M\}] + \Ez[(y - \sigma(z))^2\1\{y< \sigma(M), z\geq M\}]\leq 2\eps.
\end{align*}
Combining with the upper bound on $\Ez[(f(z)-y')^2]$ yields $\Ez[(f(z)-y')^2]\leq 6\eps$.
Hence the conditions of \Cref{app:lem:initialziation-error-small-than-bias} are satisfied, and applying \Cref{app:lem:initialziation-error-small-than-bias} we obtain $\Ez[f^2] = \Ez[(\sigma(z) - \sigma(M))^2\1\{z\geq M\}]\leq 2(C_1/16)(6\eps)\leq (3/4)C_1\eps$. However, recall that $M$ is chosen such that $\Ez[(\sigma(z) - \sigma(M))^2\1\{z\geq M\}] \geq C_1\eps$, therefore we have reached a contradiction.

Now let $M^* = \argmin\{0\leq M'\leq \bar{M}:\sigma(M') = \sigma(M) + q'\}$, $q'$ satisfying \Cref{eq:small-error}, we have $M^*\in [M,\bar{M}]$. Note that this $q'$ can be found via a grid search on the interval $[0, B-\sigma(M)]$, as we can discretize the label $y$ and activation $\sigma$ using a $\sqrt{\eps}$ grid, therefore, there will only be $\poly(1/\sqrt{\eps},B)$ number of possible choices of $q'$. The procedure is standard and we omit it here. The argument above implies that it hold $\opt' = \pr[\1\{y\geq \sigma(M) + q'\}\neq \1\{\sigma(z)\geq \sigma(M) + q'\}]\leq 4/\sqrt{C_1}\pr[z\geq M^*] = (1/C_2)\exp(-(M^*)^2/2)/M^*$ for $C_2$ being a large absolute constant. Hence the conditions of \Cref{app:fct:agnostic-general-halfspaces} are satisfied, which then implies that there exists an algorithm that given labels $\tilde{y} = \1\{y\geq \sigma(M^*)\}$ and a target halfspace $\phi(\w^*\cdot\x;M^*)$, returns a vector $\w$ such that $\theta(\w,\w^*)\leq \min\{\pi/6, C_3\exp((M^*)^2/2)\opt'\} = \min\{\pi/6, C_3/M^*\}\leq \min\{\pi/6, C_3/M\}$. This completes the proof of \Cref{app:prop:initial2}.
\end{proof}

\end{document}